\documentclass[bj,numbers,noshowframe,preprint]{imsart}

\RequirePackage{amsthm,amsmath,amsfonts,amssymb,mathtools,bm}
\RequirePackage[numbers]{natbib}
\RequirePackage{booktabs,graphicx,placeins,float,needspace}
\RequirePackage{hyperref}
\RequirePackage[capitalise,noabbrev]{cleveref}

\setpkgattr{copyright}{text}{}
\hypersetup{hypertexnames=false}

\startlocaldefs
\providecommand{\xishape}{\itshape}
\providecommand{\bseriesno}[1]{#1}

\newtheorem{theorem}{Theorem}[section]
\newtheorem{proposition}[theorem]{Proposition}
\newtheorem{lemma}[theorem]{Lemma}
\newtheorem{corollary}[theorem]{Corollary}
\theoremstyle{definition}
\newtheorem{definition}[theorem]{Definition}
\newtheorem{assumption}[theorem]{Assumption}
\newtheorem{profilecondition}{Condition}
\crefname{profilecondition}{condition}{conditions}
\Crefname{profilecondition}{Condition}{Conditions}
\theoremstyle{definition}
\newtheorem{remark}{Remark}[section]

\newcommand{\R}{\mathbb R}
\newcommand{\E}{\mathbb E}
\newcommand{\Pp}{\mathbb P}
\newcommand{\one}{\mathbf 1}
\newcommand{\cA}{\mathcal A}
\newcommand{\cE}{\mathcal E}
\newcommand{\cC}{\mathcal C}
\newcommand{\cD}{\mathcal D}
\newcommand{\cG}{\mathcal G}
\newcommand{\cH}{\mathcal H}

\newcommand{\cR}{\mathcal R}
\newcommand{\KL}{D_{\mathrm{KL}}}
\newcommand{\TV}{d_{\mathrm{TV}}}
\newcommand{\diam}{\operatorname{diam}}
\newcommand{\supp}{\operatorname{supp}}
\newcommand{\sech}{\operatorname{sech}}
\newcommand{\argmin}{\operatorname*{arg\,min}}

\newcommand{\norm}[1]{\left\lVert #1\right\rVert}

\newcommand{\set}[1]{\left\{#1\right\}}
\newcommand{\paren}[1]{\left(#1\right)}

\endlocaldefs

\begin{document}

\begin{frontmatter}
\pdfsubject{Statistics Theory preprint}
\title{Exact Global Minimax Risk near Nonidentification}
\runtitle{Exact Global Risk near Nonidentification}

\begin{aug}
\author[A]{\fnms{Kaiyue}~\snm{Yang}\ead[label=e1]{yangky24@mails.tsinghua.edu.cn}}
\author[B]{\fnms{Yunan}~\snm{Wu}\ead[label=e2]{wuyunan@mail.tsinghua.edu.cn}} 
\address[A]{Qiuzhen College, Tsinghua University, Haidian District, Beijing 100084, P. R. China\printead[presep={,\ }]{e1}} 
\address[B]{Yau Mathematical Sciences Center, Tsinghua University, Haidian District, Beijing 100084, P. R. China\printead[presep={,\ }]{e2}} 
\end{aug}

\begin{abstract}
We study squared-error estimation when a nuisance parameter is unidentified on a submodel but affects a target that vanishes there. For models with higher-order profile expansions, we give sufficient conditions under which the largest critical-path minimax risk is the exact asymptotic minimax risk over the full parameter space. An efficient observable feature, negligible calibration error and stable recovery yield one estimator attaining the leading constant uniformly, with the profile coordinates unknown. The estimator uses calibration estimates, a local feature mean and an empirical recovery feature. The constant is expressed through a bounded-normal decision problem. We verify the conditions in two models. In centered Gaussian logistic errors-in-variables regression with one surrogate and two unknown variances, we determine the fifth-order global coefficient risk and its exact constant. A sparse extension separates variance-decomposition uncertainty from direction estimation and specifies when the scalar constant is preserved. In a degree-balanced replicated-network model, we determine the sixth-order global edge-contrast risk when block proportions are unknown. Triangles identify a parameter combination; a bowtie supplies the additional information needed for the contrast.
\end{abstract}


\end{frontmatter}

\section{Introduction}\label{sec:introduction}

At regular parameter values, Fisher information determines local asymptotic efficiency bounds \cite{vanderVaart1998}. These pointwise approximations need not remain valid uniformly near nonidentification \cite{dacunhaGassiat1997,andrewsCheng2012}, where higher-order likelihood terms can determine the estimation rate. We study whether the hardest local experiments determine the exact minimax risk over the full parameter space when their defining coordinates must also be estimated. We give sufficient conditions and construct a single estimator attaining the resulting constant uniformly.

Consider $\mathcal P=\{P_\xi:\xi\in\Xi\}$, with unknown parameter $\xi$ and scalar target $\tau:\Xi\to\mathbb R$, to be estimated from $n$ independent observations. Near nonidentification, suppose $\xi=\xi(c,v,\lambda)$, where $c\in\mathcal C$ is a regular coordinate, $v$ describes departures from the null, and $\lambda\in\Lambda(c)$ lies in a compact interval. Write $P_{c,v,\lambda}=P_{\xi(c,v,\lambda)}$ and $\tau(c,v,\lambda)=\tau\{\xi(c,v,\lambda)\}$. Define the \emph{null parameter set}
$$
\Xi_0=\{\xi(c,0,\lambda):c\in\mathcal C,\ \lambda\in\Lambda(c)\},
$$
and let $\mathcal P_0 =\{P_{c,0,\lambda}:c\in\mathcal C,\ \lambda\in\Lambda(c)\}$ denote the corresponding null submodel. Suppose that 
$$
P_{c,0,\lambda}=P_{c,0},\qquad
\tau(c,0,\lambda)=0,\qquad
\tau(c,v,\lambda)=v\beta(c)\lambda+o(|v|),
$$ 
where $\beta(c)$ specifies the leading target coefficient. If $\Lambda(c)$ is nondegenerate, $\lambda$ is unidentified at $v=0$ but affects the target nearby. Although the target vanishes on $\Xi_0$, uncertainty about $\lambda$ can govern the maximal risk in shrinking neighborhoods of that set.

After lower-order variation has been absorbed into the profile coordinates, suppose that, for fixed $c$, a reference $\lambda_0\in\Lambda(c)$ and all $\lambda\in\Lambda(c)$, the likelihood satisfies
$$
 \frac{dP_{c,v,\lambda}}{dP_{c,v,\lambda_0}}(O)-1
 =(\lambda-\lambda_0)v^kS_{c,v,\lambda_0}(O)+r(O\mid c,v,\lambda)
$$
for sufficiently small $|v|$. Here $O$ is one observation, $S_{c,v,\lambda_0}$ is the leading coefficient, centered under $P_{c,v,\lambda_0}$, and $r$ is a remainder smaller than $|v|^k$ in mean square under the reference law and the stated conditions. The \emph{profiled-jet order} is the smallest positive integer $k$ satisfying the profiled-jet conditions below. Let $I(c)=\lim_{v\to0}\E_{c,v,\lambda_0}[S_{c,v,\lambda_0}^{2}]\in(0,\infty)$ denote its limiting variance. Information for $\lambda$ is of order $|v|^{2k}$ per observation at $\lambda_0$, giving critical paths $\bigl(c,xn^{-1/(2k)},\lambda\bigr)$ with $c$ and $x\ne0$ fixed. The target changes at order $n^{-1/(2k)}$ on these paths, while the likelihood ratios retain nondegenerate limits. Nondegenerate critical profiles yield global squared-risk lower bounds of order $n^{-1/k}$. We give conditions for constructing the coordinates and obtaining matching uniform upper bounds.

Under a uniform setup and conditions for an efficient observed statistic, negligible calibration error and stable recovery, we prove, for $k\ge2$,
$$
\inf_{\widehat\tau_n}\sup_{\xi\in\Xi}
\E_\xi\!\left[\{\widehat\tau_n-\tau(\xi)\}^{2}\right]
=(\mathsf C_k+o(1))\,n^{-1/k},
$$
where $\mathsf C_k =K_k\sup_{c\in\mathcal C} \beta(c)^2I(c)^{-1/k}L(c)^{2-2/k}$, $L(c)=\diam\Lambda(c)$, and $K_k$ is the universal constant of the corresponding bounded-normal minimax problem. Both $c$ and $v$ are unknown in this global result. The upper bound therefore requires a procedure that estimates the local coordinates while retaining the information needed for the leading constant.

Along $v_n=xn^{-1/(2k)}$, likelihood ratios converge to those of $W_x\mid\lambda\sim N\!\left(\lambda,\{I(c)x^{2k}\}^{-1}\right)$, $\lambda\in\Lambda(c)$, with rescaled target $x\beta(c)\lambda$. This is the limiting estimation problem on a critical path. A local approximation alone does not ensure a global upper bound: estimation must also remain controlled between critical paths and at parameters separated from the null. Calibration estimates, an efficient local statistic and an outer recovery statistic yield one rule attaining the largest critical risk uniformly over $\Xi$. Restricting estimation to these statistics preserves the exact first-order global minimax risk for the stated target and loss.

In centered Gaussian logistic errors-in-variables regression with one surrogate and two unknown variances, we quantify how uncertainty about the variance decomposition limits coefficient estimation. The third Hermite coefficient first separates the decompositions at fifth order after profiling, giving global squared risk $n^{-1/5}$ with an exact constant and an attaining estimator. The sparse extension separates this difficulty from direction estimation; the scalar constant is preserved when the maximal sparse complexity is $o(n^{4/5})$.

In a degree-balanced model of independent fixed-size networks, triangles identify $v=q\theta$, while a six-edge bowtie first separates the unknown block proportion from the edge contrast. We obtain exact global risk of order $n^{-1/6}$ and an attaining estimator. Numerical experiments assess finite-sample recovery, with critical-path checks in the Supplementary Material.

\Needspace{6\baselineskip}
\subsection{Relation to existing work}
Likelihood-based identification is classical \cite{rothenberg1971}. Rotnitzky et al.\ \cite{rotnitzkyEtAl2000} study identifiable models with singular Fisher information, whereas our null profiles are genuinely nonidentified. Locally conic and weak-identification theories describe nonregular behavior \cite{dacunhaGassiat1997,andrewsCheng2012}; Andrews and Cheng already cover nuisance nonidentification across identification regimes. Our question is the exact uniform risk of a target that vanishes on the null submodel but depends on the unidentified parameter nearby.

Finite-mixture theory connects higher-order separation, inverse inequalities and estimation rates \cite{chen1995mixture,hoNguyen2016weak,hoNguyen2016strong,hoNguyen2019singularity,heinrichKahn2018}; minimum-$\Phi$-distance theory develops recovery from empirical features \cite{weiMukherjeeNguyen2025}. We prescribe no finite-mixture representation but require profile conditions. The profiled-jet, contact and inverse orders distinguish a chosen local profile, global target separation and recovery from a chosen feature. When they coincide, the stated moment and measurability conditions yield matching global lower and upper rates.

Donoho and Liu \cite{donohoLiu1991,donohoLiu1991iii} derive modulus and sharp bounded-normal lower bounds, including affine targets over convex density classes. Our exact-constant result requires no convexity but assumes a scalar unidentified coordinate on a compact interval and an affine leading target. Efficiency, calibration and stable recovery yield one observable estimator attaining the largest critical-path constant globally, with estimated profile coordinates. This attainment result uses established local minimax and bounded-normal theory \cite{hajek1972,vanderVaart1998,lecamYang2000,bickel1981,casellaStrawderman1981}.

Logistic measurement-error methods are established \cite{carroll1984binary,stefanski1985logistic,stefanskiCarroll1987,nakamura1990corrected,carrollWand1991,carroll2006measurement}. K\"uchenhoff \cite{kuchenhoff1995identification} studies scalar Gaussian identification without extra error information and its poor estimability. We quantify this difficulty for the centered model without an intercept through the global coefficient risk, its exact constant and an attaining estimator. Our isotropic Gaussian sparse result complements conditional-score lasso \cite{sorensen2015measurement} and covariance-free selection \cite{sorensenEtAl2018covariate}: it separates coefficient-scale uncertainty from sparse direction estimation and specifies when the scalar constant is preserved.

Network identification from graph distributions and motifs is established \cite{allmanMatiasRhodes2011,gundersonBravoHermsdorffOrbanz2023,soganRebafka2026}; independent fixed-size replication is also studied \cite{duranteDunsonVogelstein2017,jochmans2024}. Jochmans's path-rank condition fails throughout our degree-balanced submodel: conditional path distributions coincide across classes at every length. Degree-separated graph-pencil guarantees likewise do not apply. We obtain the sixth-order separation, exact uniform edge-contrast risk and attaining estimator with unknown block proportions and independently redrawn labels. Triangles identify $v=q\theta$; the bowtie first distinguishes the proportion at fixed nonzero $v$. This differs from latent-label recovery~\cite{abbe2018sbm}.

\Needspace{5\baselineskip}
\subsection{Paper outline and notation}\label{sec:paper-outline}
Section~\ref{sec:identification-orders} introduces the identification orders and rate bounds. Section~\ref{sec:critical-experiment} derives the critical Gaussian experiment, and Section~\ref{sec:globalization} proves exact global risk and uniform attainment. Sections~\ref{sec:model-rate} and~\ref{sec:network} develop the logistic and network examples. Section~\ref{sec:numerics-main} presents numerical results, and Section~\ref{sec:discussion} discusses scope and extensions. All proofs are collected in the Supplementary Material.

\paragraph{Notation} The symbols $\norm\cdot_2$, $\norm\cdot_{\cH}$, and $\KL(P\Vert Q)$ denote the Euclidean norm, the norm on a Hilbert space $\cH$, and Kullback--Leibler divergence, respectively. For a nonempty closed convex set $A$, $\Pi_A$ denotes Euclidean projection onto $A$. For nonnegative quantities, $a\lesssim b$ means $a\le Cb$ with $C$ independent of the asymptotic indices, and $a\asymp b$ means $a\lesssim b$ and $b\lesssim a$.

\section{Identification orders and rate bounds}
\label{sec:identification-orders}
\label{sec:framework}

Failure to identify the full parameter need not preclude estimation of the target. The relevant question is how the target changes relative to the observed-data distribution: a target may vary substantially along a nonidentified direction even when the corresponding distributions are separated only at a higher order. We formalize this comparison through three notions of identification order.

\subsection{Basic definitions}

Fix $d\ge1$. Let $(\mathcal O,\mathcal A_{\mathcal O})$ be a standard Borel space and $\{P_\xi:\xi\in\Xi\}$ a statistical model. Let $O_i\in\mathcal O$, $i=1,\ldots,n$, be an i.i.d. sample from $P_\xi$, and write $\E_\xi$ for expectation. The target is $\tau:\Xi\to\R^d$. We also consider a map $\Psi:\Xi\to\cH$, where $\cH$ is a separable Hilbert space.

We first measure the difficulty of distinguishing parameters through the observed-data distributions. In particular, we ask how small the distributional separation can be when the corresponding targets remain separated at a specified scale. This leads to the contact order.

\begin{definition}[Contact order]
\label{def:orders}
\label{def:contact-order}
Let $\mathbb N={1,2,\ldots}$. Define the \emph{contact order} by
\begin{align*}
 k_{\mathrm C}
 :=\sup\Bigl\{j\in\mathbb N:{}&\ \exists\ t_0,b,C>0\ \text{such that for
 every }0<t<t_0\ \exists\ \xi_{0,t},\xi_{1,t}\in\Xi,\\
 &\norm{\tau(\xi_{1,t})-\tau(\xi_{0,t})}_2\ge bt,
 \quad \KL(P_{\xi_{0,t}}\Vert P_{\xi_{1,t}})\le Ct^{2j}\Bigr\}.
\end{align*}
Set $\sup\varnothing=0$.
\end{definition}

A larger contact order means that the target can remain separated while the corresponding observed-data distributions approach one another more rapidly. Such parameter pairs are difficult to distinguish and therefore provide lower bounds on the minimax risk.

A lower bound, however, does not by itself determine the attainable rate. To obtain a matching upper bound, one needs an observable quantity from which the target can be recovered. The inverse order measures the degree of this recovery.

\begin{definition}[Inverse order]
\label{def:inverse-order}
For the fixed map $\Psi$, define the \emph{inverse order} by
\begin{align*}
 k_{\mathrm I}
 :=\inf\Bigl\{j\in\mathbb N:{}&\ \exists\ A<\infty:\ 
 \ \norm{\tau(\xi)-\tau(\xi')}_2^j
 \le A\norm{\Psi(\xi)-\Psi(\xi')}_{\cH}
 \ \forall\xi,\xi'\in\Xi\Bigr\}.
\end{align*}
Set $\inf\varnothing=\infty$.
\end{definition}

When $k_{\mathrm I}<\infty$, an error of order $\varepsilon$ in $\Psi$ translates into a target error of at most order $\varepsilon^{1/k_{\mathrm I}}$. The contact order governs lower bounds, while the inverse order governs recovery from the chosen feature map.

The two global notions do not specify how higher-order separation arises near a nonidentified submodel. To make this local mechanism explicit, we first introduce a \emph{profiled submodel} in which a component $\lambda$ is unidentified at a null value $v=0$, and ask for the first order at which $\lambda$ becomes visible in the observed distribution after lower-order terms have been profiled out. This motivates the third notion, the profiled-jet order.

\paragraph{Profiled submodel} For $v_0>0$, let $\xi:[-v_0,v_0]\times\Lambda\to\Xi$ index a profiled submodel and abbreviate
$$
 P_{v,\lambda}=P_{\xi(v,\lambda)},\qquad
 \tau(v,\lambda)=\tau\{\xi(v,\lambda)\},
 \qquad |v|\le v_0,\quad\lambda\in\Lambda,
$$
where $\Lambda$ is a nondegenerate compact interval and $P_{0,\lambda}$ is
constant in $\lambda$. Write $\E_{v,\lambda}$ for expectation under
$P_{v,\lambda}$. Let $C(\Lambda)$ and $C^1(\Lambda)$ denote the continuous
and continuously differentiable real-valued functions on $\Lambda$. For
$f\in C^1(\Lambda)$, put 
$$
 \norm f_{C^1(\Lambda)}
 =\sup_{\lambda\in\Lambda}|f(\lambda)|
 +\sup_{\lambda\in\Lambda}|f'(\lambda)|.
$$
\begin{profilecondition}[Profiled jet at order $j$]
\label{cond:profile-jet}
\label{ass:profile-jet}
For $j\in\mathbb N$, there exist $\lambda_0\in\Lambda$,
$I\in(0,\infty)$, $\beta\in\R^d\setminus\{0\}$, and jointly measurable
maps $(v,o)\mapsto S_{v,\lambda_0}(o)\in\R$ and
$(v,o,\lambda)\mapsto r(o\mid v,\lambda)\in\R$.
For each $(v,o)$, the function $\lambda\mapsto r(o\mid v,\lambda)$ is
continuously differentiable.  The following conditions hold, with all
limits as $v\to0$.
\begin{description}
\item[\textup{(P1)}] \textit{Likelihood expansion.}
For all sufficiently small $|v|$, $P_{v,\lambda}\ll P_{v,\lambda_0}$ for
every $\lambda\in\Lambda$, and, as an equality in $C(\Lambda)$ under
$P_{v,\lambda_0}$,
\begin{equation}
 \frac{dP_{v,\lambda}}{dP_{v,\lambda_0}}(O)-1
 =(\lambda-\lambda_0)v^jS_{v,\lambda_0}(O)+r(O\mid v,\lambda).
 \label{eq:jet-order-ratio}
\end{equation}
\item[\textup{(P2)}] \textit{Score moments.}
\begin{equation*}
 \E_{v,\lambda_0}S_{v,\lambda_0}=0,\qquad
 \E_{v,\lambda_0}S_{v,\lambda_0}^2\longrightarrow I,\qquad
 \sup_{0<|v|\le v_0}\E_{v,\lambda_0}|S_{v,\lambda_0}|^4<\infty.
\end{equation*}
\Needspace{8\baselineskip}
\item[\textup{(P3)}] \textit{Remainder bound.}
\begin{equation*}
 \E_{v,\lambda_0}\Bigl[
 \sup_{\lambda\in\Lambda}|r(O\mid v,\lambda)|^2
 +\sup_{\lambda\in\Lambda}\left|\frac{\partial }{\partial \lambda }r(O\mid v,\lambda)\right|^2
 \Bigr]=o(|v|^{2j}).
\end{equation*}
\item[\textup{(P4)}] \textit{Target expansion.}
\begin{equation*}
 \sup_{\lambda\in\Lambda}
 \frac{\norm{\tau(v,\lambda)-v\beta\lambda}_2}{|v|}
 \longrightarrow0.
\end{equation*}
\end{description}
\end{profilecondition}

Condition~(P1) specifies the expansion, (P2) ensures a nondegenerate leading coefficient with controlled moments, and (P3) controls the remainder. Together, (P1)--(P3) identify the leading likelihood order $j$. Condition~(P4) links it to first-order target variation in direction $\beta\ne0$.

\begin{definition}[Profiled-jet order]
\label{def:profiled-jet-order}
Let
$$
\mathcal J
=\{j\in\mathbb N:\text{\Cref{cond:profile-jet} holds at order }j\}.
$$
If $\mathcal J\neq\varnothing$, define the \emph{profiled-jet order} by $k_{\mathrm J}:=\min\mathcal J.$
The order $k_{\mathrm J}$ is undefined if $\mathcal J=\varnothing$.
\end{definition}

The order $k_{\mathrm J}$ determines the local statistical scale of the profile. At $k_{\mathrm J}$, the leading likelihood contribution for one observation is of order $|v|^{k_{\mathrm J}}$, giving a total signal of order $n|v|^{2k_{\mathrm J}}$ over an i.i.d. sample. Balancing this quantity at a constant order yields the critical scale
$|v|\asymp n^{-1/(2k_{\mathrm J})}$. Together with the likelihood expansion, the target separation in (P4) gives minimax lower bounds of order $n^{-1/k_{\mathrm J}}$.

The following toy example illustrates how the three orders arise in a simple setting. Two more complex examples are given in \Cref{sec:model-rate,sec:network}.

\paragraph{A Gaussian example}
Consider independent variables
$$
X\sim N(v,1),\qquad Y\sim N(\lambda v^2,1),
\qquad v\in[-1,1],\quad \lambda\in[1,2],
$$
with target $\tau(v,\lambda)=v\lambda$. At $v=0$, the observed distribution does not depend on $\lambda$.

Note that for the two points $(t,1)$ and $(t,2)$, we have $|\tau(t,2)-\tau(t,1)|=t,$ and  $\KL(P_{t,1}\Vert P_{t,2})=\tfrac12t^4.$
Thus, the target changes at first order in $t$, whereas the corresponding observed distributions are separated only at fourth order in Kullback--Leibler divergence. For the mean map
$\Psi(v,\lambda)=(v,\lambda v^2)$, there is a finite constant $A$ such that
$$
|v\lambda-v'\lambda'|^2
\le A\norm{(v,\lambda v^2)-(v',\lambda'v'^2)}_2.
$$
Hence, the inverse order is two. The same pair of points rules out a first-order inverse inequality. Because the Gaussian Kullback--Leibler divergence is one half the squared Euclidean distance between the means, the global inverse inequality  also rules out contact orders greater than two. 

The profiled-jet order can be read directly from the likelihood ratio. For $\lambda_0=1$,
$$
\log\frac{dP_{v,\lambda}}{dP_{v,1}}
=(\lambda-1)v^2(Y-v^2)-\tfrac12(\lambda-1)^2v^4.
$$
Expanding the exponential verifies \Cref{cond:profile-jet} at order two, with
$S_{v,1}=Y-v^2$ and $I=\beta=1$. Hence, $k_{\mathrm C}=k_{\mathrm I}=k_{\mathrm J}=2,$ and the corresponding critical scale is $n^{-1/4}$. This fourth-root scaling is consistent with existing results on likelihood inference for identifiable models with singular information \citep{rotnitzkyEtAl2000}.

\begin{remark}[Smooth origin of the profiled jet]
\label{rem:smooth-jet-origin}
The connection can be made explicit through a local reparametrization.
Section~S1.1 of the Supplementary Material constructs these coordinates
by coefficient matching. The resulting likelihood
expansion has the form \eqref{eq:jet-order-ratio}.
A bijective reindexing preserves $k_{\mathrm C}$ and $k_{\mathrm I}$
when the target and feature map are reindexed with the model.
\end{remark}

\subsection{Rate bounds and the order hierarchy}
The contact and inverse orders control lower and upper bounds, respectively. The profiled-jet order supplies a local likelihood construction linking them.

We next study how these orders determine the minimax rate for estimating $\tau(\xi)$. To formalize the corresponding rate result, fix $k\in\mathbb N$ and first state the conditions needed for the lower and upper bounds:
\begin{description}
\item[\textup{(A1)}] The contact order satisfies $k_{\mathrm C}\ge k$.

\item[\textup{(A2)}] The parameter space $\Xi$ is a compact metric space,
the target $\tau$ is Borel measurable, and the inverse order satisfies
$k_{\mathrm I}\le k$.

\item[\textup{(A3)}] There is a measurable feature $\Phi:\mathcal O\to\cH$
such that
\begin{equation*}
 \Psi(\xi)=\E_\xi\Phi(O),
 \qquad
 \sup_{\xi\in\Xi}\E_\xi
 \norm{\Phi(O)-\Psi(\xi)}_{\cH}^2\le V<\infty.
\end{equation*}

\item[\textup{(A4)}] For each $n\ge1$, the minimum-distance criterion
admits a measurable selector
\begin{equation*}
 \widehat\xi_n\in\argmin_{\zeta\in\Xi}
 \norm{\widehat\Psi_n-\Psi(\zeta)}_{\cH},
 \qquad
 \widehat\Psi_n=\frac1n\sum_{i=1}^n\Phi(O_i).
\end{equation*}
\end{description}

Condition~\textup{(A1)} supplies difficult parameter pairs for the lower
bound. For the upper bound, \textup{(A3)} gives an empirical feature mean
with $n^{-1/2}$ error, \textup{(A4)} ensures measurability, and the inverse
inequality in \textup{(A2)} converts feature error into target error.

These conditions lead to matching minimax bounds when the contact and
inverse orders are controlled by the same exponent.

\begin{theorem}[Minimax consequences of contact and inverse orders]
\label{thm:generic-rate}
Under \textup{(A1)}, there is a constant $c>0$ such that, for all
sufficiently large $n$,
\begin{equation}\label{eq:generic-lower}
\inf_{T_n}\sup_{\xi\in\Xi}
\E_\xi\norm{T_n-\tau(\xi)}_2^2\ge c n^{-1/k},
\end{equation}
where the infimum is over all measurable estimators
$T_n:\mathcal O^n\to\R^d$.

Under \textup{(A2)--(A4)}, there is a constant $C<\infty$, independent of
$n$, such that
\begin{equation}\label{eq:generic-upper}
\sup_{\xi\in\Xi}\E_\xi
\norm{\tau(\widehat\xi_n)-\tau(\xi)}_2^2
\le C n^{-1/k},\qquad n\ge1.
\end{equation}
Consequently, if \textup{(A1)--(A4)} hold, then

$$
 \inf_{T_n}\sup_{\xi\in\Xi}
 \E_\xi\norm{T_n-\tau(\xi)}_2^2\asymp n^{-1/k}.
$$

\end{theorem}

The lower bound applies to every estimator; the empirical-feature
minimum-distance estimator attains the upper bound. A common exponent
therefore gives the minimax rate.

\begin{remark}[Measurable selection]
\label{rem:measurable-selection}
Under \textup{(A3)}, if $\Psi:\Xi\to\cH$ is continuous, compactness in
\textup{(A2)} makes the argmin correspondence nonempty, compact-valued, and
measurable. The Kuratowski--Ryll-Nardzewski theorem then gives a Borel
selector, and composition with the measurable $\widehat\Psi_n$ yields
\textup{(A4)}.
\end{remark}

Under \textup{(A3)}, distributional closeness constrains feature separation,
so recovery cannot be faster than the difficult pairs permit. The profiled
likelihood supplies such pairs locally, yielding the following hierarchy.

\begin{theorem}[Hierarchy of the three orders]
\label{thm:order-hierarchy}
\label{cor:contact-inverse}
\label{thm:jet-contact}
\label{cor:order-hierarchy}
Under \textup{(A3)},
\begin{equation*}
k_{\mathrm C}\le k_{\mathrm I}.
\end{equation*}
If, in addition, the profiled paths lie in $\Xi$, $\Lambda$ is
nondegenerate, and $k_{\mathrm J}$ is defined, then
\begin{equation*}
k_{\mathrm J}\le k_{\mathrm C}\le k_{\mathrm I}.
\end{equation*}
\end{theorem}

\Cref{thm:order-hierarchy} implies that the profiled-jet order provides a lower bound on the global contact order, while the inverse order provides an upper bound. When equality holds, the common order determines the minimax rate under the conditions of \Cref{thm:generic-rate}.

\section{Profiled jets and local Gaussian limits}
\label{sec:critical-experiment}

The rate bounds in the preceding section determine the order of the minimax risk but not its leading constant. We obtain the constant by studying the likelihood ratios along parameter sequences approaching the null parameter set at the critical scale.

Fix $k\in\mathcal J$ and let
$(\lambda_0,S_{v,\lambda_0},r,I,\beta)$ satisfy the profiled-jet conditions
in \Cref{cond:profile-jet}. Write $\Lambda=[\ell,u]$, and let $c$ denote a
fixed regular coordinate, which is suppressed throughout this section.

\subsection{Gaussian likelihood limit} 
\label{sec:gaussian-limit}
At the critical scale, the likelihood ratios remain nondegenerate. The next theorem identifies their limit and the corresponding rescaled target.

\begin{theorem}[Profiled-jet limit experiment]
\label{thm:profile-lan}
Suppose the order-$k$ profiled-jet conditions in \Cref{cond:profile-jet} hold.
Fix $x\ne0$ and set $v_n=xn^{-1/(2k)}$.  Under
$P_{v_n,\lambda_0}^{\otimes n}$,
\[
 \Delta_{n,x,\lambda_0}
 :=v_n^k\sum_{i=1}^nS_{v_n,\lambda_0}(O_i)
 \Rightarrow N(0,x^{2k}I),
\]
and, uniformly over $\lambda\in\Lambda$,
\begin{equation}\label{eq:generic-lan}
 \sup_{\lambda\in\Lambda}\left|
 \log\frac{dP_{v_n,\lambda}^{\otimes n}}
          {dP_{v_n,\lambda_0}^{\otimes n}}
 -(\lambda-\lambda_0)\Delta_{n,x,\lambda_0}
 +\frac12(\lambda-\lambda_0)^2x^{2k}I
 \right|\longrightarrow0
\end{equation}
in probability.  Moreover, under the respective reference measures, the
likelihood-ratio process converges in $C(\Lambda)$ to that of
\begin{equation*}
 \mathcal E_x
 =\left\{Q_\lambda=N\paren{\lambda,\frac1{I x^{2k}}}:
 \lambda\in\Lambda\right\}.
\end{equation*}
Consequently
$\mathcal E_{n,x}=\{P_{v_n,\lambda}^{\otimes n}:\lambda\in\Lambda\}$
converges to $\mathcal E_x$ weakly in Le Cam's sense: every finite
restriction converges in Le Cam distance.  Finally,
\[
 \sup_{\lambda\in\Lambda}
 \norm{n^{1/(2k)}\tau(v_n,\lambda)-x\beta\lambda}_2
 \longrightarrow0.
\]
\end{theorem}  

The limiting decision problem is to estimate $x\beta\lambda$ from one observation $W\sim N\!\left(\lambda,\frac{1}{Ix^{2k}}\right)$ under squared-error loss. Its variance records the information remaining at the critical scale. We next derive the resulting local minimax risk.

\subsection{Exact local risk and coordinate invariance}
\label{sec:local-risk}

For fixed $x\ne0$, put $v_n=xn^{-1/(2k)}$ and, whenever the limit exists, define the local minimax risk by
\begin{equation}\label{eq:local-risk-def}
 \mathfrak R_x^{\mathrm{loc}}
 :=\lim_{n\to\infty}\inf_{T_n}\sup_{\lambda\in\Lambda}
 n^{1/k}\E_{P_{v_n,\lambda}^{\otimes n}}
 \norm{T_n-\tau(v_n,\lambda)}_2^2,
\end{equation}
where the infimum is over all measurable estimators of the target.

To express this risk, let, for $\sigma^2>0$,
\begin{equation}\label{eq:M-def}
 M(\sigma^2,\Lambda)
 =\inf_\delta\sup_{\lambda\in\Lambda}
   \E_{W\sim N(\lambda,\sigma^2)}\{\delta(W)-\lambda\}^2,
\end{equation} 
where the infimum is over measurable $\delta:\mathbb R\to\mathbb R$;
see \cite{bickel1981,casellaStrawderman1981,donohoLiu1991iii}. We also
define the standardized normal risk 
\begin{equation}\label{eq:RG-Kk}
 R_G(z)=\inf_\delta\sup_{\mu\in[0,z]}
 \E_{W\sim N(\mu,1)}\{\delta(W)-\mu\}^2,
 \qquad
 K_k=\sup_{z>0}z^{-2+2/k}R_G(z).
\end{equation}
For the critical experiment, $z=L\sqrt I\,|x|^k$, with $L=\diam\Lambda$.
The following proposition relates the exact local risk to this
one-dimensional decision problem and identifies its envelope over the
critical paths.

\begin{proposition}[Universal normal-envelope constant]
\label{prop:normal-envelope}
Let $L=\diam\Lambda$.  Under \Cref{thm:profile-lan}, the limit in
\eqref{eq:local-risk-def} exists and
\begin{equation}\label{eq:local-envelope}
 \mathfrak R_x^{\mathrm{loc}}
 =\norm{\beta}_2^2x^2M\paren{\frac1{Ix^{2k}},\Lambda}.
\end{equation}
Furthermore, $0<K_k<\infty$ and
\begin{equation}\label{eq:envelope-sup}
 \sup_{x\ne0}\norm{\beta}_2^2x^2M\paren{\frac1{Ix^{2k}},\Lambda}
 =\norm{\beta}_2^2 I^{-1/k}L^{2-2/k}K_k.
\end{equation}
Every global parameter space containing these local paths therefore has
normalized minimax liminf at least the right side of
\eqref{eq:envelope-sup}.
\end{proposition}
The model enters through $\norm{\beta}_2$, $I$ and $L$, while $K_k$
depends only on $k$. Optimizing over $x$ supplies the lower bound for
any global parameter space containing the critical paths.

For future use, we record an explicit solution of the bounded-normal
problem for sufficiently short parameter intervals.
\begin{remark}[An explicit minimax rule]
\label{rem:explicit-minimax-rule}
For $0<z\le2/\sqrt3$, the equal-endpoint prior
$\tfrac12\delta_0+\tfrac12\delta_z$ is least favourable for the bounded
normal problem defining $R_G(z)$.  Its Bayes rule
\begin{equation}\label{eq:normal-endpoint-rule}
\delta_z^{\rm ep}(w)
=\frac z2+\frac z2\tanh\left\{\frac z2\left(w-\frac z2\right)\right\}
\end{equation}
is minimax, with risk
\begin{equation}\label{eq:normal-endpoint-risk}
R_G(z)
=\E_{N(0,1)}{\delta_z^{\rm ep}(W)}^2
=\E_{N(z,1)}{\delta_z^{\rm ep}(W)-z}^2
=\sup_{\mu\in[0,z]}
\E_{N(\mu,1)}{\delta_z^{\rm ep}(W)-\mu}^2.
\end{equation}
If $z<2/\sqrt3$, only the two endpoints attain the supremum.
\end{remark}

Hence, on these intervals, the hardest parameter values in the
standardized problem are the two endpoints, and the rule in
\eqref{eq:normal-endpoint-rule} attains the exact risk in \eqref{eq:local-envelope}.

The risk constant in \eqref{eq:envelope-sup} is intrinsic to the local
statistical problem, although its individual components depend on the
choice of coordinates. We next verify this invariance for regular
reparametrizations.
\begin{proposition}[Regular-coordinate covariance of the jet and normal envelope]
\label{prop:coordinate-invariance}
Suppose the order-$k$ profiled-jet conditions in \Cref{cond:profile-jet} hold and
let $v=g(u)$ be a $C^1$ diffeomorphism defined for $|u|\le u_0$, with
$g(0)=0$ and $\alpha=g'(0)\ne0$.  Then
$\widetilde P_{u,\lambda}=P_{g(u),\lambda}$ and
$\widetilde\tau(u,\lambda)=\tau\{g(u),\lambda\}$ satisfy the same conditions
at order $k$, with $ \widetilde I=|\alpha|^{2k}I,$ and  $ \widetilde\beta=\alpha\beta.$
Thus, writing $L=\diam\Lambda$,
\begin{equation}\label{eq:coordinate-envelope}
 \norm{\widetilde\beta}_2^2\widetilde I^{-1/k}
 L^{2-2/k}K_k
 =\norm{\beta}_2^2I^{-1/k}L^{2-2/k}K_k.
\end{equation}
Likewise, if $\lambda=d_0+b\mu$, $b\ne0$, maps $\widetilde\Lambda$ onto
$\Lambda$, then, after removing the common fixed-$v$ target translation
$v\beta d_0$,
\begin{equation}\label{eq:nuisance-coordinate-scaling}
 I_\mu=b^2I,\qquad \beta_\mu=b\beta,
 \qquad \diam\widetilde\Lambda=L/|b|.
\end{equation}
Both regular coordinate changes leave the jet order and the factor in
\eqref{eq:coordinate-envelope} unchanged.
\end{proposition}

This proposition shows that the changes in information, target scale, and
the parametrization of the unidentified set cancel in the leading risk
constant. Thus, although $I$ and $\beta$ are coordinate-dependent
quantities, their combination in \eqref{eq:envelope-sup} is not. The
order of the local experiment and its leading minimax constant are
therefore invariant under regular changes of coordinates.

\begin{remark}[Regular versus singular coordinate changes]
\label{rem:regular-coordinate}
The restriction to regular coordinate changes is essential.  If
$g'(0)=0$, the transformation can change the leading exponent and the
transformed target may no longer have a first-order variation in the new
coordinate.  Such singular reparametrizations therefore fall outside the
invariance result in \Cref{prop:coordinate-invariance}.
\end{remark}

\section{Exact global minimax risk and uniform attainment}
\label{sec:globalization}

The local analysis identifies the minimax risk along each critical path. We now determine whether the largest of these local risks also governs the minimax risk over the full parameter space, including parameters away from the null parameter set. Under uniform efficiency, calibration, and recovery conditions, the answer is affirmative: the largest critical-path constant is the exact global minimax constant, and a single estimator attains it uniformly over $\Xi$.

\subsection{Uniform setup and the main theorem}

We first formulate a uniform version of the local setup. Fix $k\ge2$ and a
compact interval $\mathcal C\subset\mathbb R$. Let
$\Lambda(c)=[\ell(c),u(c)]$ be a compact interval and assume, for some
$C_\Lambda<\infty$,
$$
 |\ell(c)-\ell(c')|+|u(c)-u(c')|
 \le C_\Lambda|c-c'|,\qquad c,c'\in\mathcal C.
$$
Write $L(c)=\diam\Lambda(c)$. Here $c$ is the regular coordinate and $v$
describes departures from $v=0$. The parameter $\lambda$ is unidentified
at $v=0$ whenever $L(c)>0$.
Assume that $\Xi$ is a compact metric space and specialize to a bounded
Borel target $\tau:\Xi\to\mathbb R$. Assume that
$\xi\mapsto P_\xi$ is a Borel probability kernel. Use a Borel coordinate map
$\xi(c,v,\lambda)\in\Xi$, where $c\in\mathcal C$, $|v|\le v_0$, and
$\lambda\in\Lambda(c)$, and define $\mathcal F_c=\{\xi(c,0,\lambda):\lambda\in\Lambda(c)\}.$
Abbreviate
$$
 P_{c,v,\lambda}=P_{\xi(c,v,\lambda)},\qquad
 \E_{c,v,\lambda}=\E_{\xi(c,v,\lambda)},\qquad
 \tau(c,v,\lambda)=\tau\{\xi(c,v,\lambda)\}.
$$
Assume that a continuous coordinate $v_\circ:\Xi\to\mathbb R$ satisfies
$v_\circ\{\xi(c,v,\lambda)\}=v$, and that the chart contains
${\xi\in\Xi:|v_\circ(\xi)|\le v_0/2}$.
Its estimated magnitude will determine the weights assigned to the
candidate estimators.
Let $\widehat\tau_n:\mathcal O^n\to\mathbb R$ range over measurable
estimators and define
$$
 \mathcal R_n=\inf_{\widehat\tau_n}\sup_{\xi\in\Xi}
 \E_\xi\{\widehat\tau_n-\tau(\xi)\}^2.
$$
Assume that the likelihood conditions in \Cref{ass:profile-jet} hold on
every fixed-$c$ profile with $L(c)>0$, uniformly in $c$, with continuous
$I(c)>0$. For a continuous $\beta(c)\ne0$, assume $\tau(c,0,\lambda)=0,$ and  $\tau(c,v,\lambda)=v\beta(c)\lambda+o(|v|)$ uniformly in $(c,\lambda).$
Let
\begin{equation}\label{eq:general-exact-constant}
\mathsf C_k=K_k\sup_{c\in\mathcal C}
\beta(c)^2I(c)^{-1/k}L(c)^{2-2/k}.
\end{equation}

The quantity $\mathsf C_k$ is the largest local asymptotic minimax
constant over the regular coordinate $c$. The next theorem determines
whether this local benchmark remains sharp globally. Its lower bound
follows from the critical paths already analyzed, whereas the upper
bound requires a single procedure that remains uniformly efficient as
the parameter moves from the critical region to regions where ordinary
recovery is stable.

\begin{theorem}[Exact global minimax risk and uniform attainment]
\label{thm:global-lower-bound}
\label{thm:exact-globalization}
\label{thm:attaining-estimator}
Under the uniform setup above,
\begin{equation*}
 \liminf_{n\to\infty}n^{1/k}\mathcal R_n\ge\mathsf C_k.
\end{equation*}
If, in addition, conditions \textup{(E)}, \textup{(C)}, and \textup{(O)}
below hold, then
\begin{equation*}
 \lim_{n\to\infty}n^{1/k}\mathcal R_n=\mathsf C_k,
\end{equation*}
and the estimator constructed in \Cref{sec:computable-global-rule} attains this limit, with tuning sequences chosen as in the proof in the Supplementary Material.
\end{theorem}

\Cref{thm:global-lower-bound} shows that no parameter region outside the critical paths can increase the first-order minimax risk, provided the estimator can retain the local information near $v=0$, estimate the relevant coordinates without affecting the leading risk, and recover the target stably away from the null parameter set. We next state conditions that guarantee these three properties.

\smallskip
\noindent\textbf{(E) Efficient local statistic.}\par
\noindent There are observed features $\phi_{c,v}:\mathcal O\to\mathbb R$ such that
$(c,v,o)\mapsto\phi_{c,v}(o)$ is jointly measurable. Put
$\Gamma(c,v,\lambda)=\E_{c,v,\lambda}\{\phi_{c,v}(O)\}$.
For some $\zeta,\delta>0$, as $v\to0$ uniformly in
$c\in\mathcal C$ and $\lambda\in\Lambda(c)$,
\begin{equation*}
 \begin{aligned}
 \Gamma(c,v,\lambda)
 &=I(c)\lambda v^k+O(|v|^{k+\zeta}),\qquad
 \frac{\partial\Gamma(c,v,\lambda)}{\partial\lambda}
 =I(c)v^k+O(|v|^{k+\zeta}),\\
 \operatorname{Var}_{c,v,\lambda}\{\phi_{c,v}(O)\}
 &=I(c)+O(|v|^\zeta),\qquad\sup_{\substack{c\in\mathcal C,\ 0<|v|\le v_0\\
                  \lambda\in\Lambda(c)}}
 \E_{c,v,\lambda}
 \left|\phi_{c,v}(O)-\Gamma(c,v,\lambda)\right|^{2+\delta}<\infty.
 \end{aligned}
\end{equation*}
 
\smallskip
\noindent\textbf{(C) Negligible calibration error.}\par
\noindent A calibration sample $\mathcal D_{\rm cal}$ of size $m_n=o(n)$
yields $\mathcal D_{\rm cal}$-measurable estimates
$(\widehat c,\widehat v)\in\mathcal C\times[-v_0,v_0]$.
On the remaining $N_n=n-m_n$ independent observations
$O_1,\ldots,O_{N_n}$, define
\[
 Z_{N_n}(c,v,\lambda)
 :=\frac1{N_n}\sum_{i=1}^{N_n}\phi_{c,v}(O_i)-\Gamma(c,v,\lambda).
\]
There are deterministic $\varepsilon_n\downarrow0$ and
$\mathcal D_{\rm cal}$-measurable events $\mathcal H_n(c,v)$ such that
\[
 \begin{gathered}
 N_n^{1/(2k)}\varepsilon_n\to0,\qquad
 \mathcal H_n(c,v)\subseteq
 \{|\widehat c-c|+|\widehat v-v|\le\varepsilon_n\},\\
 \sup_{\xi\in\Xi}P_\xi^{\otimes m_n}
 \{|\widehat v-v_\circ(\xi)|>\varepsilon_n\}
 =o(N_n^{-1/k}),\qquad
 \sup_{\substack{c\in\mathcal C,\ |v|\le v_0\\
                   \lambda\in\Lambda(c)}}
 P_{c,v,\lambda}^{\otimes m_n}\{\mathcal H_n(c,v)^c\}
 =o(N_n^{-1/k}).
 \end{gathered}
\]
Let $\lambda_{\widehat c}=\Pi_{\Lambda(\widehat c)}(\lambda)$.
For every fixed $h>0$,
\begin{equation}\label{eq:calibration-orthogonality}
 \sup_{\substack{c\in\mathcal C,\ \lambda\in\Lambda(c)\\
                   hN_n^{-1/(2k)}\le |v|\le v_0}}
 N_n^{1/(2k)}|v|^{1-k}
 \left\|
  \bigl[Z_{N_n}(\widehat c,\widehat v,\lambda_{\widehat c})
        -Z_{N_n}(c,v,\lambda)\bigr]\one_{\mathcal H_n(c,v)}
 \right\|_2
 \longrightarrow0.
\end{equation}
where the $L^2$-norm is taken over both sample parts under
$P_{c,v,\lambda}^{\otimes n}$.

\smallskip
\noindent\textbf{(O) Stable recovery away from the null parameter set.}\par
\noindent For constants $0<a_O\le A_O<\infty$, the following hold.

\noindent\textit{(i) Within the chart,} derivatives at boundary or singleton profiles use differentiable local
extensions. Uniformly in $c\in\mathcal C$,
$\lambda\in\Lambda(c)$, and $0<|v|\le v_0$, the map
$\lambda\mapsto\Gamma(c,v,\lambda)$ is one-to-one and
\[
 \begin{aligned}
 a_O|v|^k&\le\left|\frac{\partial\Gamma(c,v,\lambda)}{\partial\lambda}\right|
 \le A_O|v|^k,\qquad
 \left|\frac{\partial\tau(c,v,\lambda)}{\partial\lambda}\right|\le A_O|v|,\\
 \left|\frac{\partial\tau(c,v,\lambda)}{\partial c}\right|
 &+\left|\frac{\partial\tau(c,v,\lambda)}{\partial v}\right|\le A_O.
 \end{aligned}
\]

The derivative in \textup{(i)} may have either sign.

\noindent\textit{(ii) On a compact outer region} $\Xi_O\subset\Xi$
containing $\{\xi:|v_\circ(\xi)|\ge v_0/4\}$, there is a measurable
statistic $\Phi_O:\mathcal O\to\mathcal H_O$, where $\mathcal H_O$ is a
separable Hilbert space, with continuous mean
$\Psi_O(\xi)=\E_\xi\{\Phi_O(O)\}$ satisfying
\begin{subequations}
\begin{align}
 |\tau(\xi)-\tau(\xi')|
 &\le A_O\|\Psi_O(\xi)-\Psi_O(\xi')\|_{\mathcal H_O},
 \qquad \xi,\xi'\in\Xi_O,
 \label{eq:off-fibre-inverse}\\
 \sup_{\xi\in\Xi_O}\E_\xi
 \|\Phi_O(O)-\Psi_O(\xi)\|_{\mathcal H_O}^2
 &\le A_O.
 \label{eq:off-fibre-variance}
\end{align}
\end{subequations}
These three conditions separate the requirements for efficient local estimation,
negligible calibration error, and stable recovery away from the null parameter
set. Condition~(E) ensures that an observable statistic retains the local
information needed for exact constant attainment: its leading signal and
variance match the information coefficient $I(c)$. This distinction between
valid and efficient estimating functions is also emphasized in
\cite{chernozhukovEtAl2018}. Condition~(C) ensures that estimation of the
regular and transverse coordinates does not affect the leading risk; in
particular, the weighted calibration error must be negligible at the critical
scale. Related requirements arise in the preliminary localization argument
of \cite{nussbaum1996equivalence} and the score-perturbation bounds of
\cite{chernozhukovEtAl2018}. Finally, Condition~(O) ensures stable recovery
away from the critical region. The derivative bounds control inversion within
the local chart, while the inverse inequality gives $O(n^{-1})$ risk in the
outer region. Thus, under (O), regions away from the critical paths do not
contribute to the first-order minimax risk.

\begin{remark} 
For the two models in Sections~5 and~6, we verify conditions \textup{(E)}--\textup{(O)} in the Supplementary Material. The calibration procedures used in these examples satisfy condition \textup{(C)} if
$$
\begin{aligned}
N_n^{4/5}\log N_n&=o(m_n)
&&\text{for the logistic model},\\
N_n^{1/2}\log N_n&=o(m_n)
&&\text{for the network model},
\end{aligned}
$$
together with \(m_n=o(n)\). These conditions are sufficient for the calibration constructions considered in the two examples.
\end{remark}

The same strategy can be implemented with other local statistics,
coordinate estimators, and recovery procedures. Failure of a particular
bound above does not by itself preclude attainment; what is required is
the corresponding approximation and uniform squared-risk control for the
model-specific construction.

\subsection{An attaining estimator}
\label{sec:computable-global-rule}
Conditions (E)--(O) provide the ingredients for uniform attainment. We now construct a measurable estimator that combines these ingredients across the parameter space. The construction uses the calibration estimates \((\widehat c,\widehat v)\), a calibrated local statistic in the critical region, and an empirical feature mean in the outer region. The weights depend on the estimated distance to the null parameter set and smoothly transition between the corresponding candidate estimators.

Use the first \(m_n=o(n)\) observations for calibration and let \(N_n=n-m_n\). The calibration sample yields \((\widehat c,\widehat v)\) satisfying condition \textup{(C)}. Relabel the remaining observations as \(O_1,\ldots,O_{N_n}\). We first construct a small collection of summary statistics that will be used in the different identification regimes. Define
\begin{equation*}
 \begin{aligned}
 \widehat\Gamma_{N_n}(\widehat c,\widehat v)
 &=\frac1{N_n}\sum_{i=1}^{N_n}\phi_{\widehat c,\widehat v}(O_i),\\
 \widehat W_n&=
 \begin{cases}
 \displaystyle\frac{\widehat\Gamma_{N_n}(\widehat c,\widehat v)}
 {I(\widehat c)\widehat v^k},&\widehat v\ne0,\\[6pt]
 0,&\widehat v=0,
 \end{cases}
 \qquad
 \widehat\Psi_{O,N_n}=\frac1{N_n}\sum_{i=1}^{N_n}\Phi_O(O_i).
 \end{aligned}
\end{equation*} 
where \(\phi_{c,v}\) and \(\Phi_O\) are the features in conditions \textup{(E)} and \textup{(O)}, respectively. We collect these quantities in
\begin{equation}\label{eq:computable-summary-statistics}
 T_n=(\widehat c,\widehat v,\widehat W_n,\widehat\Psi_{O,N_n}).
\end{equation}

The two calibration quantities that determine the routing are
$$
 \widehat J_n=N_nI(\widehat c)|\widehat v|^{2k},
 \qquad
 \widehat s_n=N_n^{1/(2k)}|\widehat v|.
$$
Here \(\widehat J_n\) is the estimated local information for \(\lambda\),
whereas \(\widehat s_n\) measures the distance from the critical scale.
The statistic \(\widehat W_n\) provides the observation for the local
Gaussian decision problem. It also contains the empirical moment needed
for inversion: whenever \(\widehat v\ne0\),
\(\widehat\Gamma_{N_n}(\widehat c,\widehat v) =I(\widehat c)\widehat v^k\widehat W_n.\)
Finally, \(\widehat\Psi_{O,N_n}\) is used for stable recovery in the outer
region.

We next define the candidate estimators associated with these three
regimes. For \(J>0\), let
\(\delta^{(\varepsilon_n^{\rm BN})}_{\Lambda(c),J}:\mathbb R\to\Lambda(c)\)
be the measurable family of bounded \(\varepsilon_n^{\rm BN}\)-minimax rules for
\[
 W\mid\lambda\sim N(\lambda,J^{-1}),\qquad\lambda\in\Lambda(c),
\]
selected as specified in the proof of \Cref{thm:exact-globalization}.
When \(\widehat v\ne0\) and \(L(\widehat c)>0\), the corresponding local
candidate is
\begin{equation*}
 \begin{aligned}
 \widehat\lambda_n^L
 =\delta^{(\varepsilon_n^{\rm BN})}_{\Lambda(\widehat c),\widehat J_n}
 (\widehat W_n),\qquad
 \widehat\tau_n^L=\tau(\widehat c,\widehat v,\widehat\lambda_n^L).
 \end{aligned}
\end{equation*}
This candidate is designed for the critical region, where the local
Gaussian experiment determines the first-order risk.

For stronger identification, we instead recover \(\lambda\) by inverting
the population moment map. Let \(\Pi_A\) denote projection onto \(A\).
When \(\widehat v\ne0\), \(L(\widehat c)>0\), and
\(|\widehat v|\le7v_0/16\), define
\begin{align*}
 \widehat\lambda_n^I
 &=\Gamma(\widehat c,\widehat v,\cdot)^{-1}\!\left[
 \Pi_{\Gamma\{\widehat c,\widehat v,\Lambda(\widehat c)\}}
 \widehat\Gamma_{N_n}(\widehat c,\widehat v)\right],
 \\
 \widehat\tau_n^I
 &=\tau(\widehat c,\widehat v,\widehat\lambda_n^I).
\end{align*}
The projection keeps the empirical moment within the range of the
moment map before inversion. Away from the critical region, condition
\textup{(O)} makes this inversion stable and therefore suitable for
first-order estimation.

For parameters in the outer region, let
\begin{equation*}
 \widehat\xi_{O,n}\in\argmin_{\zeta\in\Xi_O}
 \|\widehat\Psi_{O,N_n}-\Psi_O(\zeta)\|_{\mathcal H_O},\qquad
 \widehat\tau_n^O=\tau(\widehat\xi_{O,n}).
\end{equation*}
where a measurable minimizer as a function of the empirical feature mean
is fixed in advance. We also retain the
zero candidate $\widehat\tau_n^0=0,$ which is appropriate at the null parameter set.

The final estimator combines these candidates according to the estimated
identification regime. Choose \(0<2r_n<M_n\), with
\(r_n\downarrow0\) and \(M_n\uparrow\infty\), using the tuning sequences
chosen in the proof in the Supplementary Material. For \(a<b\), let $ \rho_{a,b}(t)=\min\{1,\max[0,(t-a)/(b-a)]\}.$
The transitions among the local candidates are determined by
\(\widehat s_n\), while the transition to the outer estimator is
determined directly by \(|\widehat v|\):
\begin{equation*}
 \begin{aligned}
 u_n=\rho_{r_n,2r_n}(\widehat s_n),\qquad
 g_n=\rho_{M_n,2M_n}(\widehat s_n),\qquad
 o_n=\rho_{3v_0/8,\,7v_0/16}(|\widehat v|),
 \end{aligned}
\end{equation*}
and
\begin{equation*}
 \begin{aligned}
 \omega_{0,n}&=(1-o_n)(1-u_n),&
 \omega_{L,n}&=(1-o_n)u_n(1-g_n),\\
 \omega_{I,n}&=(1-o_n)u_ng_n,&
 \omega_{O,n}&=o_n.
 \end{aligned}
\end{equation*}
These weights are nonnegative and sum to one. Thus, as the estimated
critical scale increases, the estimator moves from the zero candidate
to the local Gaussian rule and then to moment inversion; as
\(|\widehat v|\) enters the outer region, the outer estimator takes over.
Evaluate a candidate only when its weight is positive, and set
zero-weight contributions to zero.

Let \(B_\tau=\sup_{\xi\in\Xi}|\tau(\xi)|\). When \(L(\widehat c)>0\), the resulting estimator is the projected
weighted combination 
\begin{equation*}
 \widehat\tau_n=\Pi_{[-B_\tau,B_\tau]}\!\left(
 \sum_{j\in\{0,L,I,O\}}\omega_{j,n}\widehat\tau_n^j\right).
\end{equation*}
If \(L(\widehat c)=0\), the interval reduces to a singleton and we use
instead
\begin{equation*}
\widehat\tau_n=\Pi_{[-B_\tau,B_\tau]}
\left[
(1-o_n)\tau\{\widehat c,\widehat v,\ell(\widehat c)\}
+o_n\widehat\tau_n^O
\right].
\end{equation*}

The construction is designed so that no
single estimator needs to provide optimal recovery throughout the
parameter space: the local Gaussian rule is used where identification is
weakest, inversion is used once the moment map becomes sufficiently
informative, and the outer estimator handles parameters separated from
the null parameter set.

The construction above is intended to preserve the local minimax
constant while providing stable recovery throughout the full parameter
space. The following result establishes that the transitions between the
different candidates do not affect the first-order risk.

Under the uniform setup and conditions \textup{(E)}, \textup{(C)}, and
\textup{(O)} of \Cref{thm:exact-globalization}, with the decision rules
and tuning sequences chosen in the proof in the Supplementary Material,
the estimator constructed above satisfies
$$
 \lim_{n\to\infty}n^{1/k}\sup_{\xi\in\Xi}
 \E_\xi\{\widehat\tau_n-\tau(\xi)\}^2
 =\mathsf C_k.
$$

Thus the constructed estimator attains the same constant as the global
minimax lower bound. Statistically, this shows that the first-order
difficulty of the estimation problem is concentrated in the critical
region: once the estimator retains the local Gaussian information there,
the regions with stronger identification can be handled without
increasing the global minimax risk. The smooth routing construction
therefore converts the local optimal procedure into a single estimator
that is uniformly optimal over \(\Xi\).

Because the estimator is a measurable function of \(T_n\), the same
construction also permits a comparison of the full experiment with the
reduced experiment generated by these summary statistics. We formalize
this comparison in the following corollary.

\begin{corollary}[Minimax risk of the reduced experiment]
\label{cor:reduced-experiment-risk}
Under the uniform setup and conditions \textup{(E)}, \textup{(C)}, and
\textup{(O)} of \Cref{thm:exact-globalization}, let $T_n$ be defined by \eqref{eq:computable-summary-statistics}, with the calibration
choices satisfying condition \textup{(C)}. Consider the full
experiment and the reduced experiment
\[
 \mathcal E_n=\{P_\xi^{\otimes n}:\xi\in\Xi\},\qquad
 \mathcal F_n=\{\mathcal L_\xi(T_n):\xi\in\Xi\}.
\]
Their minimax risks for estimating $\tau(\xi)$ under squared-error loss
satisfy
\[
 \begin{aligned}
 \lim_{n\to\infty}n^{1/k}\inf_{d_n}\sup_{\xi\in\Xi}
 \E_\xi\{d_n(T_n)-\tau(\xi)\}^2
 &=\mathsf C_k,\\
 \lim_{n\to\infty}n^{1/k}
 \inf_{\widetilde\tau_n}\sup_{\xi\in\Xi}
 \E_\xi\{\widetilde\tau_n-\tau(\xi)\}^2
 &=\mathsf C_k,
 \end{aligned}
\]
where the first infimum ranges over all measurable decision rules based
on $T_n$, and the second over all measurable estimators based on the full
sample.  The estimator constructed above attains this common limit.
\end{corollary}

\Cref{cor:reduced-experiment-risk} shows that, for squared-error estimation of \(\tau\), the summary statistics \(T_n\) retain the exact first-order global minimax accuracy of the full sample. In particular, reducing the data to the calibration estimates, the local statistic, and the outer-region feature mean entails no first-order minimax loss. Numerical implementation details are provided in the Supplementary Material.

\section{A fifth-order logistic phase}\label{sec:model-rate}

We consider logistic regression with one additive surrogate and two unknown variance components. Profiling absorbs their lower-order effects, leaving a fifth-order distinction in the observed law. We determine the resulting exact global minimax risk and extend the result to sparse coefficient vectors.

\Needspace{7\baselineskip}
\subsection{Model settings}
\label{sec:logistic-order}

For \(p\ge1\), we observe \(n\) independent copies of \((Y,Z)\), where
\begin{equation*}
 X\sim N(0,aI_p),\qquad
 W\sim N(0,\eta I_p),\qquad
 Z=X+W,\qquad
 Y\mid X\sim\operatorname{Bernoulli}\{m(X^\top\theta)\},
\end{equation*}
where \(X\perp W\), \(Y\perp W\mid X\), \(m(t)=(1+e^{-t})^{-1}\), and
\begin{equation}\label{eq:variance-rect}
 a\in\cA=[a_-,a_+],\ \ 
 \eta\in\cE=[\eta_-,\eta_+],\ \ 
 0<a_-<a_+<\infty,\ \ 0<\eta_-<\eta_+<\infty.
\end{equation}
The parameter of interest is  \(\theta\), while both the signal variance \(a\) and the measurement-error variance \(\eta\) are unknown.

We first consider the scalar case \(p=1\). Set \(R=2Y-1\),
\(c=a+\eta\), and let \(\varphi_c\) denote the \(N(0,c)\) density.
Integrating out \(X\) gives the observed density
\begin{equation*}
 \begin{aligned}
 p_{\theta,a,\eta}(r,z)
 &=\frac{\varphi_c(z)}2\{1+r g(z\mid\theta,a,\eta)\},
 \qquad r\in\{-1,1\},\\
 g(z\mid\theta,a,\eta)
 &=\E_G\tanh\left[\frac{\theta}{2}
   \left\{\frac{az}{c}+\sqrt{\frac{a\eta}{c}}G\right\}\right],
 \qquad G\sim N(0,1).
 \end{aligned}
\end{equation*}

At \(\theta=0\), the observed density depends on \(a\) and \(\eta\)
only through their sum \(c\). Hence the variance decomposition is
unidentified on the null parameter set, even though the target itself is
zero. To study the local effect of this nonidentification, suppose that
some total variance admits two distinct interior decompositions, that is,  $(a_0,\eta_0)\ne(a_1,\eta_1)$ but $a_0+\eta_0=a_1+\eta_1=:c.$ 
The relevant question is then how rapidly the observed distribution
distinguishes these decompositions as the target moves away from zero.

We reparameterize the model using
\begin{equation*}
 c=a+\eta,\qquad
 \lambda=\frac ca,\qquad
 v=\frac{a\theta}{\sqrt c}\,
 \E_G\sech^2\left(\frac{\theta\sqrt a}{2}G\right),
 \qquad G\sim N(0,1).
\end{equation*}
The admissible ranges are
\[
 \cC=[a_-+\eta_-,a_++\eta_+],\qquad
 \Lambda(c)=\left[
 \frac{c}{\min\{a_+,c-\eta_-\}},
 \frac{c}{\max\{a_-,c-\eta_+\}}\right].
\]
Let \(U(v,\lambda)\) denote the local inverse at zero of
\(u\mapsto u\,\E_G\sech^2(\sqrt\lambda\,uG/2)\). Then
\begin{equation*}
 a=\frac c\lambda,\qquad
 \eta=c(1-\lambda^{-1}),\qquad
 \theta=\tau_L(c,v,\lambda)
 =\frac{\lambda U(v,\lambda)}{\sqrt c}
 =\frac{\lambda v}{\sqrt c}+O(v^3).
\end{equation*}
The coefficient-matching argument establishing these coordinates is given
in the Supplementary Material.

To identify the first informative likelihood order, let
\(P^{(c)}_{v,\lambda}\) denote the observed law and fix
\(\lambda_0\in\Lambda(c)\). Writing \(T=Z/\sqrt c\), we have
\begin{equation*}
 \frac{dP^{(c)}_{v,\lambda}}{dP^{(c)}_{v,\lambda_0}}(R,Z)-1
 =(\lambda-\lambda_0)v^5S_{c,v,\lambda_0}(R,Z)
 +r(R,Z\mid c,v,\lambda;\lambda_0).
\end{equation*}
where, with
\(g(t\mid v,\lambda)=\E_{c,v,\lambda}(R\mid T=t)\) and
\(H_3(t)=t^3-3t\),
\begin{equation*}
 \begin{aligned}
 S_{c,v,\lambda_0}(R,Z)
 =\frac{RH_3(T)}{96\{1+Rg(T\mid v,\lambda_0)\}},\qquad I(c)=\lim_{v\to0}\E_{c,v,\lambda_0}S_{c,v,\lambda_0}^{2}
 =I_5=\frac1{1536}.
 \end{aligned}
\end{equation*}
Thus the profiled-jet order is five on nondegenerate profiles. The expansion
and the required uniform remainder bounds are verified in the Supplementary
Material.

\subsection{Identification orders and minimax risk}
\label{sec:logistic-sparse}

For \(B>0\), define
\begin{equation*}
 \Xi_B=[-B,B]\times\cA\times\cE,\qquad
 \cR_n^{\mathrm{sc}}(B)
 =\inf_{T_n}\sup_{(\theta,a,\eta)\in\Xi_B}
   \E_{\theta,a,\eta}(T_n-\theta)^2.
\end{equation*}

The following result shows that the fifth-order profile also determines the global contact and inverse orders.

\begin{theorem}[Fifth-order identification]
\label{thm:logistic-omnibus}\label{thm:logistic-order}
Fix \(B>0\) and \(c\in\cC\) with nondegenerate \(\Lambda(c)\).
For the fixed-\(c\) profile and the global scalar experiment on \(\Xi_B\),
with target \(\theta\) and the observed feature \(\Phi_L\) defined in the
Supplementary Material,
\begin{equation*}
 k_{\mathrm J}=k_{\mathrm C}=k_{\mathrm I}=5.
\end{equation*}
\end{theorem}

Thus the local likelihood, contact and inverse orders give the same global squared-risk rate $n^{-1/5}$. We next identify the critical experiment and its exact local risk at $v_n=xn^{-1/10}$.

\begin{theorem}[Fifth-order limit experiment and local risk]
\label{thm:logistic-limit}
Fix \(c\in\cC\) with nondegenerate \(\Lambda(c)\).
For \(x\ne0\), put \(v_n=xn^{-1/10}\).  The profiled experiment converges
to
\begin{equation*}
 W_x\mid\lambda
 \sim N\left(\lambda,\frac{1536}{x^{10}}\right),
 \qquad \lambda\in\Lambda(c),
\end{equation*}
and its exact local target risk is
\begin{equation*}
 \lim_{n\to\infty}\inf_{T_n}\sup_{\lambda\in\Lambda(c)}
 n^{1/5}
 \E_{(P^{(c)}_{v_n,\lambda})^{\otimes n}}
 \{T_n-\tau_L(c,v_n,\lambda)\}^2
 =\frac{x^2}{c}
 M\left(\frac{1536}{x^{10}},\Lambda(c)\right).
\end{equation*}
\end{theorem}

Optimizing the critical-path risk gives the global benchmark
\begin{equation*}
C_5(\cA,\cE)
=1536^{1/5}K_5\sup_{c\in\cC}
\frac{{\diam\Lambda(c)}^{8/5}}{c}.
\end{equation*}
The following theorem shows that this critical-path benchmark is also
the exact global minimax constant.

\begin{theorem}[Exact fifth-order minimax constant]
\label{thm:exact-scalar}\label{thm:logistic-exact}
Fix \(B>0\).  The global scalar minimax risk satisfies
\begin{equation}\label{eq:exact-scalar}
 \lim_{n\to\infty}n^{1/5}\cR_n^{\mathrm{sc}}(B)
 =C_5(\cA,\cE).
\end{equation}
One measurable split-sample estimator attains this limit without knowing
\(a\) or \(\eta\).
\end{theorem}

With an unknown coefficient direction, estimation incurs a separate high-dimensional cost. For the sparse
case, let \(p=p_n\). For \(1\le s\le p/2\), define
\begin{align*}
 \Theta_{p,s}(B)
 &=\{\theta\in\R^p:
       \|\theta\|_0\le s,\ \|\theta\|_2\le B\},
 &J_p(s)&=s\log(ep/s),\\
 \cR^{\mathrm{joint}}_{n,p,s}(B)
 &=\inf_{T_n}\sup_{(a,\eta)\in\cA\times\cE}
   \sup_{\theta\in\Theta_{p,s}(B)}
   \E_{\theta,a,\eta}\|T_n-\theta\|_2^2.
\end{align*}

The next result separates weak identification from sparse direction estimation and specifies when the scalar constant is preserved.

\begin{theorem}[Sparse fifth-order phase]
\label{thm:logistic-sparse-omnibus}
\label{thm:rate}\label{thm:exact-sparse}
Fix \(B>0\), let \(p=p_n\), and let
\(1\le s_{+,n}\le p_n/2\).  If $$\frac{J_{p_n}(s_{+,n})}{n}\longrightarrow0,$$
then there are constants \(0<c<C<\infty\) such that, for all sufficiently
large \(n\),
\begin{equation}\label{eq:rate}
 c\left\{n^{-1/5}+\frac{J_{p_n}(s)}{n}\right\}
 \le \cR^{\mathrm{joint}}_{n,p_n,s}(B)
 \le C\left\{n^{-1/5}+\frac{J_{p_n}(s)}{n}\right\},
 \qquad 1\le s\le s_{+,n}.
\end{equation}
The upper bound is attained by one estimator that uses \(s_{+,n}\) and the
fixed parameter bounds but not the actual \(s,a,\) or \(\eta\).

If, more strongly,
\begin{equation*}
 J_{p_n}(s_{+,n})=o(n^{4/5}),
\end{equation*}
then the scalar constant is preserved uniformly:
\begin{equation}\label{eq:exact-sparse-risk}
 \sup_{1\le s\le s_{+,n}}
 \left|n^{1/5}\cR^{\mathrm{joint}}_{n,p_n,s}(B)
       -C_5(\cA,\cE)\right|
 \longrightarrow0.
\end{equation}
Moreover, one estimator
\(\widehat\theta_n^{\mathrm{sharp}}\), independent of the actual
\(s,a,\) and \(\eta\), satisfies
\begin{equation}\label{eq:exact-sparse-estimator}
 \limsup_{n\to\infty}
 \sup_{\substack{1\le s\le s_{+,n}\\
                  (a,\eta)\in\cA\times\cE\\
                  \theta\in\Theta_{p_n,s}(B)}}
 n^{1/5}
 \E_{\theta,a,\eta}
 \|\widehat\theta_n^{\mathrm{sharp}}-\theta\|_2^2
 \le C_5(\cA,\cE).
\end{equation}
\end{theorem}

The term $n^{-1/5}$ remains even when the direction is known. When $J_{p_n}(s_{+,n})=o(n^{4/5})$, direction estimation contributes lower-order risk, so the scalar constant survives uniformly over the allowed sparsity levels.

\subsection{An attaining estimator}
\label{sec:logistic-computation}

We specialize the general construction in
Section~\ref{sec:computable-global-rule} to the fifth-order logistic model,
with \(I(c)=1/1536\) and \(v_0=v_1\), where \(v_1>0\) is chosen small
enough that the chart lies in \(\Xi_B\). Using a calibration sample of size
\(m_n\), estimate the regular and transverse coordinates by
\begin{equation}\label{eq:calibration-c-v}
\widehat c=\Pi_{\cC}\left(\frac1{m_n}\sum_{i=1}^{m_n}Z_i^2\right),
\qquad
\widehat v=\Pi_{[-v_1,v_1]}\left\{
\frac2{m_n}\sum_{i=1}^{m_n}
\frac{R_iZ_i}{\sqrt{\widehat c}}\right\}.
\end{equation}
For the main sample, use the observable feature
\begin{equation*}
\phi_{c,v}(r,z)
=\frac{rH_3(z/\sqrt c)+v^3/4}{96},
\qquad
\Gamma_L(c,v,\lambda)
=\E_{c,v,\lambda}\phi_{c,v}(R,Z).
\end{equation*}
The empirical mean of \(\phi_{c,v}\) is used for local estimation and
inversion of the map \(\lambda\mapsto\Gamma_L(c,v,\lambda)\).
Let \(\widehat\lambda_n^L\), \(\widehat\lambda_n^I\), and the routing
weights \(\omega_{j,n}\) be defined as in
Section~\ref{sec:computable-global-rule}, using the exact map
\(\Gamma_L\). For \(L(\widehat c)>0\), the resulting estimator is
$$
\begin{aligned}
 \widehat\theta_n=\Pi_{[-B,B]}\big[&
 \omega_{L,n}\tau_L(\widehat c,\widehat v,\widehat\lambda_n^L)
 +\omega_{I,n}\tau_L(\widehat c,\widehat v,\widehat\lambda_n^I) +\omega_{O,n}(\widehat\xi_{O,L})_1\big].
\end{aligned}
$$
The singleton case is handled as in
Section~\ref{sec:computable-global-rule}. The outer minimum-distance
estimator, tuning choices, and verification of the required conditions are
given in the Supplementary Material.

\section{A sixth-order network phase}\label{sec:network}

We consider independent replicates of a network with a fixed number of vertices. Unknown block proportions conceal the edge contrast from lower-order terms; the first informative motif is a bowtie, formed by two triangles sharing a vertex. This gives a combinatorial counterpart to the fifth-order EIV model.

\subsection{Model settings}\label{sec:net-order}

For fixed \(M\ge5\), we observe \(n\) independent copies
\(A^{(1)},\ldots,A^{(n)}\) of an adjacency matrix \(A\). Thus, \(n\) is the
number of independent graph replicates, while the number of vertices \(M\)
remains fixed. The latent block labels \(Z_1,\ldots,Z_M\) are independent with $ \Pp(Z_i=1)=\Pp(Z_i=-1)=\frac q2,$ and  $\Pp(Z_i=0)=1-q.$
Conditional on the labels, the edges are independent and satisfy
\begin{equation}\label{eq:net-edge-model}
 A_{ij}\mid Z\sim\operatorname{Bernoulli}(p+\theta Z_iZ_j),
 \qquad 1\le i<j\le M.
\end{equation}
In the block order \((1,0,-1)\), the edge-probability matrix is
\begin{equation*}
 B=
 \begin{pmatrix}
  p+\theta&p&p-\theta\\
  p&p&p\\
  p-\theta&p&p+\theta
 \end{pmatrix},
\end{equation*}
and the target contrast is $\tau_{\rm net}=B_{11}-B_{1,-1}=2\theta.$ Both \(p\) and \(q\) are unknown. We assume $p\in\mathcal P=[p_-,p_+]\subset(0,1),$ $q\in[q_-,q_+],$ and  $0<q_-<q_+<1.$ Let \(P^M_{p,\theta,q}\) denote the marginal law of \(A\), and let \(P^M_{p,0}\) denote the independent-edge Bernoulli law with edge probability
\(p\). 

At \(\theta=0\), the observed law depends on \(p\) but not on \(q\). Thus,
\(p\) is identified under the null, whereas \(q\) indexes an unidentified
direction. The key question is how rapidly the observed distribution changes
along this direction when \(\theta\) moves away from zero.

We use the profile coordinates
\begin{equation*}
 v=q\theta,\qquad
 \lambda=q^{-1},\qquad
 q=\lambda^{-1},\qquad
 \theta=\lambda v,\qquad
 \tau_{\rm net}(p,v,\lambda)=2v\lambda.
\end{equation*}
with \(\Lambda=[q_+^{-1},q_-^{-1}]\). Choose \(v_{\max}>0\) such that
\(q_-^{-1}v_{\max}<\min(p_-,1-p_+)\), and define
\[
 \Xi_{\rm net}=\mathcal P\times[-v_{\max},v_{\max}]\times\Lambda.
\]
Let \(P^M_{p,v,\lambda}\) denote the corresponding observed law.

To identify the first informative term, let \(\mathcal E_M\) be the edge set
of the complete graph and let \(\mathcal B_M\) denote the collection of
bowties, each formed by two triangles sharing one vertex. Define
\begin{equation*}
 B_{M,p}(A)=\sum_{F\in\mathcal B_M}
 \prod_{e\in F}\frac{A_e-p}{p(1-p)}.
\end{equation*}
For a reference value \(\lambda_0\in\Lambda\), set
\[
 L_{p,v,\lambda_0}
 =\frac{dP^M_{p,v,\lambda_0}}{dP^M_{p,0}},
 \qquad S_{p,v,\lambda_0}=\frac{B_{M,p}}{L_{p,v,\lambda_0}}.
\]
Then
\begin{equation*}
 \begin{aligned}
 \frac{dP^M_{p,v,\lambda}}{dP^M_{p,v,\lambda_0}}(A)-1
 &=(\lambda-\lambda_0)v^6S_{p,v,\lambda_0}(A)
 +r(A\mid p,v,\lambda;\lambda_0),\\
 I(p)&=\lim_{v\to0}\E_{p,v,\lambda_0}S_{p,v,\lambda_0}^{2}
 =I_M(p)=15\binom M5\{p(1-p)\}^{-6}.
 \end{aligned}
\end{equation*}
Thus the profiled-jet order is six. The coefficient calculations and uniform
remainder bounds are provided in the Supplementary Material. Choose
\(v_0\in(0,v_{\max})\) so that these bounds hold on \(|v|\le v_0\).

\subsection{Identification orders and minimax risk}
\label{sec:net-exact}

The local expansion identifies a sixth-order mechanism, but it remains to determine whether the same order governs the global estimation problem. Since the graph space is finite, the entire marginal law can be represented by the finite collection of edge-monomial features
\begin{equation*}
 \begin{aligned}
 \Phi_{\rm net}(A)
 =\left(\prod_{e\in F}A_e\right)_{F\subseteq\mathcal E_M}
   \in\mathbb R^{2^{\binom M2}},\qquad\qquad
 \Psi_{\rm net}(p,v,\lambda)
 =\E_{p,v,\lambda}\Phi_{\rm net}(A).
 \end{aligned}
\end{equation*}
The map \(\Psi_{\rm net}\) therefore determines the full observed law. Define the global squared-error minimax risk by
\begin{equation*}
 \mathcal R_n^{\rm net}
 =\inf_{T_n}\sup_{(p,v,\lambda)\in\Xi_{\rm net}}
   \E_{p,v,\lambda}(T_n-2v\lambda)^2.
\end{equation*}
We first establish a global identification inequality and use it to determine the minimax rate.

\begin{theorem}[Sixth-order identification]
\label{thm:net-main}\label{thm:net-order}
For every fixed \(M\ge5\), there is \(C<\infty\) such that
\begin{equation*}
 |\tau_{\rm net}(\xi)-\tau_{\rm net}(\xi')|^6
 \le C\|\Psi_{\rm net}(\xi)-\Psi_{\rm net}(\xi')\|_2,
 \qquad \xi,\xi'\in\Xi_{\rm net}.
\end{equation*}
Consequently,
\begin{equation*}
 k_{\mathrm J}=k_{\mathrm C}=k_{\mathrm I}=6,
 \qquad \mathcal R_n^{\rm net}\asymp n^{-1/6}.
\end{equation*}
\end{theorem}

The three orders coincide, so the bowtie mechanism controls the global squared-risk rate. At the corresponding critical scale \(v_n=xn^{-1/12}\), the next result identifies the remaining information about \(\lambda\).

\begin{theorem}[Sixth-order limit experiment and local risk]
\label{thm:net-limit}
For fixed \(p\in\mathcal P\), \(x\ne0\), and
\(v_n=xn^{-1/12}\), the local experiment converges to
\begin{equation*}
 W_x\mid\lambda\sim
 N\left(\lambda,\{I_M(p)x^{12}\}^{-1}\right),
 \qquad \lambda\in\Lambda,
\end{equation*}
and its exact local normalized risk is
\begin{equation*}
 \lim_{n\to\infty}\inf_{T_n}\sup_{\lambda\in\Lambda}
 n^{1/6}\E_{(P^M_{p,v_n,\lambda})^{\otimes n}}
       (T_n-2v_n\lambda)^2
 =4x^2M\left(\{I_M(p)x^{12}\}^{-1},\Lambda\right).
\end{equation*}
\end{theorem}

Optimizing this risk over \(x\) and \(p\) gives the global benchmark.

Let  \(L=\diam\Lambda=q_-^{-1}-q_+^{-1}\), and define
\begin{equation}\label{eq:net-C6}
 C_6^{\rm net}
 =4K_6\left\{15\binom M5\right\}^{-1/6}
 L^{5/3}\sup_{p\in\mathcal P}p(1-p).
\end{equation}
The remaining question is whether this largest critical-path risk can be
attained uniformly when \(p\), \(q\), and \(\theta\) are all unknown. The next theorem answers this question and gives the exact global minimax constant.

\begin{theorem}[Exact sixth-order minimax constant]
\label{thm:net-exact}
For every fixed \(M\ge5\),
\begin{equation}\label{eq:net-exact-limit}
 \lim_{n\to\infty}n^{1/6}\mathcal R_n^{\rm net}=C_6^{\rm net}.
\end{equation}
One measurable estimator attains the limit without knowing \(p,q\), or
\(\theta\).
\end{theorem}

The critical-path benchmark is therefore globally sharp despite the unknown edge probability and block proportion.

\subsection{An attaining estimator}
\label{sec:network-computation}

We specialize the construction in Section~\ref{sec:computable-global-rule} to
the sixth-order network model, with \(c=p\) and information
\(I_M(p)\). We first use the calibration fold to estimate the edge
probability and the local coordinate. Specifically, define
\begin{align}
 \widehat p
 &=\Pi_{\mathcal P}\left\{
 \frac1{m_n\binom M2}\sum_{r=1}^{m_n}
 \sum_{e\in\mathcal E_M}A_e^{(r)}\right\},\notag\\
 \widetilde t
 &=\Pi_{[-v_{\max}^3,v_{\max}^3]}\left\{
 \frac1{m_n\binom M3}\sum_{r=1}^{m_n}
 \sum_{\{i,j,k\}}\prod_{e\in\{ij,jk,ki\}}
 (A_e^{(r)}-\widehat p)\right\},\notag\\
 \widetilde v
 &=\operatorname{sgn}(\widetilde t)|\widetilde t|^{1/3},
 \qquad
 \widehat v=\Pi_{[-v_0,v_0]}(\widetilde v).
 \label{eq:net-calibration}
\end{align}
Here \(\widehat v\) is used for the local construction, whereas
\(\widetilde v\) is retained for estimation in the outer region.

On the main fold, we estimate the sixth-order profile parameter using the
bowtie statistic. Let
\[
 \widehat\Gamma_{\rm net}
 =\frac1{N_n}\sum_{r=1}^{N_n}B_{M,\widehat p}(A^{(r)}),
 \qquad
 \Gamma_{\rm net}(p,v,\lambda)=I_M(p)\lambda v^6.
\]
The local and inverse estimates,
\(\widehat\lambda_n^L\) and \(\widehat\lambda_n^I\), together with their
routing weights, are constructed as in
Section~\ref{sec:computable-global-rule}. For the outer region, we use
a model-specific moment inversion that also retains \(\widetilde v\):
\[
 \widehat\lambda_n^O
 =\Pi_\Lambda\left\{
 \frac{\widehat\Gamma_{\rm net}}
 {I_M(\widehat p)\widetilde v^6}\right\}.
\]
With \(B_{\rm net}=2q_-^{-1}v_{\max}\), the final estimator combines
these three candidates according to the same routing rule:
\[
 \widehat\tau_n^{\rm net}=\Pi_{[-B_{\rm net},B_{\rm net}]}\left[
 2\omega_{L,n}\widehat v\widehat\lambda_n^L
 +2\omega_{I,n}\widehat v\widehat\lambda_n^I
 +2\omega_{O,n}\widetilde v\widehat\lambda_n^O\right].
\]
Candidate evaluation follows Section~\ref{sec:computable-global-rule}.
The moment-based outer candidate replaces the minimum-distance candidate
in that section. Its risk bound, the tuning choices, and the verification
of the required conditions are provided in the Supplementary Material.

\section{Numerical experiments}\label{sec:numerics-main}

We evaluate the implementations in
Sections~\ref{sec:logistic-computation} and~\ref{sec:network-computation}
at \(n=10^3,10^4,10^5,10^6\). Each setting uses 200 independent
split samples at each endpoint \(\lambda\in\{3/2,3\}\). We report the
larger empirical mean squared error, with pointwise 95\%
within-cell bootstrap intervals. The generating parameters are not
supplied to the rule. Supplement A gives the simulation laws, tuning,
numerical optimization and validation checks; Supplement B contains
executable code and all replication records.

For logistic regression, take \(\mathcal A=\mathcal E=[1,2]\),
\(B=4\), \(c=3\), and fixed profile coordinates \(v=0.25\) and \(0.65\).
For networks, take \(M=5\), \(p=1/2\),
\(\mathcal P=[0.49,0.51]\), \(\Lambda=[3/2,3]\), and
\(v=0.10\) and \(0.14\). The calibration size is
\(m=\lceil n^{9/10}\rceil\) for logistic regression and
\(m=\lceil n^{3/4}\rceil\) for networks. Graph counts are sampled from
the exact marginal law; logistic counts use Gaussian quadrature, with
separate checks based on direct continuous-data simulation. The estimator uses the prescribed routing weights and evaluates
candidates with positive weight.
The logistic setting $v=0.65$ includes recovery outside the local chart.
For logistic outer recovery, deterministic numerical refinement of the
minimum-distance criterion supplements the original finite grid.

\begin{figure}[t]
\centering
\includegraphics[page=10,width=\textwidth]{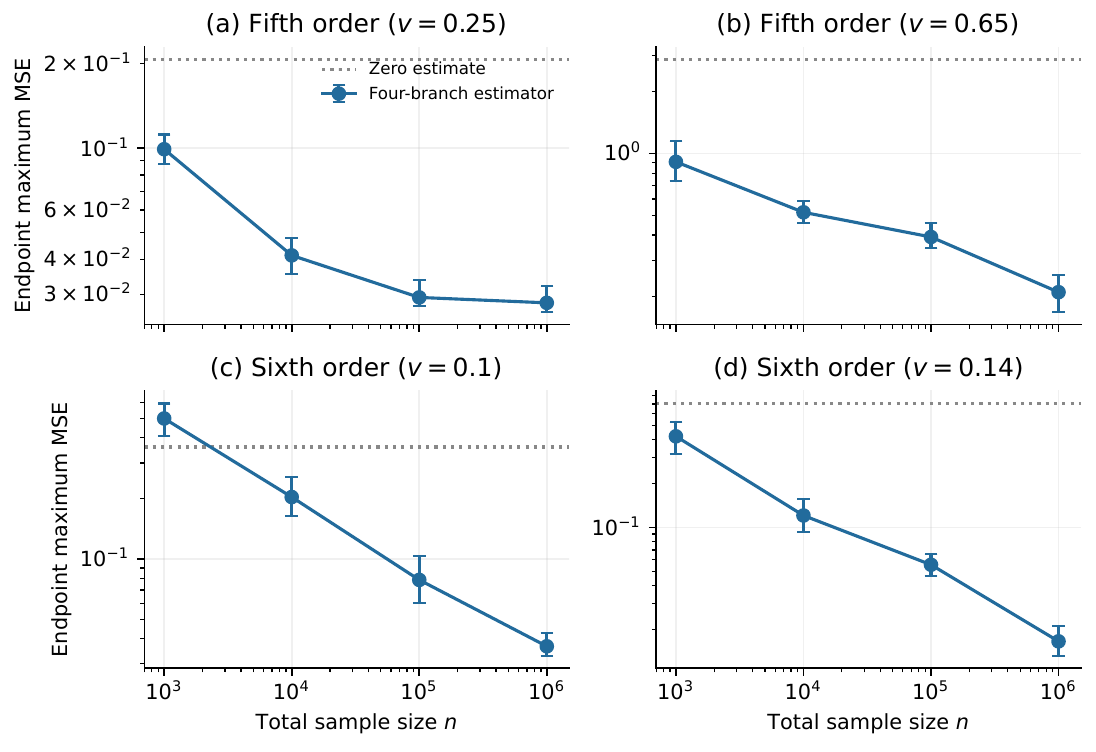}
\caption{Finite-sample endpoint risks at fixed profile coordinates.
Each point uses 200 replications at each endpoint. Bars are pointwise
95\% bootstrap intervals for the maximum of the two empirical risks.
Dotted lines give the corresponding endpoint risk of the zero estimate.
These curves assess finite-sample recovery, rather than convergence
along least-favourable critical paths.}
\label{fig:estimator-monte-carlo}
\end{figure}

Figure~\ref{fig:estimator-monte-carlo} separates weaker and stronger signals.
For networks at \(v=0.14\), endpoint risk decreases from \(0.4205\)
at \(n=10^3\) to \(0.0166\) at \(n=10^6\).
For logistic regression at \(v=0.25\), endpoint risk decreases from
\(0.0990\) to \(0.0279\), but identification remains weak. The divergence
\(\KL(P_{3,v,3/2}\Vert P_{3,v,3})\) is approximately
\(5.3019\times10^{-10}\) per observation, so even \(10^6\) observations
supply total divergence only \(0.0005\). At \(v=0.65\), the corresponding
total divergence is approximately \(2.8902\).

The fixed-coordinate experiments assess finite-sample estimation.
Supplement A also reports critical-path checks with joint calibration
and a separate observed-likelihood comparison for logistic regression.
At \(n=3\times10^{12}\), the equal-endpoint average risk divided by the
asymptotic benchmark is \(1.0388\) for the likelihood rule and \(1.0412\)
for the moment rule. A numerically evaluated two-point Bayes lower bound
already has ratio \(1.2059\) at \(n=10^9\), explaining why algorithmic
improvement alone cannot remove the fifth-order finite-sample discrepancy.
The sixth-order checks end at normalized grid risk \(0.9993\) with
\(n=10^8\). The likelihood rule is a numerical comparison, separate from
the global attainment theorem. The original models, targets and paths
are retained; finite-grid maxima do not represent full-space suprema.

\FloatBarrier

\section{Discussion}\label{sec:discussion}

The main result gives sufficient conditions under which the largest critical-path risk is the exact global minimax risk. The attaining estimator uses observable features and estimated profile coordinates, so its uniform risk accounts for calibration and recovery across identification regimes. For the stated target and squared-error loss, the resulting summary statistics retain the exact first-order global minimax accuracy of the full sample.

In the logistic model, the matching lower and upper bounds quantify the limitation on coefficient estimation imposed by an unknown variance decomposition. Sparsity reduces the direction-estimation cost but does not remove the scalar difficulty; its exact constant is preserved when the maximal sparse complexity is $o(n^{4/5})$. In the network model, triangles identify a combination of the block proportion and edge contrast, whereas the bowtie supplies the information needed to separate them. The exact risk quantifies the cost of this remaining uncertainty.

The efficiency, calibration and recovery conditions require model-specific verification; the likelihood expansion alone does not imply uniform attainment.

The theorem concerns a scalar unidentified coordinate on a compact interval, an affine leading target, a fixed finite identification order and squared-error loss. The number of vertices per network is fixed. Vector-valued unidentified parameters, growing networks and adaptation to an unknown order require further analysis of the local decision problem and uniform recovery.

\begin{acks}[Acknowledgements]
Generative AI tools were used to assist with literature searches, manuscript preparation, mathematical derivations, and numerical code development.
\end{acks}

\begin{supplement}
\stitle{Technical proofs and reproducibility materials}
\sdescription{Supplement A contains complete proofs, model verifications, and a self-contained specification of the numerical experiments. Supplement B contains code and machine-readable output for reproducibility.}
\end{supplement}

\clearpage
\setcounter{page}{1}
\setcounter{section}{0}
\setcounter{equation}{0}
\setcounter{figure}{0}
\setcounter{table}{0}
\setcounter{profilecondition}{0}
\setcounter{addressref}{0}
\setcounter{emailref}{0}
\renewcommand{\thesection}{S\arabic{section}}
\renewcommand{\thesubsection}{\thesection.\arabic{subsection}}
\crefname{assumption}{assumption}{assumptions}
\Crefname{assumption}{Assumption}{Assumptions}
\makeatletter
\global\authors@list{}
\global\keywords@list{}
\csundef{author@1@ead@marks}
\csundef{author@2@ead@marks}
\makeatother
\begin{frontmatter}
\pdfsubject{Statistics Theory preprint}
\pdftitle{Exact Global Minimax Risk near Nonidentification}
\pdfauthor{Kaiyue Yang and Yunan Wu}
\title{Supplement to ``Exact Global Minimax Risk near Nonidentification''}
\runtitle{Supplement: Exact Global Minimax Risk}

\begin{aug}
\author[supp-A]{\fnms{Kaiyue}~\snm{Yang}\ead[label=supp-e1]{yangky24@mails.tsinghua.edu.cn}}
\author[supp-B]{\fnms{Yunan}~\snm{Wu}\ead[label=supp-e2]{wuyunan@mail.tsinghua.edu.cn}} 
\address[supp-A]{Qiuzhen College, Tsinghua University, Haidian District, Beijing 100084, P. R. China\printead[presep={,\ }]{supp-e1}} 
\address[supp-B]{Yau Mathematical Sciences Center, Tsinghua University, Haidian District, Beijing 100084, P. R. China\printead[presep={,\ }]{supp-e2}} 
\end{aug}

\begin{abstract}
Proofs, model calculations, and numerical protocols for the fifth- and
sixth-order phases.
\end{abstract}
\end{frontmatter}

\section*{Guide to the supplement}
The table gives the proof location for each result of the main paper.
Restated results retain their supplementary numbering; their titles identify
the corresponding main-paper result.

\begin{center}
\begin{tabular}{@{}ll@{}}
\toprule
Main-paper result & Proof in this supplement \\
\midrule
Theorem~2.4 & \hyperref[supp:sec:supp-contact-inverse]{S2.1} \\
Theorem~2.5 & \hyperref[supp:sec:supp-order-hierarchy]{S2.2} \\
Theorem~3.1 & \hyperref[supp:sec:supp-jet-limit-experiment]{S1.2} \\
Proposition~3.2 & \hyperref[supp:sec:supp-local-minimax]{S2.3} \\
Proposition~3.3 & \hyperref[supp:sec:supp-coordinate-invariance]{S2.4} \\
Theorem~4.1 & \hyperref[supp:sec:supp-global-lower]{S3.1} (lower bound),
              \hyperref[supp:sec:supp-uniform-attainment]{S3.2} (attainment) \\
Corollary~4.2 & \hyperref[supp:sec:supp-reduced-experiment]{S3.3} \\
Theorem~5.1 & \hyperref[supp:sec:supp-logistic-jet]{S4.1}--\hyperref[supp:sec:supp-logistic-proofs]{S4.2} \\
Theorem~5.2 & \hyperref[supp:sec:supp-logistic-limit]{S4.3} \\
Theorem~5.3 & \hyperref[supp:sec:supp-calibration]{S4.4}--\hyperref[supp:sec:supp-scalar-exact]{S4.6} \\
Theorem~5.4 & \hyperref[supp:sec:supp-sparse-rate]{S5.1} (rate),
              \hyperref[supp:sec:exact-sparse]{S5.2} (exact constant) \\
Theorem~6.1 & \hyperref[supp:sec:net-order]{S6.1}--\hyperref[supp:sec:supp-net-orders]{S6.2} \\
Theorem~6.2 & \hyperref[supp:sec:supp-net-limit]{S6.3} \\
Theorem~6.3 & \hyperref[supp:sec:net-exact]{S6.4} \\
\bottomrule
\end{tabular}
\end{center}

\hyperref[supp:sec:supp-smooth-jet-origin]{Section~S1.1} gives the coordinate
construction used in Remark~2.1 of the main paper and the two model examples.
\hyperref[supp:sec:supp-logistic-jet]{Section~S4.1} and
\hyperref[supp:sec:net-order]{Section~S6.1} give the likelihood expansions and
information calculations stated in Sections~5.1 and~6.1 of the main paper.
 \hyperref[supp:sec:supp-computable]{Section~\ref*{supp:sec:supp-computable}} gives numerical
implementation details for Section~4.2, and
\hyperref[supp:sec:numerics]{Section~\ref*{supp:sec:numerics}} records the computations for Section~7.

\section{Profile coordinates and proof of Theorem~3.1}
\label{supp:sec:supp-profile-proof}
\subsection{Construction of profile coordinates}
\label{supp:sec:supp-smooth-jet-origin}
We compute local coordinates by matching coefficients in the observed
density expansion.  The construction applies when the equations below have
smooth solutions.  It supplies the coordinates described in the
smooth-origin remark (Remark~2.1) in the main paper.

\paragraph{Density expansion}
Consider a dominated two-parameter model with densities \((p_\vartheta)\) with
respect to \(\mu\).  Suppose that a regular \(C^{k+1}\) curve
\(\gamma:\widetilde\Lambda\to\Theta\subset\mathbb R^2\) satisfies
\[
 \gamma'(t)\ne0,\qquad p_{\gamma(t)}=p_0,
 \qquad t\in\widetilde\Lambda.
\]
Choose a smooth transverse field \(e(t)\) such that
\(\det\{e(t),\gamma'(t)\}\ne0\).  After restricting the neighborhood, the
map \((x,\eta)\mapsto\gamma(\eta)+xe(\eta)\) is a local coordinate chart.  Write the
observed density in this chart as \(p_{x,\eta}\), so that \(p_{0,\eta}=p_0\).

Let \(\mathbb B\) be a Banach space continuously embedded in \(L^1(\mu)\),
in which the required density derivatives exist.  Restrict
\(\widetilde\Lambda\) to a compact interval inside the chart, with all
functions defined on an open neighborhood.  For integers \(1\le r\le k\),
suppose uniformly on this interval that
\begin{equation*}
 p_{x,\eta}=p_0+\sum_{j=r}^k x^j f_j(\eta)+R(x,\eta),
 \qquad f_j\in C^{k-j+1}(\widetilde\Lambda;\mathbb B),
\end{equation*}
where
\begin{equation}\label{supp:eq:smooth-density-raw-remainder}
 \|R(x,\cdot)\|_{C^1(\widetilde\Lambda;\mathbb B)}=o(|x|^k),
 \qquad
 \sup_\eta\left\|\frac{\partial}{\partial x}R(x,\eta)\right\|_{\mathbb B}=o(|x|^{k-1}).
\end{equation}
These bounds follow, for example, from a joint density-valued
\(C^{k+1}\) expansion.

\paragraph{Coefficient equations}
Seek an exact polynomial coordinate map
\begin{equation}\label{supp:eq:smooth-coordinate-polynomials}
 x=B(v,t):=\sum_{i=1}^k b_i(t)v^i,\qquad
 \eta=t+A(v,t),\qquad
 A(v,t):=\sum_{i=1}^{k-1}a_i(t)v^i,
\end{equation}
where the coefficient functions are \(C^1\) and
\(\inf_t|b_1(t)|>0\).  Its Jacobian determinant at \(v=0\) is \(b_1(t)\).
The polynomial map therefore defines local coordinates for sufficiently
small \(|v|\).  Work on a smaller compact interval in \(t\), so that
\(t+A(v,t)\) stays in the interval where the bounds above hold.

For a power series \(H\), let \([v^m]H\) denote its coefficient of \(v^m\).
Taylor expansion in \(\eta\) around \(t\) gives
\begin{equation}\label{supp:eq:smooth-matched-density-expansion}
 p_{B(v,t),t+A(v,t)}
 =p_0+\sum_{m=r}^k v^mF_m(t)+\widetilde R(v,t),
\end{equation}
where the coefficients are computed by the finite formula
\begin{equation}\label{supp:eq:smooth-coefficient-formula}
 F_m(t)=\sum_{j=r}^m\sum_{s=0}^{m-j}
 \frac{f_j^{(s)}(t)}{s!}
 [v^m]\{B(v,t)^jA(v,t)^s\}.
\end{equation}
For example, when \(r=1\) and \(k\ge3\), the first three coefficients are
\begin{align*}
 F_1={}&b_1f_1,\\
 F_2={}&b_2f_1+b_1a_1f_1'+b_1^2f_2,\\
 F_3={}&b_3f_1+(b_2a_1+b_1a_2)f_1'
       +\tfrac12b_1a_1^2f_1'' \\
     &\quad+2b_1b_2f_2+b_1^2a_1f_2'+b_1^3f_3.
\end{align*}
All equalities are identities in \(\mathbb B\).

The coefficient-matching condition is that one can choose the functions in
\eqref{supp:eq:smooth-coordinate-polynomials} and fixed elements
\(\overline F_r,\ldots,\overline F_k\) and \(Q\) of \(\mathbb B\) such that
\begin{equation}\label{supp:eq:smooth-matching-condition}
 F_m(t)=\overline F_m\quad(r\le m<k),
 \qquad
 F_k(t)=\overline F_k+q_k(t)Q,
\end{equation}
where \(Q\ne0\), \(q_k\) is \(C^1\), and \(q_k'\ne0\) on the restricted
interval.  Their solvability is an additional condition on the model.  When \(r=1\)
and \(2\le m<k\), fixing the lower coefficients gives
\[
 F_m(t)=G_m(t)+b_m(t)f_1(t)+b_1(t)a_{m-1}(t)f_1'(t),
\]
where \(G_m\) is already determined.  Thus one solves for \(b_m\) and
\(a_{m-1}\) subject to
\[
 \overline F_m-G_m(t)\in
 \operatorname{span}\{f_1(t),b_1(t)f_1'(t)\},
\]
and checks that the resulting coefficient functions are \(C^1\).

When the density first changes at an order \(r>1\), the leading equation is
\begin{equation*}
 F_r(t)=b_1(t)^r f_r(t).
\end{equation*}
Consequently, when \(r<k\), making the leading coefficient common requires
\(f_r(t)=c_r(t)Q_r\) for a fixed nonzero direction \(Q_r\).  If \(c_r\)
does not vanish, orient \(Q_r\) so that \(c_r>0\) on the interval and take
\[
 b_1(t)=c_r(t)^{-1/r},\qquad F_r(t)=Q_r.
\]
The fixed-direction requirement is an additional condition on the model.

The remainder in \eqref{supp:eq:smooth-matched-density-expansion} has the required
differentiability in \(t\). Indeed, \(B=O(v)\),
\(\frac{\partial B}{\partial t}=O(v)\), \(A=O(v)\), and
\(\frac{\partial A}{\partial t}=O(v)\). Hence
\[
 \frac{\partial}{\partial t}R\{B(v,t),t+A(v,t)\}
 =\frac{\partial R}{\partial x}\frac{\partial B}{\partial t}
  +\frac{\partial R}{\partial\eta}
    \left(1+\frac{\partial A}{\partial t}\right)=o(|v|^k),
\]
where the derivatives of \(R\) are evaluated at
\((B(v,t),t+A(v,t))\). This follows from
\eqref{supp:eq:smooth-density-raw-remainder}; the finite Taylor remainders in
\eqref{supp:eq:smooth-coefficient-formula} obey the same bound.  Therefore
\begin{equation*}
 \|\widetilde R(v,\cdot)\|_{C^1(\widetilde\Lambda;\mathbb B)}=o(|v|^k).
\end{equation*}

\paragraph{Likelihood and target}
It remains to express the target in the same coordinates.  Suppose the
target is zero on the null parameter set and, for continuous
\(\tau_1\), uniformly in \(\eta\),
\begin{equation*}
 \tau(x,\eta)=x\tau_1(\eta)+o(|x|).
\end{equation*}
Then
\[
 \tau\{B(v,t),t+A(v,t)\}=v b_1(t)\tau_1(t)+o(|v|).
\]
Assume that, for some \(\beta\in\mathbb R^d\setminus\{0\}\) and scalars
\(d_0\) and \(d_1\ne0\),
\begin{equation*}
 b_1(t)\tau_1(t)=\beta\{d_0+d_1q_k(t)\}.
\end{equation*}
Set \(\lambda=d_0+d_1q_k(t)\) and restrict the interval so that this is a
\(C^1\) coordinate.  Write \(t(\lambda)\) for its inverse, and
\(p_{v,\lambda}\) and \(\tau(v,\lambda)\) for the density and target at
\((B(v,t(\lambda)),t(\lambda)+A(v,t(\lambda)))\).
For \(t_0=t(\lambda_0)\),
\begin{align*}
 p_{v,\lambda}-p_{v,\lambda_0}
 &=(\lambda-\lambda_0)v^k\frac{Q}{d_1}
   +\widetilde R\{v,t(\lambda)\}-\widetilde R(v,t_0),
 \\
 \tau(v,\lambda)&=v\beta\lambda+o(|v|).
\end{align*}

Because both sides of \eqref{supp:eq:smooth-matching-condition} are density
coefficients, integration and \(q_k'\ne0\) give \(\int Q\,d\mu=0\).
Assume, for sufficiently small \(|v|\), that
\[
 P_{v,\lambda}\ll P_{v,\lambda_0},\qquad
 |Q|\,d\mu\ll P_{v,\lambda_0}.
\]
Division by the reference density then gives the exact representation
\begin{equation*}
 \frac{dP_{v,\lambda}}{dP_{v,\lambda_0}}(O)-1
 =(\lambda-\lambda_0)v^kS_{v,\lambda_0}(O)+r(O\mid v,\lambda),
 \qquad
 S_{v,\lambda_0}(O)=\frac{Q(O)/d_1}{p_{v,\lambda_0}(O)}.
\end{equation*}
Here \(\E_{v,\lambda_0}S_{v,\lambda_0}=\int Q/d_1\,d\mu=0\).  The remainder is
\[
 r(O\mid v,\lambda)
 =\frac{\widetilde R\{v,t(\lambda)\}(O)-\widetilde R(v,t_0)(O)}
        {p_{v,\lambda_0}(O)}.
\]
To obtain the profiled-jet conditions, one also verifies convergence of
\(\E_{v,\lambda_0}S_{v,\lambda_0}^2\) to a
positive finite limit, a uniform fourth moment, and the divided
\(C^1\)--\(L^2\) remainder bound.  These bounds are proved for the two
models below in the corresponding model calculations.

\paragraph{Fifth-order model}
For the scalar logistic measurement-error model, fix \(c=a+\eta\), and
write \(u=a\theta/\sqrt c\) and \(\lambda=c/a\).  The observed regression
has the following expansion in the probabilists' Hermite polynomials:
\begin{align*}
 \widetilde g_{u,\lambda}
 ={}&\left(\frac u2-\frac{\lambda u^3}{8}
                +\frac{\lambda^2u^5}{16}\right)H_1
 +\left(-\frac{u^3}{24}+\frac{\lambda u^5}{24}\right)H_3
 +\frac{u^5}{240}H_5+O(u^7).
\end{align*}
For \(T\sim N(0,1)\), Gaussian integration by parts gives the exact
coordinate
\[
 v=2\E\{T\widetilde g_{u,\lambda}(T)\}
  =u\,\E_G\sech^2\left(\frac{u\sqrt\lambda}{2}G\right)
  =\frac{a\theta}{\sqrt c}\,
    \E_G\sech^2\left(\frac{\theta\sqrt a}{2}G\right),
 \qquad G\sim N(0,1).
\]
Put \(u=v+A_3(\lambda)v^3+A_5(\lambda)v^5+O(v^7)\).  Making the
\(H_1\) coefficient equal to \(v/2\) through fifth order gives
\[
 \frac{A_3}{2}-\frac\lambda8=0,
 \qquad
 \frac{A_5}{2}-\frac{3\lambda A_3}{8}+\frac{\lambda^2}{16}=0,
\]
and hence \(A_3=\lambda/4\) and \(A_5=\lambda^2/16\).  Substitution in the
\(H_3\) coefficient gives
\[
 -\frac{v^3}{24}+\frac{\lambda v^5}{96}+O(v^7).
\]
The exact coordinate \(u=U(v,\lambda)\) is defined by making the observed
\(H_1\) coefficient exactly \(v/2\); the implicit-function theorem applies
because the derivative of this coefficient with respect to \(u\) is nonzero
at \(u=0\).  Its finite jet is the one
just computed, and the target satisfies
\[
 \theta=\frac{\lambda U(v,\lambda)}{\sqrt c}
       =\frac{\lambda v}{\sqrt c}+O(v^3).
\]
The uniform remainder and likelihood conditions are proved
in \Cref{supp:lem:profile-expansion,supp:lem:fifth-jet}.

\paragraph{Sixth-order model}
For the network model, the raw observed-density coefficient associated with
an Eulerian edge set \(F\) is
\[
 \theta^{|F|}q^{|V(F)|}H_{F,p}(A),\qquad
 H_{F,p}(A)=\prod_{e\in F}\frac{A_e-p}{p(1-p)}.
\]
The leading triangles have \(|F|=|V(F)|=3\).  Equation
\(q^3b_1(q)^3=1\) gives \(b_1(q)=q^{-1}\), and the exact coordinate is
\[
 \theta=\frac vq,\qquad \lambda=\frac1q.
\]
Every nonempty Eulerian term with at most five edges is a cycle and therefore has a
coefficient independent of \(q\) after this substitution.  At six edges, a
bowtie has \(|F|=6\) and \(|V(F)|=5\), and its coefficient becomes
\[
 \theta^6q^5=\frac{v^6}{q}=\lambda v^6.
\]
Thus the first density coefficient depending on \(\lambda\) is the bowtie
coefficient at order six,
while the target is \(2\theta=2\lambda v\).  The exact density remainder and
moment conditions are verified in \Cref{supp:lem:net-jet}.

\subsection{Proof of Theorem~3.1}
\label{supp:sec:supp-jet-limit-experiment}
Let $\Lambda=[\ell,u]$, $O_1,\ldots,O_n\stackrel{\rm iid}{\sim}
P_{v,\lambda}$ on a standard Borel space, where $P_{0,\lambda}$ is
independent of $\lambda$.

For $f:\Lambda\to\R$, write
\[
 \norm f_{\Lambda}=\sup_{\lambda\in\Lambda}|f(\lambda)|,
 \qquad
 \norm f_{C^1(\Lambda)}
 =\norm f_{\Lambda}+\sup_{\lambda\in\Lambda}|f'(\lambda)|.
\]

\begin{assumption}[Uniform profiled likelihood jet]\label{supp:ass:profile-jet}
Fix $\lambda_0\in\Lambda$ and an integer $k\ge1$.  All limits below are as
$v\to0$.  For all sufficiently small $|v|$, suppose
$P_{v,\lambda}\ll P_{v,\lambda_0}$ for every $\lambda\in\Lambda$ and choose
jointly measurable versions of the Radon--Nikodym derivatives such that,
as an equality in $C(\Lambda)$ under $P_{v,\lambda_0}$,
\begin{equation}\label{supp:eq:ratio-jet}
 \frac{dP_{v,\lambda}}{dP_{v,\lambda_0}}(O)-1
 = (\lambda-\lambda_0)v^kS_{v,\lambda_0}(O)+r(O\mid v,\lambda),
\end{equation}
where $(v,o)\mapsto S_{v,\lambda_0}(o)$ and
$(v,o,\lambda)\mapsto r(o\mid v,\lambda)$ are jointly measurable,
$\lambda\mapsto r(o\mid v,\lambda)$ is continuously differentiable,
and, under $P_{v,\lambda_0}$,
\begin{align}
 &\E_{v,\lambda_0}S_{v,\lambda_0}=0,\qquad
   \E_{v,\lambda_0}S_{v,\lambda_0}^2\longrightarrow I\in(0,\infty),\qquad
   \sup_{0<|v|\le v_0}\E_{v,\lambda_0}|S_{v,\lambda_0}|^4<\infty,
   \label{supp:eq:score-cond}\\
 &\E_{v,\lambda_0}\left[
   \sup_{\lambda\in\Lambda}|r(O\mid v,\lambda)|^2
   +\sup_{\lambda\in\Lambda}\left|
      \frac{\partial}{\partial\lambda}r(O\mid v,\lambda)\right|^2
   \right]=o(|v|^{2k}).\label{supp:eq:rem-cond}
\end{align}
Finally, for some $\beta\in\R^d\setminus\{0\}$, suppose
\begin{equation}\label{supp:eq:target-jet}
 \sup_{\lambda\in\Lambda}
 \frac{\norm{\tau(v,\lambda)-v\beta\lambda}_2}{|v|}
 \longrightarrow0.
\end{equation}
\end{assumption}

The order-$k$ likelihood contrast accumulates information of constant
order when $n|v|^{2k}\asymp1$.  The following theorem identifies the
Gaussian location model that describes inference at this scale.

\begin{theorem}[Theorem~3.1 of the main paper]\label{supp:thm:profile-lan}
Under \Cref{supp:ass:profile-jet}, fix $x\ne0$ and set $v_n=xn^{-1/(2k)}$.
\Needspace{12\baselineskip}
Under $P_{v_n,\lambda_0}^{\otimes n}$,
\[
 \Delta_{n,x,\lambda_0}
 :=v_n^k\sum_{i=1}^nS_{v_n,\lambda_0}(O_i)
 \Rightarrow N(0,x^{2k}I),
\]
and
\begin{equation}\label{supp:eq:generic-lan}
 \sup_{\lambda\in\Lambda}\left|
 \log\frac{dP_{v_n,\lambda}^{\otimes n}}
          {dP_{v_n,\lambda_0}^{\otimes n}}
 -(\lambda-\lambda_0)\Delta_{n,x,\lambda_0}
 +\frac12(\lambda-\lambda_0)^2x^{2k}I
 \right|\longrightarrow0
\end{equation}
in probability.  Moreover, under the respective reference measures, the
likelihood-ratio process converges in $C(\Lambda)$ to that of
\begin{equation*}
 \mathcal E_x
 =\left\{Q_\lambda=N\paren{\lambda,\frac1{I x^{2k}}}:
 \lambda\in\Lambda\right\}. 
\end{equation*}
Consequently
$\mathcal E_{n,x}=\{P_{v_n,\lambda}^{\otimes n}:\lambda\in\Lambda\}$
converges to $\mathcal E_x$ weakly in Le Cam's sense: every finite
restriction converges in Le Cam distance.  Finally,
\[
 \sup_{\lambda\in\Lambda}
 \norm{n^{1/(2k)}\tau(v_n,\lambda)-x\beta\lambda}_2
 \longrightarrow0.
\]
\end{theorem}
\Cref{supp:thm:profile-lan} shows that the data retain finite information about
$\lambda$ at $|v|\asymp n^{-1/(2k)}$, while the target varies
on the same scale; the corresponding estimation risk is determined below.

\begin{proof}
Write $P_{n,\lambda}=P_{v_n,\lambda}^{\otimes n}$ and, within this proof,
put
\[
 q_{n,i}(\lambda)
 =\frac{dP_{v_n,\lambda}}{dP_{v_n,\lambda_0}}(O_i)-1,
 \qquad
 r_{n,i}(\lambda)=r(O_i\mid v_n,\lambda).
\]
Let $D=\sup_{\lambda\in\Lambda}|\lambda-\lambda_0|$ and
$L=\diam\Lambda$.  A Radon--Nikodym derivative has expectation one, so
\eqref{supp:eq:ratio-jet} and \eqref{supp:eq:score-cond} imply
$\E_{v,\lambda_0}r(O\mid v,\lambda)=0$ for every $\lambda$.  Also
$r(O\mid v,\lambda_0)=0$.  The derivative envelope in \eqref{supp:eq:rem-cond} permits
differentiation under the expectation, and hence
$\E_{v,\lambda_0}\{\frac{\partial}{\partial\lambda}r(O\mid v,\lambda)\}=0$.  The fundamental
theorem of calculus, independence, and Cauchy--Schwarz now give
\begin{align}
 \E_{P_{n,\lambda_0}}
 \norm{\sum_{i=1}^n r_{n,i}}_{\Lambda}
 &\le
 \int_\Lambda
 \left\{n\E_{v_n,\lambda_0}
       \left\{\frac{\partial}{\partial\lambda}r(O\mid v_n,\lambda)\right\}^2\right\}^{1/2}d\lambda\nonumber\\
 &\le L\sqrt n\,
 \left\{\E_{v_n,\lambda_0}
       \norm{r(O\mid v_n,\cdot)}_{C^1(\Lambda)}^2\right\}^{1/2}
 =o(1). \label{supp:eq:uniform-linear-rem}
\end{align}
Similarly, if
\[
 C_n=\sum_{i=1}^n\norm{r_{n,i}}_{\Lambda}^2,
\]
then
\begin{equation*}
 \E_{P_{n,\lambda_0}}C_n
 \le n\E_{v_n,\lambda_0}
       \norm{r(O\mid v_n,\cdot)}_{C^1(\Lambda)}^2=o(1).
\end{equation*}
Here we used $n|v_n|^{2k}=|x|^{2k}$.

The uniform fourth-moment bound gives the triangular-array Lyapunov
condition because
\[
 n|v_n|^{4k}\E_{v_n,\lambda_0}|S_{v_n,\lambda_0}|^4=O(n^{-1}).
\]
Consequently
\begin{equation}\label{supp:eq:generic-central-clt}
 \Delta_{n,x,\lambda_0}\Rightarrow Z_x,
 \qquad Z_x\sim N(0,x^{2k}I).
\end{equation}
The same moment bound shows that
\begin{equation*}
 J_n:=v_n^{2k}\sum_{i=1}^nS_{v_n,\lambda_0}(O_i)^2
 \longrightarrow x^{2k}I
\end{equation*}
in probability: its expectation has this limit and
$\operatorname{Var}(J_n)\le
n|v_n|^{4k}\E S_{v_n,\lambda_0}^4=O(n^{-1})$.

Equation \eqref{supp:eq:ratio-jet} and \eqref{supp:eq:uniform-linear-rem} yield
\begin{equation}\label{supp:eq:generic-linear-uniform}
 \norm{\sum_{i=1}^nq_{n,i}(\cdot)
 -(\cdot-\lambda_0)\Delta_{n,x,\lambda_0}}_{\Lambda}
 \longrightarrow0
\end{equation}
in probability.  For the quadratic term,
\begin{align}
 &\norm{\sum_{i=1}^nq_{n,i}(\cdot)^2
 -(\cdot-\lambda_0)^2J_n}_{\Lambda}\nonumber\\
 &\qquad\le2D J_n^{1/2}C_n^{1/2}+C_n
 \longrightarrow0. \label{supp:eq:generic-quadratic-uniform}
\end{align}
For the logarithmic remainder, set
$M_n=\max_{1\le i\le n}\norm{q_{n,i}}_{\Lambda}$.  By
\eqref{supp:eq:ratio-jet},
\[
 M_n\le D|v_n|^k\max_{1\le i\le n}|S_{v_n,\lambda_0}(O_i)|+C_n^{1/2}.
\]
For every $\varepsilon>0$, the union bound and the fourth-moment condition
give
\[
 P_{n,\lambda_0}\left(
 |v_n|^k\max_i|S_{v_n,\lambda_0}(O_i)|>\varepsilon\right)
 \le\frac{n|v_n|^{4k}\E|S_{v_n,\lambda_0}|^4}{\varepsilon^4}
 =O(n^{-1}).
\]
Thus $M_n=o_{P_{n,\lambda_0}}(1)$.  Moreover,
\[
 \sum_{i=1}^n\norm{q_{n,i}}_{\Lambda}^2
 \le2D^2J_n+2C_n=O_{P_{n,\lambda_0}}(1),
\]
and therefore
\begin{equation}\label{supp:eq:generic-cubic-uniform}
 \sum_{i=1}^n\norm{q_{n,i}}_{\Lambda}^3
 \le M_n\sum_{i=1}^n\norm{q_{n,i}}_{\Lambda}^2
 \longrightarrow0
\end{equation}
in probability.

On the event $A_n=\{M_n\le1/2\}$, whose probability tends to one, every
factor $1+q_{n,i}(\lambda)$ is at least $1/2$, simultaneously in $i$ and
$\lambda$.  Hence the logarithm is finite there, and
\[
 |\log(1+y)-y+\tfrac12y^2|\le C|y|^3,
 \qquad |y|\le\tfrac12.
\]
The product rule for Radon--Nikodym derivatives, together with
\eqref{supp:eq:generic-linear-uniform},
\eqref{supp:eq:generic-quadratic-uniform}, and
\eqref{supp:eq:generic-cubic-uniform}, proves \eqref{supp:eq:generic-lan}; if a
logarithm is not finite, the supremum in that display is understood as
$+\infty$, an event whose probability tends to zero.

To prove convergence of experiments, let an auxiliary log-likelihood
process equal the actual one on $A_n$ and the zero function on $A_n^c$.
It is a measurable $C(\Lambda)$-valued random element.  The map
\[
 z\longmapsto
 \left\{(\lambda-\lambda_0)z
 -\tfrac12(\lambda-\lambda_0)^2x^{2k}I:
 \lambda\in\Lambda\right\}
\]
is continuous from $\R$ to the separable Banach space $C(\Lambda)$.
Equations \eqref{supp:eq:generic-lan} and \eqref{supp:eq:generic-central-clt}, followed
by the continuous mapping theorem for exponentiation on $C(\Lambda)$, show
that the exponential of the auxiliary process converges to the exponential
of this limit.
\Needspace{9\baselineskip}
It equals the actual likelihood-ratio process on $A_n$, so
the same convergence holds for the actual process.  If
\[
 W_x=\lambda_0+\frac{Z_x}{x^{2k}I}\sim Q_{\lambda_0},
\]
then
\begin{equation*}
 \frac{dQ_\lambda}{dQ_{\lambda_0}}(W_x)
 =\exp\left\{(\lambda-\lambda_0)x^{2k}I(W_x-\lambda_0)
 -\frac12(\lambda-\lambda_0)^2x^{2k}I\right\}.
\end{equation*}
Thus, under the respective reference measures,
\begin{equation}\label{supp:eq:lr-process}
 \left\{
 \frac{dP_{n,\lambda}}{dP_{n,\lambda_0}}:
 \lambda\in\Lambda\right\}
 \Rightarrow
 \left\{\frac{dQ_\lambda}{dQ_{\lambda_0}}(W_x):
 \lambda\in\Lambda\right\}
\end{equation}
in $C(\Lambda)$.

To pass from the reference measure to the finite experiments, note from
\eqref{supp:eq:ratio-jet}--\eqref{supp:eq:rem-cond} that
\begin{align*}
 \sup_{\lambda\in\Lambda}\E_{v,\lambda_0}
 \left(\frac{dP_{v,\lambda}}{dP_{v,\lambda_0}}-1\right)^2
 &\le2D^2|v|^{2k}\E_{v,\lambda_0}S_{v,\lambda_0}^2
 +2\E_{v,\lambda_0}\norm{r(O\mid v,\cdot)}_{C^1(\Lambda)}^2\\
 &=O(|v|^{2k}).
\end{align*}
Since the increment has mean zero, independence gives
\begin{align}
 \sup_{\lambda\in\Lambda}
 \E_{P_{n,\lambda_0}}
 \left(\frac{dP_{n,\lambda}}{dP_{n,\lambda_0}}\right)^2
 &=\sup_{\lambda\in\Lambda}\left[1+
 \E_{v_n,\lambda_0}
 \left(\frac{dP_{v_n,\lambda}}{dP_{v_n,\lambda_0}}-1\right)^2
 \right]^n\nonumber\\
 &\le C_x<\infty. \label{supp:eq:generic-lr-ui}
\end{align}

Fix $\lambda_0,\lambda_1,\ldots,\lambda_m\in\Lambda$ and write
\[
 \mathbf R_n=\left(
  \frac{dP_{n,\lambda_r}}{dP_{n,\lambda_0}}
 \right)_{r=0}^m,
 \qquad
 \mathbf R=\left(
  \frac{dQ_{\lambda_r}}{dQ_{\lambda_0}}(W_x)
 \right)_{r=0}^m.
\]
For every bounded continuous $\varphi:\R_+^{m+1}\to\R$ and every $j$,
\begin{align*}
 \E_{P_{n,\lambda_j}}\varphi(\mathbf R_n)
 &=\E_{P_{n,\lambda_0}}
   \left[R_{n,j}\varphi(\mathbf R_n)\right]\\
 &\longrightarrow
   \E_{Q_{\lambda_0}}
   \left[R_j\varphi(\mathbf R)\right]
 =\E_{Q_{\lambda_j}}\varphi(\mathbf R). 
\end{align*}
For fixed $A$, \eqref{supp:eq:lr-process} applies to the bounded
continuous function $(r_j\wedge A)\varphi(\mathbf r)$, while
\[
 \left|\E_{P_{n,\lambda_0}}
 \left[\{R_{n,j}-(R_{n,j}\wedge A)\}\varphi(\mathbf R_n)\right]\right|
 \le\norm{\varphi}_\infty
 \E_{P_{n,\lambda_0}}[R_{n,j}\one\{R_{n,j}>A\}]
 \le\frac{\norm{\varphi}_\infty C_x}{A}.
\]
The limiting Gaussian likelihood ratio has a finite second moment, so first
letting $n\to\infty$ and then $A\to\infty$ proves the convergence.  To bound the Le Cam distance, construct Markov kernels as follows.
Let $\mu_n$ and $\mu$ be the laws of $\mathbf R_n$ under
$P_{n,\lambda_0}$ and of $\mathbf R$ under $Q_{\lambda_0}$, respectively.
By \eqref{supp:eq:lr-process}, $\mu_n\Rightarrow\mu$.  Since
$\R_+^{m+1}$ is Polish, the Skorohod representation theorem gives, on one
probability space $(\Omega^\ast,\mathcal F^\ast,\Pp^\ast)$, random vectors
$\mathbf R_n^\ast$ and $\mathbf R^\ast$ with respective laws $\mu_n$ and
$\mu$ such that $\mathbf R_n^\ast\to\mathbf R^\ast$ almost surely.  Let
\[
 \gamma_n=\mathcal L_{\Pp^\ast}(\mathbf R_n^\ast,\mathbf R^\ast);
\]
then $\gamma_n$ is a coupling of $\mu_n$ and $\mu$, with coordinate variables
denoted by $(\mathbf r,\mathbf s)$.  After discarding finitely many initial
indices, \eqref{supp:eq:generic-lr-ui} and the finite Gaussian second moments give
\[
 \begin{aligned}
 \sup_{n\ge n_0}\E^\ast
       \norm{\mathbf R_n^\ast-\mathbf R^\ast}_1^2
 &=\sup_{n\ge n_0}\int\norm{\mathbf r-\mathbf s}_1^2\,d\gamma_n\\
 &\le2(m+1)\sum_{j=0}^m
 \left\{\sup_{n\ge n_0}\E_{P_{n,\lambda_0}}R_{n,j}^2
       +\E_{Q_{\lambda_0}}R_j^2\right\}<\infty.
 \end{aligned}
\]
Thus $\norm{\mathbf R_n^\ast-\mathbf R^\ast}_1$ is uniformly integrable and
converges almost surely to zero.  Vitali's theorem therefore gives
\begin{equation}\label{supp:eq:lr-coupling-l1}
 \int\norm{\mathbf r-\mathbf s}_1\,
       \gamma_n(d\mathbf r,d\mathbf s)
 =\E^\ast\norm{\mathbf R_n^\ast-\mathbf R^\ast}_1
 \longrightarrow0.
\end{equation}
Disintegrate $\gamma_n$ in both directions to obtain Markov kernels
$K_n(\mathbf r,d\mathbf s)$ and
$\widetilde K_n(\mathbf s,d\mathbf r)$.  In the canonical likelihood-ratio
experiments, the measure at $\lambda_j$ is
$r_j\mu_n(d\mathbf r)$ or $s_j\mu(d\mathbf s)$.  Therefore
\begin{align*}
 \norm{K_n\{r_j\mu_n(d\mathbf r)\}-s_j\mu(d\mathbf s)}_{\mathrm{TV}}
 &\le\int|r_j-s_j|\,d\gamma_n,\\
 \norm{\widetilde K_n\{s_j\mu(d\mathbf s)\}-r_j\mu_n(d\mathbf r)}_{\mathrm{TV}}
 &\le\int|r_j-s_j|\,d\gamma_n,
\end{align*}
and both bounds tend to zero by \eqref{supp:eq:lr-coupling-l1}.

The original finite experiment is exactly equivalent to its canonical one:
one kernel records the likelihood-ratio vector, and the reverse kernel is the
regular conditional distribution of the observation given that vector under
the reference measure.  Such a conditional distribution exists because the
observation space is standard Borel; it is common to every parameter because
each density is a coordinate of the sufficient vector.  Explicitly, if
$H_n(\mathbf r,A)$ is the conditional probability of the original observation
falling in $A$ given $\mathbf R_n=\mathbf r$ under $P_{n,\lambda_0}$, then
\[
 \int H_n(\mathbf r,A)r_j\,\mu_n(d\mathbf r)
 =\E_{P_{n,\lambda_0}}[R_{n,j}\one_A]
 =P_{n,\lambda_j}(A).
\]
The same argument applies to the Gaussian finite experiment.  Composing these exact
equivalences with $K_n$ and $\widetilde K_n$ proves that the Le Cam distance
between the two finite restrictions tends to zero.  A finite set not
containing $\lambda_0$ is handled by first adjoining $\lambda_0$ and then
restricting.

Finally, the target conclusion follows immediately from
\eqref{supp:eq:target-jet} and $n^{1/(2k)}v_n=x$.
\end{proof}
\section{Proofs of Theorems~2.4--2.5 and Propositions~3.2--3.3}
\label{supp:sec:supp-general-results}

We use the statistical model $\{P_\xi:\xi\in\Xi\}$, target
$\tau:\Xi\to\R^d$, and mean map $\Psi:\Xi\to\cH$ of Section~2 of the main
paper, where $d\ge1$, $\mathcal O$ is standard Borel, and $\cH$ is separable.
The orders and bounds used in the proofs are recalled below.

\begin{definition}[Contact order]
\label{supp:def:orders}
\label{supp:def:contact-order}
Let $\mathbb N=\{1,2,\ldots\}$.  Define
\begin{align*}
 k_{\mathrm C}
 :=\sup\Bigl\{j\in\mathbb N:{}&\ \exists\ t_0,b,C>0\ \text{such that for
 every }0<t<t_0\ \exists\ \xi_{0,t},\xi_{1,t}\in\Xi,\\
 &\norm{\tau(\xi_{1,t})-\tau(\xi_{0,t})}_2\ge bt,
 \quad \KL(P_{\xi_{0,t}}\Vert P_{\xi_{1,t}})\le Ct^{2j}\Bigr\}.
\end{align*}
Set $\sup\varnothing=0$.
\end{definition}

\begin{definition}[Inverse order]
\label{supp:def:inverse-order}
For the fixed map $\Psi$, define
\begin{align*}
 k_{\mathrm I}
 :=\inf\Bigl\{j\in\mathbb N:{}&\ \exists\ A<\infty:\ 
 \ \norm{\tau(\xi)-\tau(\xi')}_2^j
 \le A\norm{\Psi(\xi)-\Psi(\xi')}_{\cH}
 \ \forall\xi,\xi'\in\Xi\Bigr\}.
\end{align*}
Set $\inf\varnothing=\infty$.
\end{definition}

\paragraph{Profiled submodel}
For $v_0>0$, let
$\xi:[-v_0,v_0]\times\Lambda\to\Xi$ index a profiled submodel and
abbreviate
\[
 P_{v,\lambda}=P_{\xi(v,\lambda)},\qquad
 \tau(v,\lambda)=\tau\{\xi(v,\lambda)\},
 \qquad |v|\le v_0,\quad\lambda\in\Lambda,
\]
where $\Lambda$ is a nondegenerate compact interval and $P_{0,\lambda}$ is
constant in $\lambda$.  Expectations and function-space norms have the
meaning fixed in Section~S1.2.

\begin{profilecondition}[Profiled jet at order $j$]
\label{supp:cond:profile-jet}
For $j\in\mathbb N$, there exist $\lambda_0\in\Lambda$,
$I\in(0,\infty)$, $\beta\in\R^d\setminus\{0\}$, and jointly measurable
maps $(v,o)\mapsto S_{v,\lambda_0}(o)\in\R$ and
$(v,o,\lambda)\mapsto r(o\mid v,\lambda)\in\R$.
For each $(v,o)$, the function $\lambda\mapsto r(o\mid v,\lambda)$ is
continuously differentiable. The following conditions hold, with all
limits as $v\to0$.
\begin{description}
\item[\textup{(P1)}] \textit{Likelihood expansion.}
For all sufficiently small $|v|$, $P_{v,\lambda}\ll P_{v,\lambda_0}$ for
every $\lambda\in\Lambda$, and, as an equality in $C(\Lambda)$ under
$P_{v,\lambda_0}$,
\begin{equation*}
 \frac{dP_{v,\lambda}}{dP_{v,\lambda_0}}(O)-1
 =(\lambda-\lambda_0)v^jS_{v,\lambda_0}(O)+r(O\mid v,\lambda).
\end{equation*}
\item[\textup{(P2)}] \textit{Score moments.}
\begin{equation*}
 \E_{v,\lambda_0}S_{v,\lambda_0}=0,\qquad
 \E_{v,\lambda_0}S_{v,\lambda_0}^2\longrightarrow I,\qquad
 \sup_{0<|v|\le v_0}\E_{v,\lambda_0}|S_{v,\lambda_0}|^4<\infty.
\end{equation*}
\Needspace{8\baselineskip}
\item[\textup{(P3)}] \textit{Remainder bound.}
\begin{equation*}
 \E_{v,\lambda_0}\Bigl[
 \sup_{\lambda\in\Lambda}|r(O\mid v,\lambda)|^2
 +\sup_{\lambda\in\Lambda}\left|
 \frac{\partial}{\partial\lambda}r(O\mid v,\lambda)\right|^2
 \Bigr]=o(|v|^{2j}).
\end{equation*}
\item[\textup{(P4)}] \textit{Target expansion.}
\begin{equation*}
 \sup_{\lambda\in\Lambda}
 \frac{\norm{\tau(v,\lambda)-v\beta\lambda}_2}{|v|}
 \longrightarrow0.
\end{equation*}
\end{description}
\end{profilecondition}

\begin{definition}[Profiled-jet order]
\label{supp:def:profiled-jet-order}
For the chosen profiled subexperiment, let
\[
 \mathcal J
 =\{j\in\mathbb N:\text{\Cref{supp:cond:profile-jet} holds at order }j\}.
\]
Define
\begin{equation*}
 k_{\mathrm J}:=\min\mathcal J.
\end{equation*}
The order $k_{\mathrm J}$ is undefined if $\mathcal J=\varnothing$.
\end{definition}

\subsection{Proof of Theorem~2.4: contact and inverse bounds}
\label{supp:sec:supp-contact-inverse}

The following theorem turns close observed distributions into a risk lower
bound and a stable inverse of the feature mean into an upper bound.

Fix $k\in\mathbb N$.  We first state the conditions used in this subsection.
\begin{description}
\item[\textup{(A1)}] The contact order satisfies $k_{\mathrm C}\ge k$.

\item[\textup{(A2)}] The parameter space $\Xi$ is a compact metric space,
the target $\tau$ is Borel measurable, and the inverse order satisfies
$k_{\mathrm I}\le k$.

\item[\textup{(A3)}] There is a measurable feature $\Phi:\mathcal O\to\cH$
such that
\begin{equation*}
 \Psi(\xi)=\E_\xi\Phi(O),
 \qquad
 \sup_{\xi\in\Xi}\E_\xi
 \norm{\Phi(O)-\Psi(\xi)}_{\cH}^2\le V<\infty.
\end{equation*}

\item[\textup{(A4)}] For each $n\ge1$, the minimum-distance criterion
admits a measurable selector
\begin{equation*}
 \widehat\xi_n\in\argmin_{\zeta\in\Xi}
 \norm{\widehat\Psi_n-\Psi(\zeta)}_{\cH},
 \qquad
 \widehat\Psi_n=\frac1n\sum_{i=1}^n\Phi(O_i).
\end{equation*}
\end{description}

\begin{theorem}[Theorem~2.4 of the main paper]
\label{supp:thm:generic-rate}
Under \textup{(A1)}, there is a constant $c>0$ such that, for all
sufficiently large $n$,
\begin{equation}\label{supp:eq:generic-lower}
 \inf_{T_n}\sup_{\xi\in\Xi}
 \E_\xi\norm{T_n-\tau(\xi)}_2^2\ge c n^{-1/k},
\end{equation}
where the infimum is over all measurable estimators
$T_n:\mathcal O^n\to\R^d$.

Under \textup{(A2)--(A4)}, there is a constant $C<\infty$, independent of
$n$, such that
\begin{equation}\label{supp:eq:generic-upper}
 \sup_{\xi\in\Xi}\E_\xi
 \norm{\tau(\widehat\xi_n)-\tau(\xi)}_2^2
 \le C n^{-1/k},\qquad n\ge1.
\end{equation}
Consequently, if \textup{(A1)--(A4)} hold, then
\[
 \inf_{T_n}\sup_{\xi\in\Xi}
 \E_\xi\norm{T_n-\tau(\xi)}_2^2\asymp n^{-1/k}.
\]
\end{theorem}

\Cref{supp:thm:generic-rate} shows that matching contact and inverse bounds
determine the squared-risk rate $n^{-1/k}$, with separate conditions for
the lower bound and the estimator achieving the upper bound.

\label{supp:app:proof-contact-inverse}
\begin{proof}[Proof of Theorem~2.4]
The admissible integer contact orders are downward closed.  Thus
\textup{(A1)} supplies $b,C_0,t_0>0$ and, for every $0<t<t_0$, pairs
$\xi_{0,t},\xi_{1,t}\in\Xi$ satisfying
\[
 \norm{\tau(\xi_{1,t})-\tau(\xi_{0,t})}_2\ge bt,
 \qquad \KL(P_{\xi_{0,t}}\Vert P_{\xi_{1,t}})\le C_0t^{2k}.
\]
Choose $\alpha>0$ with $C_0\alpha^{2k}\le1/2$ and set
$t=\alpha n^{-1/(2k)}$.  For all sufficiently large $n$, $t<t_0$ and
tensorisation and Pinsker's inequality give
\[
 \KL(P_{\xi_{0,t}}^{\otimes n}\Vert P_{\xi_{1,t}}^{\otimes n})
 \le C_0\alpha^{2k}\le\frac12,
 \qquad
 \TV(P_{\xi_{0,t}}^{\otimes n},P_{\xi_{1,t}}^{\otimes n})\le\frac12.
\]
The two-point Le Cam inequality therefore yields
\[
 \inf_{T_n}\sup_{\xi\in\Xi}
 \E_\xi\norm{T_n-\tau(\xi)}_2^2
 \ge\frac{b^2t^2}{16}
 =\frac{b^2\alpha^2}{16}n^{-1/k},
\]
proving \eqref{supp:eq:generic-lower}.

For the upper bound, put $j=k_{\mathrm I}\le k$.  Since the defining
infimum is over integers, there is $A_j<\infty$ such that
\[
 \norm{\tau(\xi)-\tau(\xi')}_2^j
 \le A_j\norm{\Psi(\xi)-\Psi(\xi')}_{\cH},
 \qquad \xi,\xi'\in\Xi.
\]
By \textup{(A3)--(A4)}, minimum distance and independence give
\[
 \norm{\Psi(\widehat\xi_n)-\Psi(\xi)}_{\cH}
 \le2\norm{\widehat\Psi_n-\Psi(\xi)}_{\cH},
 \qquad
 \E_\xi\norm{\widehat\Psi_n-\Psi(\xi)}_{\cH}^2\le\frac Vn.
\]
Jensen's inequality for $x\mapsto x^{1/j}$ now yields
\[
 \sup_{\xi\in\Xi}\E_\xi
 \norm{\tau(\widehat\xi_n)-\tau(\xi)}_2^2
 \le\left(\frac{4A_j^2V}{n}\right)^{1/j}
 \le Cn^{-1/k},\qquad n\ge1,
\]
for a finite constant $C$, proving \eqref{supp:eq:generic-upper}.
\end{proof}

\subsection{Proof of Theorem~2.5: comparison of the identification orders}
\label{supp:sec:supp-order-hierarchy}

The contact and inverse bounds must be compatible with the order found in
the profiled likelihood.  The following theorem establishes this
relationship by comparing the three orders.

\begin{theorem}[Theorem~2.5 of the main paper]
\label{supp:thm:order-hierarchy}
\label{supp:cor:contact-inverse}
\label{supp:thm:jet-contact}
\label{supp:cor:order-hierarchy}
Under \textup{(A3)},
\begin{equation*}
 k_{\mathrm C}\le k_{\mathrm I}.
\end{equation*}
If, in addition, the profiled paths lie in $\Xi$, $\Lambda$ is
nondegenerate, and $k_{\mathrm J}$ is defined, then
\begin{equation*}
 k_{\mathrm J}\le k_{\mathrm C}\le k_{\mathrm I}.
\end{equation*}
\end{theorem}

\Cref{supp:thm:order-hierarchy} shows that recovery through a bounded-variance
feature cannot overcome the difficulty exhibited by close distributions;
the profiled likelihood supplies such contrasts directly.

\label{supp:app:proof-order-hierarchy}
\begin{proof}[Proof of Theorem~2.5]
For $P=P_\xi$ and $Q=P_{\xi'}$, let $p,q$ be their densities with respect to
$\mu=P+Q$, and put
\[
 \begin{aligned}
 m_P&=\E_P\Phi,\qquad m_Q=\E_Q\Phi,\qquad
 a=\frac{m_P+m_Q}{2},\\
 h^2(P,Q)&=\int(\sqrt p-\sqrt q)^2\,d\mu,\qquad
 \Delta=\norm{m_P-m_Q}_{\cH}.
 \end{aligned}
\]
By Cauchy--Schwarz and \textup{(A3)},
\[
 \begin{aligned}
 \Delta^2
 &\le h^2(P,Q)\int\norm{\Phi-a}_{\cH}^2
                     (\sqrt p+\sqrt q)^2\,d\mu\\
 &\le 2h^2(P,Q)\bigl\{\E_P\norm{\Phi-a}_{\cH}^2
                        +\E_Q\norm{\Phi-a}_{\cH}^2\bigr\}\\
 &\le (4V+\Delta^2)h^2(P,Q).
 \end{aligned}
\]
Thus, when $h^2(P,Q)\le1/2$,
\[
 \Delta^2\le8Vh^2(P,Q)\le8V\KL(P\Vert Q).
\]
If $k_{\mathrm I}<\infty$, its defining inverse inequality and any admissible
contact order $j$ give, for the contact pairs,
\[
 b^{k_{\mathrm I}}t^{k_{\mathrm I}}
 \le\norm{\tau(\xi_{1,t})-\tau(\xi_{0,t})}_2^{k_{\mathrm I}}
 \le A\norm{\Psi(\xi_{1,t})-\Psi(\xi_{0,t})}_{\cH}
 =O(t^j).
\]
Hence $j\le k_{\mathrm I}$, and taking the supremum gives
$k_{\mathrm C}\le k_{\mathrm I}$.  The case $k_{\mathrm I}=\infty$ is immediate.

For the jet comparison, take its reference point $\lambda_0$ and choose
$\lambda_1\ne\lambda_0$ in $\Lambda$.  As $t\downarrow0$,
\[
 \begin{aligned}
 \KL(P_{t,\lambda_1}\Vert P_{t,\lambda_0})
 &\le\chi^2(P_{t,\lambda_1},P_{t,\lambda_0})
 =O(t^{2k_{\mathrm J}}),\\
 \norm{\tau(t,\lambda_1)-\tau(t,\lambda_0)}_2
 &=t\norm{\beta}_2|\lambda_1-\lambda_0|+o(t).
 \end{aligned}
\]
These are admissible contact pairs of order $k_{\mathrm J}$, proving
$k_{\mathrm J}\le k_{\mathrm C}$.
\end{proof}

\subsection{Proof of Proposition~3.2: local minimax risk}
\label{supp:sec:supp-local-minimax}

At the critical scale $n|v|^{2k}\asymp1$, the Gaussian likelihood limit
retains finite information about $\lambda$.

For fixed $x\ne0$, put $v_n=xn^{-1/(2k)}$ and define the exact local
asymptotic minimax risk, whenever the limit exists, by
\begin{equation}\label{supp:eq:local-risk-def}
 \mathfrak R_x^{\mathrm{loc}}
 :=\lim_{n\to\infty}\inf_{T_n}\sup_{\lambda\in\Lambda}
 n^{1/k}\E_{P_{v_n,\lambda}^{\otimes n}}
 \norm{T_n-\tau(v_n,\lambda)}_2^2,
\end{equation}
where the infimum is over all measurable estimators of the target.

For $\sigma^2>0$ and compact $\Lambda$, define the bounded-normal minimax risk
over measurable $\delta:\R\to\R$; see
\cite{supp:bickel1981,supp:casellaStrawderman1981,supp:donohoLiu1991iii}:
\begin{equation}\label{supp:eq:M-def}
 M(\sigma^2,\Lambda)
 =\inf_\delta\sup_{\lambda\in\Lambda}
   \E_{W\sim N(\lambda,\sigma^2)}\{\delta(W)-\lambda\}^2.
\end{equation}
Let
\begin{equation}\label{supp:eq:RG-Kk}
 R_G(z)=\inf_\delta\sup_{\mu\in[0,z]}
 \E_{W\sim N(\mu,1)}\{\delta(W)-\mu\}^2,
 \qquad
 K_k=\sup_{z>0}z^{-2+2/k}R_G(z).
\end{equation}
For the critical experiment, the standardized interval width is
$z=L/\sigma_x=L\sqrt I\,|x|^k$, where $L=\diam\Lambda$ and
$\sigma_x^2=(Ix^{2k})^{-1}$.

The following proposition computes the target risk in this limit model
and identifies the largest lower-bound constant over the critical paths.

\begin{proposition}[Proposition~3.2 of the main paper]\label{supp:prop:normal-envelope}
Let $L=\diam\Lambda$.  Under \Cref{supp:thm:profile-lan}, the limit in
\eqref{supp:eq:local-risk-def} exists and
\begin{equation}\label{supp:eq:local-envelope}
 \mathfrak R_x^{\mathrm{loc}}
 =\norm{\beta}_2^2x^2M\paren{\frac1{Ix^{2k}},\Lambda}.
\end{equation}
Furthermore, $0<K_k<\infty$ and
\begin{equation}\label{supp:eq:envelope-sup}
 \sup_{x\ne0}\norm{\beta}_2^2x^2M\paren{\frac1{Ix^{2k}},\Lambda}
 =\norm{\beta}_2^2 I^{-1/k}L^{2-2/k}K_k.
\end{equation}
Every global parameter space containing these local paths therefore has normalized minimax liminf at least the right side of \eqref{supp:eq:envelope-sup}.
\end{proposition}

\Cref{supp:prop:normal-envelope} determines the risk on each critical path and
the resulting global lower bound; uniform attainment of this constant
requires a separate construction.

For short intervals, the bounded-normal minimax rule is explicit.

\begin{remark}[An explicit minimax rule]
\label{supp:rem:explicit-minimax-rule}
For $0<z\le2/\sqrt3$, the equal-endpoint prior
$\tfrac12\delta_0+\tfrac12\delta_z$ is least favourable for the
bounded-normal problem defining $R_G(z)$.  Its Bayes rule
\begin{equation}\label{supp:eq:normal-endpoint-rule}
 \delta_z^{\rm ep}(w)
 =\frac z2+\frac z2\tanh\left\{\frac z2\left(w-\frac z2\right)\right\}
\end{equation}
is minimax, and its risk is maximized at $0$ and $z$:
\begin{equation}\label{supp:eq:normal-endpoint-risk}
 R_G(z)
 =\E_{N(0,1)}\{\delta_z^{\rm ep}(W)\}^2
 =\E_{N(z,1)}\{\delta_z^{\rm ep}(W)-z\}^2
 =\sup_{\mu\in[0,z]}
   \E_{N(\mu,1)}\{\delta_z^{\rm ep}(W)-\mu\}^2.
\end{equation}
If $z<2/\sqrt3$, only the two endpoints attain the supremum.
\end{remark}

\label{supp:app:proof-normal-envelope}
\begin{proof}[Proof of Proposition~3.2]
Put
\[
 P_{n,\lambda}=P_{v_n,\lambda}^{\otimes n},\qquad
 Q_\lambda=N\paren{\lambda,\frac1{Ix^{2k}}},
 \qquad
 g_n(\lambda)=n^{1/(2k)}\tau(v_n,\lambda),\qquad
 g(\lambda)=x\beta\lambda,
\]
and let $R_n(x)$ denote the finite-$n$ infimum--supremum inside
\eqref{supp:eq:local-risk-def}.  By \eqref{supp:eq:target-jet},
\[
 \epsilon_n:=\sup_{\lambda\in\Lambda}
 \norm{g_n(\lambda)-g(\lambda)}_2\longrightarrow0.
\]
For all sufficiently large $n$, the sets $g_n(\Lambda)$ and $g(\Lambda)$ lie
in one compact convex set $C$; write $D=\diam C$.  Projecting the scaled
decision $n^{1/(2k)}T_n$ onto $C$ cannot increase squared loss.  Hence all
finite and limiting decisions may be restricted to $C$, and, uniformly in
$a\in C$ and $\lambda\in\Lambda$,
\begin{equation}\label{supp:eq:local-loss-difference}
 \left|
  \norm{a-g_n(\lambda)}_2^2-\norm{a-g(\lambda)}_2^2
 \right|
 \le2D\epsilon_n. 
\end{equation}
The likelihood-ratio calculation in the proof of \Cref{supp:thm:profile-lan} also
gives the uniform second-moment bound
\begin{equation}\label{supp:eq:profile-lr-ui-summary}
 \sup_{n\ge n_0}\sup_{\lambda\in\Lambda}
 \E_{P_{n,\lambda_0}}
 \left(\frac{dP_{n,\lambda}}{dP_{n,\lambda_0}}\right)^2
 \le C_x<\infty.
\end{equation}

For the lower bound, let
 $\pi=\sum_{j=0}^m p_j\delta_{\lambda_j}$ be any finite prior, where
$\delta_\lambda$ is the point mass at $\lambda$, with
$\lambda_0$ in its support and every $p_j>0$.  Its Bayes risk in the
finite-sample experiment is
\begin{align}
 B_n(\pi)
 =\E_{P_{n,\lambda_0}}\Bigg[
 &\sum_{j=0}^m p_j\norm{g_n(\lambda_j)}_2^2
       \frac{dP_{n,\lambda_j}}{dP_{n,\lambda_0}}
 \nonumber\\
 &-\frac{\displaystyle
   \norm{\sum_{j=0}^m p_jg_n(\lambda_j)
       \frac{dP_{n,\lambda_j}}{dP_{n,\lambda_0}}}_2^2}
 {\displaystyle
   \sum_{j=0}^m p_j
       \frac{dP_{n,\lambda_j}}{dP_{n,\lambda_0}}}
 \Bigg]. \label{supp:eq:finite-bayes-risk}
\end{align}
The posterior mean is the Bayes rule, and the denominator is at least
$p_0$, because
$dP_{n,\lambda_0}/dP_{n,\lambda_0}=1$.

Here and below, $\delta$ ranges over measurable maps from $\mathbb R$ to
$C\subset\mathbb R^d$.  By the likelihood-ratio process convergence in
\Cref{supp:thm:profile-lan}, under $P_{n,\lambda_0}$,
\[
 \left(
  \frac{dP_{n,\lambda_j}}{dP_{n,\lambda_0}}
 \right)_{j=0}^m
 \Rightarrow
 \left(
  \frac{dQ_{\lambda_j}}{dQ_{\lambda_0}}(W)
 \right)_{j=0}^m,
 \qquad W\sim Q_{\lambda_0}.
\]
The integrand in \eqref{supp:eq:finite-bayes-risk} is nonnegative and bounded by
\[
 D^2\sum_{j=0}^m p_j
 \frac{dP_{n,\lambda_j}}{dP_{n,\lambda_0}}.
\]
Consequently \eqref{supp:eq:profile-lr-ui-summary}, uniform convergence of $g_n$, and
uniform integrability yield
\[
 B_n(\pi)\longrightarrow
 \inf_\delta\sum_{j=0}^m p_j
 \E_{Q_{\lambda_j}}\norm{\delta(W)-g(\lambda_j)}_2^2
 =:B(\pi).
\]
Since $R_n(x)\ge B_n(\pi)$, taking the supremum over finite priors gives
\begin{equation}\label{supp:eq:local-risk-liminf}
 \liminf_{n\to\infty}R_n(x)
 \ge
 \inf_\delta\sup_{\lambda\in\Lambda}
 \E_{Q_\lambda}\norm{\delta(W)-g(\lambda)}_2^2. 
\end{equation}
Here the equality between the supremum of finite-prior Bayes risks and the
minimax risk is justified as follows.  Write
\[
 \sigma_x^2=(Ix^{2k})^{-1},\qquad
 q_\lambda(w)=\frac1{\sqrt{2\pi}\sigma_x}
 \exp\left\{-\frac{(w-\lambda)^2}{2\sigma_x^2}\right\},
\]
and let $\Lambda=[\underline\lambda,\overline\lambda]$.  The Gaussian
envelope
\[
 h_x(w)=\max_{t\in\Lambda}q_t(w)
\]
is obtained by projecting $w$ onto $\Lambda$, whence
\[
 H_x:=\int_\R h_x(w)\,dw
 =1+\frac{\overline\lambda-\underline\lambda}
          {\sqrt{2\pi}\sigma_x}<\infty.
\]
Set $\nu_x(dw)=H_x^{-1}h_x(w)\,dw$.  Then
\begin{equation*}
 r_\lambda(w):=\frac{dQ_\lambda}{d\nu_x}(w)
 =H_x\frac{q_\lambda(w)}{h_x(w)}\le H_x.
\end{equation*}
Let
\[
 \mathcal D_x=\{\delta\in L^2(\nu_x;\R^d):
                  \delta(w)\in C\ \nu_x\text{-a.e.}\}
\]
with the strong $L^2(\nu_x)$ topology, and let $\mathcal P(\Lambda)$ be the
Borel probability measures with their weak topology, equivalently the
weak-star topology $\sigma(C(\Lambda)^*,C(\Lambda))$.  By the Riesz--Markov representation and Banach--Alaoglu
theorems, $\mathcal P(\Lambda)$ is a compact convex subset of
$C(\Lambda)^*$: it is the weak-star closed subset of the dual unit ball
consisting of positive functionals $\pi$ with $\pi(1)=1$.

For
\[
 \begin{aligned}
 R(\lambda,\delta)
 &=\int_\R\norm{\delta(w)-g(\lambda)}_2^2Q_\lambda(dw),\\
 \mathcal F(\pi,\delta)
 &=\int_\Lambda\int_\R
   \norm{\delta(w)-g(\lambda)}_2^2r_\lambda(w)
   \,\nu_x(dw)\pi(d\lambda).
 \end{aligned}
\]
$\mathcal F(\cdot,\delta)$ is continuous and affine.  Indeed, uniformly over
$C$-valued $\delta$,
\begin{equation}\label{supp:eq:normal-risk-equicontinuity}
 \left|R(\lambda,\delta)-R(t,\delta)\right|
 \le D^2\norm{q_\lambda-q_t}_{L^1(dw)}
      +2D\norm{g(\lambda)-g(t)}_2,
\end{equation}
which tends to zero as $t\to\lambda$.  For fixed $\pi$,
$\mathcal F(\pi,\cdot)$ is convex and continuous, since
\begin{equation*}
 |\mathcal F(\pi,\delta)-\mathcal F(\pi,\delta')|
 \le2D\sqrt{H_x}\norm{\delta-\delta'}_{L^2(\nu_x)}.
\end{equation*}
Sion's minimax theorem \cite{supp:sion1958} therefore gives
\begin{equation}\label{supp:eq:normal-minimax-bayes}
 \inf_{\delta\in\mathcal D_x}\sup_{\lambda\in\Lambda}R(\lambda,\delta)
 =\max_{\pi\in\mathcal P(\Lambda)}
   \inf_{\delta\in\mathcal D_x}\mathcal F(\pi,\delta).
\end{equation}
Here $\sup_\pi\mathcal F(\pi,\delta)=\sup_\lambda R(\lambda,\delta)$ because
Dirac measures are allowed.  The maximum in \eqref{supp:eq:normal-minimax-bayes}
is attained because the Bayes risk, being the infimum over $\delta$ of
continuous functions of $\pi$, is upper semicontinuous on the compact prior
space.

There is also a saddle point.  Indeed, $\mathcal D_x$ is closed, convex, and
bounded in the reflexive Hilbert space $L^2(\nu_x;\R^d)$, hence weakly compact;
the convex norm-continuous risk is weakly lower semicontinuous.  Thus the
minimax infimum is attained.

If $\pi^*$ is least favourable, its squared-loss
Bayes rule is
\begin{equation}\label{supp:eq:normal-posterior-mean}
 \delta_{\pi^*}(w)
 =\frac{\int_\Lambda g(\lambda)q_\lambda(w)\,\pi^*(d\lambda)}
        {\int_\Lambda q_\lambda(w)\,\pi^*(d\lambda)}.
\end{equation}
It is continuous, takes values in $C$, and is the unique Bayes rule up to the
mixture marginal's null sets.  That marginal has a strictly positive density,
so its null sets agree with those of every $Q_\lambda$.  The saddle-point
equalities therefore show that \eqref{supp:eq:normal-posterior-mean} is itself a
bounded continuous minimax rule.

For finite-prior approximation, the right side of
\eqref{supp:eq:normal-risk-equicontinuity} supplies a common modulus
$\omega(\rho)\downarrow0$.  Map each point of $\Lambda$ to a nearest point
of a finite $\rho$-net and let $\pi_\rho$ be the pushforward of $\pi$.  Then
\[
 \sup_{\delta\in\mathcal D_x}
 |\mathcal F(\pi_\rho,\delta)-\mathcal F(\pi,\delta)|
 \le\omega(\rho),
\]
so the corresponding Bayes risks differ by at most $\omega(\rho)$.  Thus the
maximum over all priors equals the supremum over finite priors.  Finally,
replacing a finite $\pi$ by
$(1-a)\pi+a\delta_{\lambda_0}$ changes its Bayes risk by at most $aD^2$.
Hence requiring positive mass at $\lambda_0$ causes no loss.

For the matching upper bound, choose, for arbitrary $\varepsilon>0$, a bounded
continuous rule $\delta_\varepsilon:\R\to C$ satisfying
\[
 \sup_{\lambda\in\Lambda}
 \E_{Q_\lambda}
 \norm{\delta_\varepsilon(W)-g(\lambda)}_2^2
 \le
 \inf_\delta\sup_{\lambda\in\Lambda}
 \E_{Q_\lambda}\norm{\delta(W)-g(\lambda)}_2^2+\varepsilon.
\]
Such a rule exists by \eqref{supp:eq:normal-posterior-mean} and the minimax
property established above. Define
\[
 \widehat W_n
 =\lambda_0+\frac{\Delta_{n,x,\lambda_0}}{x^{2k}I}.
\]

For
$f_\lambda(w)=\norm{\delta_\varepsilon(w)-g(\lambda)}_2^2$,
change of measure gives
\begin{align}
 \E_{P_{n,\lambda}}f_\lambda(\widehat W_n)
 &=
 \E_{P_{n,\lambda_0}}\left[
  \frac{dP_{n,\lambda}}{dP_{n,\lambda_0}}
  f_\lambda(\widehat W_n)
 \right]\nonumber\\
 &\longrightarrow
 \E_{Q_{\lambda_0}}\left[
  \frac{dQ_\lambda}{dQ_{\lambda_0}}(W)f_\lambda(W)
 \right]\nonumber\\
 &=\E_{Q_\lambda}f_\lambda(W).
 \label{supp:eq:local-risk-change-measure}
\end{align}
uniformly in $\lambda\in\Lambda$.  To see this, take any sequence
$\lambda_n\to\lambda$.  Joint $C(\Lambda)\times\R$ convergence in the proof
of \Cref{supp:thm:profile-lan}, continuity of $\delta_\varepsilon$ and $g$, and
\eqref{supp:eq:profile-lr-ui-summary} give convergence along this sequence.
Compactness of $\Lambda$ then gives uniform convergence.

Using the estimator
$T_{n,\varepsilon}=n^{-1/(2k)}
\delta_\varepsilon(\widehat W_n)$,
\eqref{supp:eq:local-loss-difference} and
\eqref{supp:eq:local-risk-change-measure} imply
\[
 \limsup_{n\to\infty}R_n(x)
 \le
 \inf_\delta\sup_{\lambda\in\Lambda}
 \E_{Q_\lambda}\norm{\delta(W)-g(\lambda)}_2^2+\varepsilon.
\]
Letting $\varepsilon\downarrow0$ and combining with
\eqref{supp:eq:local-risk-liminf} proves that the limit exists.  Finally,
$g(\lambda)=x\beta\lambda$ and projection onto its one-dimensional range give
\[
 \lim_{n\to\infty}R_n(x)
 =\norm{\beta}_2^2x^2
 \inf_\delta\sup_{\lambda\in\Lambda}
 \E_{Q_\lambda}\{\delta(W)-\lambda\}^2
 =\norm{\beta}_2^2x^2M\paren{\frac1{Ix^{2k}},\Lambda},
\]
which proves \eqref{supp:eq:local-envelope}.

Translation and scale invariance imply
\[
 M\paren{\frac1{Ix^{2k}},\Lambda}
 =\frac1{Ix^{2k}}R_G\paren{L\sqrt I\,|x|^k}.
\]

It remains to prove the endpoint statement.  Put $a=z/2$, $Y=W-a$, and
$\vartheta=\mu-a$.  Under the prior assigning equal mass to the endpoints of $[-a,a]$, the posterior
mean is $d_a(y)=a\tanh(ay)$.  For $Z\sim N(0,1)$, write
\[
 r_a(\vartheta)=\E\{d_a(\vartheta+Z)-\vartheta\}^2.
\]
Since
\[
 d_a'(y)=a^2\sech^2(ay),\qquad
 d_a''(y)=-2d_a(y)d_a'(y),
\]
differentiation under the expectation gives
\[
 r_a''(\vartheta)=2\E\left[
   \{d_a'(\vartheta+Z)-1\}^2
   +\{d_a(\vartheta+Z)-\vartheta\}d_a''(\vartheta+Z)
 \right].
\]
For $|\vartheta|\le a\le1/\sqrt3$,
$|d_a|\le a$, $|d_a-\vartheta|\le2a$, and $0\le d_a'\le a^2$; hence
\[
 r_a''(\vartheta)
 \ge2\{(1-a^2)^2-4a^4\}
 =2(1+a^2)(1-3a^2)\ge0.
\]
Thus $r_a$ is convex.  It is even by symmetry, so its maximum on $[-a,a]$
is attained at the two endpoints, and strictly only there when
$a<1/\sqrt3$.  Write
$R(\vartheta,d)=\E\{d(\vartheta+Z)-\vartheta\}^2$.  For any rule $d$,
\[
 \sup_{|\vartheta|\le a}R(\vartheta,d)
 \ge\frac{R(-a,d)+R(a,d)}2
 \ge\frac{r_a(-a)+r_a(a)}2=r_a(a),
\]
because $d_a$ minimizes the Bayes risk of this prior.  The
maximum risk of $d_a$ equals this lower bound.  Translating back proves
\eqref{supp:eq:normal-endpoint-rule}--\eqref{supp:eq:normal-endpoint-risk}.

Put $z=L\sqrt I\,|x|^k$.  Direct algebra transforms \eqref{supp:eq:local-envelope} into
\[
 \norm{\beta}_2^2I^{-1/k}L^{2-2/k}z^{-2+2/k}R_G(z),
\]
which proves \eqref{supp:eq:envelope-sup}.  A two-point bound at $z=1$ gives $K_k>0$.  The estimators $W$ and $z/2$ give
\[
 0<R_G(z)\le\min\{1,z^2/4\}.
\]
For $z\le1$, the defining integrand for $K_k$ is at most $z^{2/k}/4$; for $z\ge1$ it is at most $z^{-2+2/k}$.  Hence $K_k<\infty$.
\end{proof}

\subsection{Proof of Proposition~3.3: changes of coordinates}
\label{supp:sec:supp-coordinate-invariance}

The local risk was expressed in profile coordinates.  The following
proposition checks that regular rescalings of these coordinates do not
alter the identification order or the risk constant.

\begin{proposition}[Proposition~3.3 of the main paper]\label{supp:prop:coordinate-invariance}
Under \Cref{supp:ass:profile-jet}, let $v=g(u)$ be a $C^1$ diffeomorphism defined
for $|u|\le u_0$, with
$g(0)=0$ and $\alpha=g'(0)\ne0$.  Then
$\widetilde P_{u,\lambda}=P_{g(u),\lambda}$ and
$\widetilde\tau(u,\lambda)=\tau\{g(u),\lambda\}$ satisfy the same assumption
at order $k$, with
\begin{equation}\label{supp:eq:coordinate-scaling}
 \widetilde I=|\alpha|^{2k}I,\qquad
 \widetilde\beta=\alpha\beta.
\end{equation}
Thus, writing $L=\diam\Lambda$,
\begin{equation}\label{supp:eq:coordinate-envelope}
 \norm{\widetilde\beta}_2^2\widetilde I^{-1/k}
 L^{2-2/k}K_k
 =\norm{\beta}_2^2I^{-1/k}L^{2-2/k}K_k.
\end{equation}
Likewise, if $\lambda=d_0+b\mu$, $b\ne0$, maps $\widetilde\Lambda$ onto
$\Lambda$, then, after removing the common fixed-$v$ target translation
$v\beta d_0$,
\begin{equation}\label{supp:eq:nuisance-coordinate-scaling}
 I_\mu=b^2I,\qquad \beta_\mu=b\beta,
 \qquad \diam\widetilde\Lambda=L/|b|.
\end{equation}
Both regular coordinate changes leave the jet order and the factor in
\eqref{supp:eq:coordinate-envelope} unchanged.
\end{proposition}

\Cref{supp:prop:coordinate-invariance} ensures that the constant describes the
statistical difficulty independently of the units used for the profile
coordinates.

\begin{proof}[Proof of Proposition~3.3]
\leavevmode\par
\paragraph{Profile coordinate}
For $u\ne0$, set
\[
 \widetilde S_{u,\lambda_0}
 =\left\{\frac{g(u)}u\right\}^{k}S_{g(u),\lambda_0},
 \qquad \widetilde r(O\mid u,\lambda)=r(O\mid g(u),\lambda).
\]
Then \eqref{supp:eq:ratio-jet} is exactly
\[
 \frac{d\widetilde P_{u,\lambda}}
      {d\widetilde P_{u,\lambda_0}}(O)-1
 = (\lambda-\lambda_0)u^k\widetilde S_{u,\lambda_0}(O)
   +\widetilde r(O\mid u,\lambda).
\]
Composition preserves absolute continuity and joint measurability, and the
score rescaling preserves centering.  Since $g(u)/u\to\alpha\ne0$, after
reducing $u_0>0$ if necessary,
\begin{align*}
 &\sup_{0<|u|\le u_0}\E_{g(u),\lambda_0}
 |\widetilde S_{u,\lambda_0}|^4<\infty,\\
 &\E_{g(u),\lambda_0}\Bigl[
 \sup_{\lambda\in\Lambda}|\widetilde r(O\mid u,\lambda)|^2
 +\sup_{\lambda\in\Lambda}\left|
 \frac{\partial}{\partial\lambda}\widetilde r(O\mid u,\lambda)\right|^2
 \Bigr]=o(|u|^{2k}),\\
 &\E_{g(u),\lambda_0}\widetilde S_{u,\lambda_0}^2
 =\left\{\frac{g(u)}u\right\}^{2k}
   \E_{g(u),\lambda_0}S_{g(u),\lambda_0}^2
 \longrightarrow |\alpha|^{2k}I.
\end{align*}
Uniformly in $\lambda$,
\[
 \widetilde\tau(u,\lambda)-u\alpha\beta\lambda
 =\tau\{g(u),\lambda\}-g(u)\beta\lambda
  +\{g(u)-\alpha u\}\beta\lambda=o(|u|).
\]
This proves \eqref{supp:eq:coordinate-scaling}; the powers $|\alpha|^2$ and
$|\alpha|^{-2}$ cancel in \eqref{supp:eq:coordinate-envelope}.
Applying the same argument to the inverse diffeomorphism shows that the
minimal admissible jet order is unchanged.

\paragraph{Nuisance coordinate}
Under $\lambda=d_0+b\mu$, put
$\mu_0=(\lambda_0-d_0)/b$,
$S_{v,\mu_0}^{(\mu)}=bS_{v,\lambda_0}$, and
$r^{(\mu)}(O\mid v,\mu)=r(O\mid v,d_0+b\mu)$.
The likelihood coefficient is then
$(\mu-\mu_0)v^kS_{v,\mu_0}^{(\mu)}$, and the chain rule gives
\[
 \norm{r^{(\mu)}(O\mid v,\cdot)}_{C^1(\widetilde\Lambda)}
 \le\max(1,|b|)\norm{r(O\mid v,\cdot)}_{C^1(\Lambda)}.
\]
Thus $I_\mu=b^2I$, and, uniformly in $\mu$,
\[
 \tau(v,d_0+b\mu)-v\beta d_0=v(b\beta)\mu+o(|v|).
\]
The term $v\beta d_0$ is common to the fixed-$v$ experiment; translating every
action and target by it leaves squared risk unchanged.  Finally,
$|b|^2|b|^{-2/k}|b|^{-(2-2/k)}=1$, which proves
\eqref{supp:eq:nuisance-coordinate-scaling} and invariance.
\end{proof}

\section{Proofs of Theorem~4.1 and Corollary~4.2: global minimax risk and attainment}
\label{supp:sec:supp-exact-globalization}

\begin{equation}\label{supp:eq:general-exact-constant}
 \mathsf C_k=K_k\sup_{c\in\mathcal C}
 \beta(c)^2I(c)^{-1/k}\{\diam\Lambda(c)\}^{2-2/k}.
\end{equation}
With the notation of Section~4.1 of the main paper,
\[
 \mathcal R_n=\inf_{\widehat\tau_n}\sup_{\xi\in\Xi}
 \E_\xi\{\widehat\tau_n-\tau(\xi)\}^2.
\]

The local risk gives a lower bound over the full parameter space.  The
following theorem states when one estimator also attains this constant
uniformly, including parameters outside the critical paths.

\begin{theorem}[Theorem~4.1 of the main paper]
\label{supp:thm:global-lower-bound}
\label{supp:thm:exact-globalization}
\label{supp:thm:attaining-estimator}
Under the uniform setup in Section~4.1 of the main paper,
\[
 \liminf_{n\to\infty}n^{1/k}\mathcal R_n\ge\mathsf C_k.
\]
If, in addition, conditions \textup{(E)}, \textup{(C)}, and \textup{(O)}
of the main paper hold, then $n^{1/k}\mathcal R_n\to\mathsf C_k$.
The estimator in Section~4.2 of the main paper, with tuning sequences
chosen in the proof below, attains this limit.
\end{theorem}

\Cref{supp:thm:exact-globalization} separates the lower bound supplied by the
profiled likelihood from uniform attainment, which additionally requires
conditions \textup{(E)}, \textup{(C)}, and \textup{(O)}.

\subsection{Lower bound in Theorem~4.1}
\label{supp:sec:supp-global-lower}

We first transfer the critical-path risks to the global parameter space.

\begin{proof}[Proof of Theorem~4.1: lower bound]
Fix $c\in\mathcal C$ with $L(c)>0$ and $x\ne0$.  The coordinate neighborhood contains the
path $v_n=xn^{-1/(2k)}$, $\lambda\in\Lambda(c)$ for all large $n$.  The
profiled likelihood expansion, the uniform target expansion, and
\Cref{supp:prop:normal-envelope} give
\begin{equation*}
 \liminf_{n\to\infty}n^{1/k}\mathcal R_n
 \ge\beta(c)^2x^2
 M\left(\frac1{I(c)x^{2k}},\Lambda(c)\right).
\end{equation*}
For fixed $c$, put $z=L(c)\sqrt{I(c)}|x|^k$.  Translation and scale
invariance give
\[
 \beta(c)^2x^2M\left(\frac1{I(c)x^{2k}},\Lambda(c)\right)
 =\beta(c)^2I(c)^{-1/k}L(c)^{2-2/k}
 z^{-2+2/k}R_G(z).
\]
Taking first $\sup_x$ and then $\sup_c$ yields
$\liminf n^{1/k}\mathcal R_n\ge\mathsf C_k$.
\end{proof}

\subsection{Uniform attainment in Theorem~4.1}
\label{supp:sec:supp-uniform-attainment}

The lower bound alone does not provide an estimator using unknown
coordinates.  We now control coordinate substitution and recovery outside
the critical region to prove the matching upper bound.

Use conditions \textup{(E)}, \textup{(C)}, and \textup{(O)} and their
notation from Section~4.1 of the main paper.  In particular,
\[
 Z_{N_n}(c,v,\lambda)
 =\frac1{N_n}\sum_{i=1}^{N_n}\phi_{c,v}(O_i)-\Gamma(c,v,\lambda),
 \qquad
 \lambda_{\widehat c}=\Pi_{\Lambda(\widehat c)}(\lambda).
\]
Condition \textup{(C)} gives, for every fixed $h>0$,
\begin{equation}\label{supp:eq:calibration-orthogonality}
 \sup_{\substack{c\in\mathcal C,\ \lambda\in\Lambda(c)\\
          hN_n^{-1/(2k)}\le |v|\le v_0}}
 N_n^{1/(2k)}|v|^{1-k}
 \left\|
  [Z_{N_n}(\widehat c,\widehat v,\lambda_{\widehat c})
       -Z_{N_n}(c,v,\lambda)]\one_{\mathcal H_n(c,v)}
 \right\|_2\longrightarrow0.
\end{equation}

Condition \textup{(O)(ii)} gives
\begin{subequations}
\begin{align}
 |\tau(\xi)-\tau(\xi')|
 &\le A_O\|\Psi_O(\xi)-\Psi_O(\xi')\|_{\mathcal H_O},
 \qquad \xi,\xi'\in\Xi_O,
 \label{supp:eq:off-fibre-inverse}\\
 \sup_{\xi\in\Xi_O}\E_\xi
 \|\Phi_O(O)-\Psi_O(\xi)\|_{\mathcal H_O}^2
 &\le A_O.
 \label{supp:eq:off-fibre-variance}
\end{align}
\end{subequations}
\begin{proof}[Proof of Theorem~4.1: uniform attainment]
\leavevmode\par
\paragraph{Step 1: local and inverse estimators}
Put $q=1/(2k)$ and write
\[
 \widehat\Gamma_{N_n}(c,v)
 =\frac1{N_n}\sum_{i=1}^{N_n}\phi_{c,v}(O_i)
 =Z_{N_n}(c,v,\lambda)+\Gamma(c,v,\lambda).
\]
The moment bound in (E) gives, as $v\to0$,
\[
 \sqrt{N_n}\{\widehat\Gamma_{N_n}(c,v)-\Gamma(c,v,\lambda)\}
 \Rightarrow N\{0,I(c)\}
\]
with convergence in distribution uniform in $c\in\mathcal C$ and
$\lambda\in\Lambda(c)$.
Thus, for every $0<r<M<\infty$, on $v=xN_n^{-q}$,
\[
 W_{N_n}^\circ=\frac{\widehat\Gamma_{N_n}(c,v)}{I(c)v^k}
 \Rightarrow
 N\left\{\lambda,\frac1{I(c)x^{2k}}\right\}.
\]
The convergence in distribution is uniform in $c\in\mathcal C$,
$\lambda\in\Lambda(c)$, and $r\le|x|\le M$.
For later use, write
\[
 D_n(c,v,\lambda)=
 \left\|
  \bigl[Z_{N_n}(\widehat c,\widehat v,\lambda_{\widehat c})
        -Z_{N_n}(c,v,\lambda)\bigr]\one_{\mathcal H_n(c,v)}
 \right\|_2.
\]
If $v=xN_n^{-q}$ and $r\le|x|\le M$, then
\[
 \sqrt{N_n}D_n(c,v,\lambda)
 =|x|^{k-1}N_n^q|v|^{1-k}D_n(c,v,\lambda)=o(1)
\]
uniformly by \eqref{supp:eq:calibration-orthogonality}.
Together with $P\{\mathcal H_n(c,v)^c\}\to0$, this makes the same normal
limit true after replacing $(c,v,\lambda)$ by
$(\widehat c,\widehat v,\lambda_{\widehat c})$.  On
$\mathcal H_n(c,v)$, the Lipschitz endpoints give
\[
 |\lambda_{\widehat c}-\lambda|
 \le d_H\{\Lambda(\widehat c),\Lambda(c)\}
 \le C_\Lambda|\widehat c-c|
 \le C_\Lambda\varepsilon_n.
\]

We now construct a measurable family of normal decision rules for all
positive precisions. Fix $\eta>0$ and put
\[
 \lambda_-:=\min_{c\in\mathcal C}\ell(c),\qquad
 \lambda_+:=\max_{c\in\mathcal C}u(c),\qquad
 \mathcal A=\{(a,b,J):\lambda_-\le a\le b\le\lambda_+,\ J>0\}.
\]
For each $\vartheta=(a,b,J)\in\mathcal A$, choose a bounded Lipschitz
rule $f_\vartheta:\mathbb R\to[\lambda_-,\lambda_+]$ whose maximal risk
on $[a,b]$ is at most $M\{J^{-1},[a,b]\}+\eta/2$. Such a rule exists
because bounded rules can be approximated in risk by bounded Lipschitz
rules on compact Gaussian families. The risk of each fixed
$f_\vartheta$ and the normal minimax risk are continuous in $(a,b,J)$.
Thus there is a relative open neighborhood $U_\vartheta$ on which
\[
 \sup_{\lambda\in[a',b']}
 \E_{N(\lambda,{J'}^{-1})}\{f_\vartheta(W)-\lambda\}^2
 \le M\{{J'}^{-1},[a',b']\}+\eta,
 \qquad (a',b',J')\in U_\vartheta.
\]
Take a countable locally finite refinement $(U_j)$ of these
neighborhoods and a subordinate locally Lipschitz partition of unity
$(\alpha_j)$ on $\mathcal A$, and let $f_j$ denote the associated rule.
Define
\[
 D_\eta(a,b,J,w)
 =\Pi_{[a,b]}\sum_j\alpha_j(a,b,J)f_j(w).
\]
The sum is locally finite. Consequently $D_\eta$ is jointly Borel,
takes values in $[a,b]$, and is jointly Lipschitz on
$K\times\mathbb R$ for every compact $K\subset\mathcal A$.
For $\lambda\in[a,b]$, projection and convexity give
\[
 \{D_\eta(a,b,J,w)-\lambda\}^2
 \le\sum_j\alpha_j(a,b,J)\{f_j(w)-\lambda\}^2.
\]
The neighborhood risk bounds therefore show that $D_\eta$ is
$\eta$-minimax for every $(a,b,J)\in\mathcal A$.

Fix $0<r<M<\infty$ and put
\[
 \mathcal K_{r,M}^+
 =\{(c,x):c\in\mathcal C,\ r/2\le|x|\le2M\},
 \qquad
 \sigma^2(c,x)=\frac1{I(c)x^{2k}}.
\]
The set $\mathcal K_{r,M}^+$ is compact. Its interval endpoints and
precisions range over the compact subset
\[
 \mathcal A_{r,M}
 =\{(\ell(c),u(c),I(c)|x|^{2k}):(c,x)\in\mathcal K_{r,M}^+\},
\]
of $\mathcal A$, whose precisions are bounded away from zero. Put
\[
 \delta_\eta(c,x,w)
 =D_\eta\{\ell(c),u(c),I(c)|x|^{2k},w\}.
\]
This rule is uniformly continuous in $(c,x)$, uniformly Lipschitz in $w$,
and satisfies
\begin{equation*}
 \sup_{\lambda\in\Lambda(c)}
 \E_{W\sim N\{\lambda,\sigma^2(c,x)\}}
 \{\delta_\eta(c,x,W)-\lambda\}^2
 \le M\{\sigma^2(c,x),\Lambda(c)\}+\eta
\end{equation*}
uniformly on $\mathcal K_{r,M}^+$. This is the globally defined family
used in Section~4.2 of the main paper; the transfer argument below only
requires its uniform regularity on compact parameter ranges.

For $r\le|x|\le M$, let
\[
 \widehat x=N_n^q\widehat v,\qquad
 \widehat W_{N_n}^L
 =\frac{\widehat\Gamma_{N_n}(\widehat c,\widehat v)}
        {I(\widehat c)\widehat v^k},\qquad
 \widehat\tau_{N_n}^L
 =\tau\{\widehat c,\widehat v,
        \delta_\eta(\widehat c,\widehat x,\widehat W_{N_n}^L)\},
\]
with $\widehat W_{N_n}^L=0$ when $\widehat v=0$.  On
$\mathcal H_n(c,v)$, where $v=xN_n^{-q}$ and $r\le|x|\le M$,
\[
 |\widehat x-x|\le N_n^q\varepsilon_n=o(1),
 \qquad |\widehat c-c|\le\varepsilon_n=o(1).
\]
The preceding uniform normal limit, uniform continuity in $(c,x)$, and
the bounded Lipschitz dependence on $w$ consequently imply
\begin{align}
 &\sup_{\substack{c\in\mathcal C,\ \lambda\in\Lambda(c)\\r\le|x|\le M}}
 \left|
 \E\{\delta_\eta(\widehat c,\widehat x,\widehat W_{N_n}^L)-\lambda\}^2
 -\E_{W\sim N\{\lambda,\sigma^2(c,x)\}}
       \{\delta_\eta(c,x,W)-\lambda\}^2
 \right|\longrightarrow0. \label{supp:eq:normal-rule-transfer-supp}
\end{align}
Since the target expansion gives
\[
 N_n^q\tau(c,xN_n^{-q},\lambda)=x\beta(c)\lambda+o(1)
\]
uniformly for $r\le|x|\le M$, the same expansion at the estimated coordinates
gives, on $\mathcal H_n(c,v)$,
\[
 N_n^q\left[\widehat\tau_{N_n}^L-\widehat v\beta(\widehat c)
   \delta_\eta(\widehat c,\widehat x,\widehat W_{N_n}^L)\right]
 =o_{L^2}(1).
\]
Uniform continuity of $\beta$ and boundedness of $\delta_\eta$ then yield
\[
 N_n^q\{\widehat\tau_{N_n}^L-\tau(c,xN_n^{-q},\lambda)\}
 =x\beta(c)
  \{\delta_\eta(\widehat c,\widehat x,\widehat W_{N_n}^L)-\lambda\}
  +o_{L^2}(1)
\]
uniformly in $(c,x,\lambda)$.  Hence
\eqref{supp:eq:normal-rule-transfer-supp} yields
\begin{align*}
 &\sup_{\substack{c\in\mathcal C,\ \lambda\in\Lambda(c)\\r\le|x|\le M}}
 N_n^{1/k}\E\{\widehat\tau_{N_n}^L-\tau(c,xN_n^{-q},\lambda)\}^2\\
 &\quad\le
 \sup_{\substack{c\in\mathcal C\\r\le|x|\le M}}
 \beta(c)^2x^2
 \left[M\left\{\frac1{I(c)x^{2k}},\Lambda(c)\right\}+\eta\right]
 +o(1).
\end{align*}
The complement of $\mathcal H_n(c,v)$ contributes $o(1)$ by boundedness and
the probability bound in (C).  Finally, when $L(c)>0$, put
$z=L(c)\sqrt{I(c)}|x|^k$; then
\begin{align*}
 &\beta(c)^2x^2
 M\left\{\frac1{I(c)x^{2k}},\Lambda(c)\right\}\\
 &\qquad=
 \beta(c)^2I(c)^{-1/k}L(c)^{2-2/k}
 z^{-2+2/k}R_G(z)
 \le\mathsf C_k.
\end{align*}
If $L(c)=0$, the parameter $\lambda$ is known and the corresponding normal
minimax risk is zero, so the same bound holds.
Letting $\eta\downarrow0$ proves the claimed
 $\mathsf C_k+o(1)$ bound for every fixed $r\le|x|\le M$.

For fixed $(c,v)$ with $v\ne0$, the mean-value theorem and
\textup{(O)(i)} give
\[
 |\Gamma(c,v,\lambda)-\Gamma(c,v,\lambda')|
 \ge a_O|v|^k|\lambda-\lambda'|,
 \qquad \lambda,\lambda'\in\Lambda(c).
\]
Consequently, the inverse has Lipschitz constant at most
$a_O^{-1}|v|^{-k}$, for either sign of its derivative.
For the inverse estimator, project
$\widehat\Gamma_{N_n}(\widehat c,\widehat v)$ onto
$\Gamma\{\widehat c,\widehat v,\Lambda(\widehat c)\}$, invert in $\lambda$,
substitute into $\tau$, and clip to
$[-B_\tau,B_\tau]$, where $B_\tau=\|\tau\|_{\infty,\Xi}$; call the result
$\widehat\tau_{N_n}^I$.  This inverse is used only when $\widehat v\ne0$ and its
transition weight is positive; set it to zero otherwise.  Projection, the
Lipschitz endpoints, (C), and the variance bound in (E) give
\[
 \left\|
 \{\widehat\Gamma_{N_n}(\widehat c,\widehat v)
   -\Gamma(\widehat c,\widehat v,\lambda_{\widehat c})\}
 \one_{\mathcal H_n(c,v)}
 \right\|_{L^2(P_{c,v,\lambda}^{\otimes n})}
 \le C N_n^{-1/2}+D_n(c,v,\lambda).
\]
Moreover, $N_n^q\varepsilon_n\to0$, so on $\mathcal H_n(c,v)$ and
$N_n^q|v|\ge M>0$, one has $|\widehat v|\ge |v|/2$ for all large $n$.
The derivative bounds in (O) therefore give, uniformly
for $N_n^q|v|\ge M>0$,
\[
 \left\|
 \{\widehat\tau_{N_n}^I-\tau(c,v,\lambda)\}
 \one_{\mathcal H_n(c,v)}
 \right\|_{L^2(P_{c,v,\lambda}^{\otimes n})}
 \le C\left\{
 \frac{N_n^{-1/2}}{|v|^{k-1}}
 +\frac{D_n(c,v,\lambda)}{|v|^{k-1}}+\varepsilon_n\right\}.
\]
Consequently, for every fixed $M>0$,
\begin{equation}\label{supp:eq:transition-envelope-supp}
 \limsup_{n\to\infty}
 \sup_{\substack{c\in\mathcal C,\ \lambda\in\Lambda(c),\ 0<|v|\le v_0\\
                  N_n^q|v|\ge M}}
 N_n^{1/k}\E\{(\widehat\tau_{N_n}^I-\tau)^2\one_{\mathcal H_n(c,v)}\}
 \le\frac{C}{M^{2k-2}}.
\end{equation}
Indeed, after multiplication by $N_n^q$, the three terms in the preceding
display are bounded by
\[
 \frac{C}{(N_n^q|v|)^{k-1}},\qquad
 C N_n^q|v|^{1-k}D_n(c,v,\lambda),\qquad
 C N_n^q\varepsilon_n.
\]
The first is at most $CM^{-(k-1)}$, and the other two are $o(1)$ by
\eqref{supp:eq:calibration-orthogonality} and condition~\textup{(C)}.

\paragraph{Step 2: combining the four regions}
Set $\widehat\Psi_{O,N_n}=N_n^{-1}\sum_{i=1}^{N_n}\Phi_O(O_i)$.
Because $\Xi_O$ is compact and $\Psi_O$ is continuous, the measurable minimum
theorem gives a Borel choice
\[
 \widehat\xi_{O,N_n}\in\argmin_{\zeta\in\Xi_O}
 \|\widehat\Psi_{O,N_n}-\Psi_O(\zeta)\|_{\mathcal H_O},
 \qquad
 \widehat\tau_{N_n}^O=\tau(\widehat\xi_{O,N_n}).
\]
For $\xi\in\Xi_O$, the defining minimum and
\eqref{supp:eq:off-fibre-inverse}--\eqref{supp:eq:off-fibre-variance} give
\[
 \E_\xi(\widehat\tau_{N_n}^O-\tau(\xi))^2
 \le4A_O^2\E_\xi
 \|\widehat\Psi_{O,N_n}-\Psi_O(\xi)\|_{\mathcal H_O}^2
 \le\frac{4A_O^3}{N_n}.
\]

A diagonal argument applied to the bounds for fixed $r,M$ and
\eqref{supp:eq:transition-envelope-supp} gives
\[
 \begin{gathered}
 r_n\downarrow0,\qquad M_n\uparrow\infty,
 \qquad \eta_n\downarrow0,\qquad d_n\downarrow0,\\
 N_n^q\varepsilon_n=o(r_n),\qquad M_nN_n^{-q}\to0,
 \qquad M_n^2d_n\to0.
 \end{gathered}
\]
Use $\delta_{\eta_n}$ in the local rule, with
$\varepsilon_n^{\rm BN}=\eta_n$, and choose $d_n$ so that
\begin{align*}
 d_n\ge{}&\eta_n
 +\sup_{\substack{c\in\mathcal C,\ \lambda\in\Lambda(c)\\
                    r_n/2\le|x|\le3M_n}}
 \left|\E\{\delta_{\eta_n}(\widehat c,\widehat x,
                    \widehat W_{N_n}^L)-\lambda\}^2
 -\E_{W\sim N\{\lambda,\sigma^2(c,x)\}}
       \{\delta_{\eta_n}(c,x,W)-\lambda\}^2\right|\\
 &+\sup_{\substack{c\in\mathcal C,\ \lambda\in\Lambda(c)\\
                    r_n/2\le|x|\le3M_n}}
 \left\|N_n^q\{\widehat\tau_{N_n}^L-
       \tau(c,xN_n^{-q},\lambda)\}
 -x\beta(c)\{\delta_{\eta_n}(\widehat c,\widehat x,
       \widehat W_{N_n}^L)-\lambda\}\right\|_2.
\end{align*}
Let $\widehat\tau_{N_n}^0=0$ and let
$\widehat\tau_{N_n}^L,\widehat\tau_{N_n}^I,\widehat\tau_{N_n}^O$ denote the local,
inverse, and outer estimators.  Choose a continuous partition
$\omega_{b,n}\ge0$, $\sum_{b\in\{0,L,I,O\}}\omega_{b,n}=1$, with
\[
 \begin{aligned}
 \supp(\omega_{0,n})&\subset\{N_n^q|\widehat v|\le2r_n\},&
 \supp(\omega_{L,n})&\subset\{r_n\le N_n^q|\widehat v|\le2M_n\},\\
 \supp(\omega_{I,n})&\subset\{N_n^q|\widehat v|\ge M_n,
                              |\widehat v|\le7v_0/16\},&
 \supp(\omega_{O,n})&\subset\{|\widehat v|\ge3v_0/8\}.
 \end{aligned}
\]
The explicit ramp weights in Section~4.2 of the main paper satisfy these
support restrictions; use those weights below.  On $\mathcal H_n(c,v)$ and the event
controlling the coordinate estimation error in (C), the support restrictions imply,
for all large $n$,
$r_n/2\le N_n^q|v|\le3M_n$ where the local estimator has positive weight,
$N_n^q|v|\ge M_n/2$ and $|v_\circ(\xi)|\le v_0/2$ where the inverse estimator has positive weight,
and $|v_\circ(\xi)|\ge v_0/4$ where the outer estimator has positive weight.  Thus each estimator
has positive weight only where its risk bound holds.  The combined estimator is
\[
 \widehat\tau_n
 =\Pi_{[-B_\tau,B_\tau]}\!\left(
   \sum_{b\in\{0,L,I,O\}}\omega_{b,n}\widehat\tau_{N_n}^b\right).
\]
If $L(\widehat c)=0$, use the singleton formula in Section~4.2 of the main
paper.  On $\mathcal H_n(c,v)$, the Lipschitz endpoints give
\[
 L(c)\le C_\Lambda\varepsilon_n,\qquad
 |\lambda-\ell(\widehat c)|\le C_\Lambda\varepsilon_n.
\]
The derivative bounds in (O) therefore imply, wherever the nonouter
weight is positive,
\[
 |\tau\{\widehat c,\widehat v,\ell(\widehat c)\}
      -\tau(c,v,\lambda)|\le C\varepsilon_n.
\]
Its normalized squared risk is $O(N_n^{1/k}\varepsilon_n^2)=o(1)$;
the outer contribution has the bound above, and clipping controls the
exceptional events as in (C).

When $L(\widehat c)>0$, projection and convexity give pointwise,
since $\tau(\xi)\in[-B_\tau,B_\tau]$,
\[
 \begin{aligned}
 (\widehat\tau_n-\tau(\xi))^2
 &\le\left\{\sum_{b}\omega_{b,n}
       (\widehat\tau_{N_n}^b-\tau(\xi))\right\}^2\\
 &\le\sum_b\omega_{b,n}
       (\widehat\tau_{N_n}^b-\tau(\xi))^2.
 \end{aligned}
\]
On these events, the zero, local, inverse, and outer contributions are,
respectively, $O(r_n^2)$, $\mathsf C_k+O(M_n^2d_n)+o(1)$,
$O(M_n^{-(2k-2)})+o(1)$, and $O(N_n^{-1+1/k})$.  On their complements,
clipping and (C) give $o(1)$.  Hence
\begin{align*}
 \sup_{\xi\in\Xi}N_n^{1/k}
 \E_\xi(\widehat\tau_n-\tau(\xi))^2
 &\le \mathsf C_k+O(r_n^2)+O(M_n^2d_n)
      +O(M_n^{-(2k-2)})\notag\\
 &\qquad{}+O(N_n^{-1+1/k})+o(1).
\end{align*}
Since $N_n/n\to1$,
\[
 \limsup_{n\to\infty}n^{1/k}
 \sup_{\xi\in\Xi}\E_\xi(\widehat\tau_n-\tau(\xi))^2
 \le\mathsf C_k.
\]
The inequality
$\mathcal R_n\le\sup_{\xi\in\Xi}\E_\xi(\widehat\tau_n-\tau(\xi))^2$
and the lower bound proved above give both the global minimax limit and
attainment.
\end{proof}
\clearpage

\subsection{Minimax risk of the reduced experiment}
\label{supp:sec:supp-reduced-experiment}

Use the statistics
$T_n=(\widehat c,\widehat v,\widehat W_n,\widehat\Psi_{O,N_n})$
and the decision rule defined in Section~4.2 of the main paper.
The following argument proves Corollary~4.2 under the uniform setup and
conditions \textup{(E)}, \textup{(C)}, and \textup{(O)}.

\begin{proof}[Proof of Corollary~4.2]
Fix the measurable decision rules and outer selector used in Section~4.2 of the main paper.
When $\widehat v\ne0$,
\[
 \widehat\Gamma_{N_n}(\widehat c,\widehat v)
 =I(\widehat c)\widehat v^k\widehat W_n;
\]
when $\widehat v=0$, the local and inverse weights are zero.
The weights, local and inverse rules, and singleton rule thus depend
only on $(\widehat c,\widehat v,\widehat W_n)$, while the fixed outer
selector depends only on $\widehat\Psi_{O,N_n}$.  Hence
$\widehat\tau_n=h_n(T_n)$ for a measurable $h_n$.
Put
\[
 \mathcal R_n^T=\inf_{d_n}\sup_{\xi\in\Xi}
 \E_\xi\{d_n(T_n)-\tau(\xi)\}^2.
\]
Since $T_n$ is a statistic of the full sample,
\[
 \mathcal R_n\le\mathcal R_n^T
 \le\sup_{\xi\in\Xi}\E_\xi\{\widehat\tau_n-\tau(\xi)\}^2.
\]
Multiplying by $n^{1/k}$ and applying
\Cref{supp:thm:exact-globalization} proves the claimed limits.
\end{proof}

\section{Proofs of Theorems~5.1--5.3: the scalar fifth-order logistic phase}
\label{supp:sec:supp-fifth-order}

\subsection{Profile expansion and jet claim in Theorem~5.1}
\label{supp:sec:supp-logistic-jet}
For $p\ge1$, observe independent copies of $(Y,Z)$ from
\begin{equation*}
 X\sim N(0,aI_p),\qquad W\sim N(0,\eta I_p),\qquad Z=X+W,
 \qquad Y\mid X\sim\operatorname{Bernoulli}\{m(X^\top\theta)\},
\end{equation*}
with $X\perp W$, $Y\perp W\mid X$, and $m(x)=(1+e^{-x})^{-1}$.

\begin{equation}\label{supp:eq:variance-rect}
 a\in\cA=[a_-,a_+],\qquad \eta\in\cE=[\eta_-,\eta_+],
 \qquad 0<a_-<a_+<\infty,\quad0<\eta_-<\eta_+<\infty.
\end{equation}
Assume two distinct interior decompositions of one total variance:
\begin{equation}\label{supp:eq:common-c}
 (a_0,\eta_0)\ne(a_1,\eta_1),\qquad
 a_0+\eta_0=a_1+\eta_1=:c.
\end{equation}

For the scalar profile, take $p=1$.  Under \eqref{supp:eq:variance-rect},
the feasible total variances form
\begin{equation*}
 \cC=[a_-+\eta_-,a_++\eta_+],
\end{equation*}
and
\begin{equation*}
 \underline a(c)=\max\{a_-,c-\eta_+\},\qquad
 \overline a(c)=\min\{a_+,c-\eta_-\},
\end{equation*}
so the feasible interval for $\lambda$ is
\begin{equation}\label{supp:eq:Lambda-c}
 \Lambda(c)=\left[\frac c{\overline a(c)},
                    \frac c{\underline a(c)}\right].
\end{equation}
For $\lambda\in\Lambda(c)$,
\[
 a=c/\lambda,\qquad \eta=c(1-\lambda^{-1}).
\]
With
\begin{equation*}
 T=Z/\sqrt c,\qquad R=2Y-1,\qquad u=a\theta/\sqrt c,
\end{equation*}
we have $T\sim N(0,1)$ and
\begin{equation}\label{supp:eq:g-ulambda}
 \widetilde g_{u,\lambda}(t):=\E(R\mid T=t)
 =\E_G\tanh\left[\frac u2\{t+\sqrt{\lambda-1}G\}\right].
\end{equation}
Let $H_1(t)=t$, $H_3(t)=t^3-3t$, and
$H_5(t)=t^5-10t^3+15t$, and define
\begin{equation*}
 \Gamma_j(u,\lambda)=\E\{\widetilde g_{u,\lambda}(T)H_j(T)\},
 \qquad v=2\Gamma_1(u,\lambda).
\end{equation*}
Since $\frac{\partial}{\partial u}\{2\Gamma_1(u,\lambda)\}|_{u=0}=1$, compactness and the
implicit function theorem give $v_0>0$ and
\begin{equation*}
 u=U(v,\lambda),\qquad
 2\Gamma_1\{U(v,\lambda),\lambda\}=v,
\end{equation*}
uniformly for $|v|\le v_0$ and feasible $\lambda$.
Equivalently, the local inverse used in Section~5.1 of the main paper satisfies
\begin{equation*}
 U(v,\lambda)\,
 \E_G\sech^2\left(\frac{\sqrt\lambda\,U(v,\lambda)}2G\right)=v,
 \qquad U(0,\lambda)=0.
\end{equation*}

Write
\begin{align}
 g_{v,\lambda}&=\widetilde g_{U(v,\lambda),\lambda},&
 F(v,\lambda)&=\Gamma_3\{U(v,\lambda),\lambda\},\notag\\
 J(v,\lambda)&=F(v,\lambda)+\frac{v^3}{4},&
 \tau_L(c,v,\lambda)&=\frac{\lambda U(v,\lambda)}{\sqrt c}.
 \label{supp:eq:F-Theta}
\end{align}
Let $P^{(c)}_{v,\lambda}$ be the law of $(R,Z)$ at
$(\theta,a,\eta)=\{\tau_L(c,v,\lambda),c/\lambda,
c(1-\lambda^{-1})\}$.
For fixed $c$, the transformation $(R,Z)\mapsto(R,T)$ is invertible;
we compute the likelihood ratios in these standardized coordinates.
Here $g_{v,\lambda}(t)=g(t\mid v,\lambda)$ in the notation of
Section~5.1 of the main paper.

To apply the general theory, we first need a uniform regression expansion
and then the corresponding likelihood contrast.  The next lemma computes
the expansion and its differentiated remainders; the contact and
inverse-order claims in Theorem~5.1 follow in
\Cref{supp:sec:supp-logistic-proofs}.

\begin{lemma}[Uniform profile estimates]\label{supp:lem:profile-expansion}
Under \eqref{supp:eq:variance-rect}, for some $v_0>0$ and every fixed
$1\le q<\infty$, uniformly over
feasible $(c,\lambda)$ and $|v|\le v_0$,
\begin{align}
 U(v,\lambda)
 &=v+\frac{\lambda v^3}{4}+\frac{\lambda^2v^5}{16}
   +R_U(v,\lambda),
 \label{supp:eq:U-expansion}\\
 g_{v,\lambda}(t)
 &=\frac v2H_1(t)-\frac{v^3}{24}H_3(t)
   +\frac{\lambda v^5}{96}H_3(t)\notag\\
 &\quad+\frac{v^5}{240}H_5(t)
   +\bar r_{v,\lambda}(t),
 \label{supp:eq:profile-expansion}\\[2pt]
 \norm{\bar r_{v,\lambda}}_{L^q\{N(0,1)\}}
 +\norm{\frac{\partial}{\partial\lambda}\bar r_{v,\lambda}}_{L^q\{N(0,1)\}}
 &\le C_q|v|^7,\notag\\
 |\bar r_{v,\lambda}(t)|
 +|\frac{\partial}{\partial\lambda}\bar r_{v,\lambda}(t)|
 &\le C|v|^7(1+|t|^7).\label{supp:eq:profile-rem}
\end{align}
Set $R_J(v,\lambda)=J(v,\lambda)-\lambda v^5/16$.  Then
\begin{align}
 |R_U|+|\frac{\partial}{\partial\lambda} R_U|
 &\le C|v|^7,
 &|\frac{\partial}{\partial v}R_U|&\le C|v|^6,\notag\\
 |R_J|+|\frac{\partial}{\partial\lambda} R_J|
 &\le C|v|^7,
 &|\frac{\partial}{\partial v}R_J|&\le C|v|^6,
 \label{supp:eq:joint-profile-remainders}\\[2pt]
 F(v,\lambda)
 &=-\frac{v^3}{4}+\frac{\lambda v^5}{16}+O(v^7),
 &\tau_L(c,v,\lambda)
 &=\frac{\lambda v}{\sqrt c}+O(v^3),
 \label{supp:eq:Gamma3-profile}\\
 J(v,\lambda)
 &=\frac{\lambda v^5}{16}+O(v^7),\notag\\
 \frac{\partial}{\partial\lambda} J
 &=\frac{v^5}{16}+O(v^7),
 &\frac{\partial}{\partial v}J
 &=\frac{5\lambda v^4}{16}+O(v^6).
 \label{supp:eq:J-derivs}
\end{align}
\end{lemma}
\Cref{supp:lem:profile-expansion} locates the first regression term depending
on $\lambda$ at order five and controls the remainder uniformly, allowing the
likelihood conditions to be checked next.

\begin{proof}[Proof of \Cref{supp:lem:profile-expansion}]
\leavevmode\par
\paragraph{Hermite expansion}
Put $S=t+\sqrt{\lambda-1}G$.  The Gaussian addition identities are
\begin{align*}
 \E_GS&=H_1(t),\\
 \E_GS^3&=H_3(t)+3\lambda H_1(t),\\
 \E_GS^5&=H_5(t)+10\lambda H_3(t)+15\lambda^2H_1(t).
\end{align*}
Combining them with
$\tanh(x/2)=x/2-x^3/24+x^5/240+O(x^7)$ gives
\begin{align*}
 \widetilde g_{u,\lambda}
 ={}&\left(\frac u2-\frac{\lambda u^3}{8}
       +\frac{\lambda^2u^5}{16}\right)H_1\\
 &+\left(-\frac{u^3}{24}+\frac{\lambda u^5}{24}\right)H_3
 +\frac{u^5}{240}H_5+O_{L^q}(u^7).
\end{align*}
Boundedness of the seventh derivative of $\tanh(x/2)$ makes the conditional remainder at most
\[
 C|u|^7\E_G|t+\sqrt{\lambda-1}G|^7
 \le C|u|^7(1+|t|^7)
\]
uniformly on the compact $\lambda$-set.

\paragraph{Coordinate transformation}
The first Hermite coefficient satisfies
\[
 v=u-\frac\lambda4u^3+\frac{\lambda^2}{8}u^5+O(u^7).
\]
Writing $u=v+Av^3+Bv^5+O(v^7)$ and matching coefficients gives
$A=\lambda/4$ and $B=\lambda^2/16$.  In fact the implicit-function
calculation is uniform in $C^1(\Lambda)$:
\[
 U(v,\lambda)=v+\frac\lambda4v^3+\frac{\lambda^2}{16}v^5
 +R_U(v,\lambda).
\]

\paragraph{Differentiated bounds}
Every feasible $\lambda$ is bounded away from one because
$\lambda=1+\eta/a\ge1+\eta_-/a_+$.  We may therefore differentiate the
Gaussian expectation in \eqref{supp:eq:g-ulambda} with respect to $\lambda$ on
an open neighbourhood of the feasible compact set.  The bounded derivatives
of $\tanh$ give Taylor remainders
$|\mathcal R_7(z)|\le C|z|^7$ and
$|\mathcal R_7'(z)|\le C|z|^6$.  Differentiating the implicit equation in
$v$ and in $\lambda$, using its denominator
$\frac{\partial}{\partial u}\{2\Gamma_1(u,\lambda)\}=1+O(u^2)$, gives the first two bounds
for $R_U$ in \eqref{supp:eq:joint-profile-remainders}.

Together with
$\frac{\partial}{\partial\lambda} U(v,\lambda)=O(v^3)$ and Gaussian domination, these bounds
show that the differentiated seventh-order regression remainder is at most
$C|v|^7(1+|t|^7)$.

Substitution gives \eqref{supp:eq:profile-expansion}--\eqref{supp:eq:profile-rem}.

\paragraph{Moment and target expansions}
Orthogonality and $\E H_3(T)^2=6$ multiply the $H_3$ coefficient by six,
yielding the expansion for $F$ in \eqref{supp:eq:Gamma3-profile}.  The expansion
for $\tau_L$ follows from \eqref{supp:eq:U-expansion}.

Applying the differentiated Taylor bounds to $J=F+v^3/4$ gives the
remaining bounds in \eqref{supp:eq:joint-profile-remainders} and
\eqref{supp:eq:J-derivs}.
\end{proof}

The regression expansion must also give a likelihood ratio with positive
limiting information.  The following lemma verifies this requirement and
identifies the profiled-jet order.

\begin{lemma}[Fifth-order profiled jet]\label{supp:lem:fifth-jet}
For every feasible $c$ with nondegenerate $\Lambda(c)$ and
$\lambda_0\in\Lambda(c)$,
the profiled-jet conditions of Condition~1 in the main paper hold with
\[
 k=5,\qquad I_5=1/1536,\qquad \beta=c^{-1/2}.
\]
Then $k_{\mathrm J}=5$.
\end{lemma}

\Cref{supp:lem:fifth-jet} shows that variance decomposition first becomes
informative at fifth order along the profile, yielding the critical scale
$|v|\asymp n^{-1/10}$.

The proof below fixes $c$ and uses standardized observations. Its score and
likelihood remainder correspond to those in Section~5.1 of the main paper by
\[
 S_{c,v,\lambda_0}(R,Z)
 =S_{v,\lambda_0}(R,Z/\sqrt c),\qquad
 r(R,Z\mid c,v,\lambda;\lambda_0)
 =e_{v,\lambda;\lambda_0}(R,Z/\sqrt c).
\]
The regression remainder $\bar r_{v,\lambda}$ is separately defined in
\Cref{supp:lem:profile-expansion}.

\begin{proof}[Proof of \Cref{supp:lem:fifth-jet}]
For fixed feasible $c$ and $\lambda_0\in\Lambda(c)$, set
\begin{equation*}
 S_{v,\lambda_0}(R,T)
 =\frac{RH_3(T)}{96\{1+Rg_{v,\lambda_0}(T)\}},\qquad
 e_{v,\lambda;\lambda_0}(R,T)
 =\frac{R\{\bar r_{v,\lambda}(T)-
                  \bar r_{v,\lambda_0}(T)\}}
        {1+Rg_{v,\lambda_0}(T)}.
\end{equation*}

Abbreviate $g_0=g_{v,\lambda_0}$.  The marginal law of $T$ is $N(0,1)$ for every
$\lambda$, while
$P^{(c)}_{v,\lambda}(R=r\mid T=t)=\{1+rg_{v,\lambda}(t)\}/2$.  Therefore,
under $P^{(c)}_{v,\lambda_0}$,
\begin{equation}\label{supp:eq:fifth-one-rn}
 \frac{dP^{(c)}_{v,\lambda}}{dP^{(c)}_{v,\lambda_0}}(R,Z)-1
 =\frac{R\{g_{v,\lambda}(T)-g_0(T)\}}{1+Rg_0(T)}.
\end{equation}
For $r\in\{-1,1\}$, $1+r\tanh y\ge e^{-2|y|}$.  Jensen's inequality,
$U(v,\lambda_0)=O(v)$, and compactness give
\begin{equation*}
 1+rg_0(t)\ge\exp\{-C|v|(1+|t|)\}.
\end{equation*}
Hence every fixed polynomial in $T$ multiplied by a fixed negative power of
$1+Rg_0(T)$ has a uniformly bounded moment for small $|v|$.

Cancellation of the common terms in \eqref{supp:eq:profile-expansion} gives
\begin{equation*}
 g_{v,\lambda}-g_0
 =\frac{(\lambda-\lambda_0)v^5}{96}H_3
  +\bar r_{v,\lambda}-\bar r_{v,\lambda_0}.
\end{equation*}
Substitution in \eqref{supp:eq:fifth-one-rn} proves the exact jet
\begin{equation}\label{supp:eq:fifth-exact-jet}
 \frac{dP^{(c)}_{v,\lambda}}{dP^{(c)}_{v,\lambda_0}}(R,Z)-1
 =(\lambda-\lambda_0)v^5S_{v,\lambda_0}(R,T)
   +e_{v,\lambda;\lambda_0}(R,T).
\end{equation}

Conditioning on $T$ under $P^{(c)}_{v,\lambda_0}$ gives
\begin{align*}
 \E\left\{\frac{R}{1+Rg_0(T)}\,\middle|\,T\right\}&=0,
 \\
 \E(S_{v,\lambda_0}^2\mid T)
 &=\frac{H_3(T)^2}{96^2\{1-g_0(T)^2\}}.
\end{align*}
The denominator bound and dominated convergence yield
\begin{align}
 &\E_{v,\lambda_0}S_{v,\lambda_0}=0,\qquad
 \E_{v,\lambda_0}S_{v,\lambda_0}^2\longrightarrow
 I_5=\frac1{1536},\qquad
 \sup_{0<|v|\le v_0}\E_{v,\lambda_0}|S_{v,\lambda_0}|^4<\infty,
 \label{supp:eq:fifth-information-check}
\end{align}
since $\E H_3(T)^2/96^2=6/96^2=1/1536$.  The function and derivative bounds
in \eqref{supp:eq:profile-rem} imply
\[
 \norm{e_{v,\cdot;\lambda_0}(R,T)}_{C^1\{\Lambda(c)\}}
 \le C|v|^7(1+|T|^7)e^{C|v|(1+|T|)},
\]
and hence
\begin{equation*}
 \E_{v,\lambda_0}
 \norm{e_{v,\cdot;\lambda_0}}_{C^1\{\Lambda(c)\}}^2
 \le C|v|^{14}=o(|v|^{10}).
\end{equation*}
Continuity in $\lambda$ and joint
Borel measurability in $(\lambda,R,T)$ make the likelihood paths random
elements of the separable space $C\{\Lambda(c)\}$.

Finally, \eqref{supp:eq:Gamma3-profile} gives the target jet with
$\beta=1/\sqrt c$.  The $L^2$ norm of the likelihood contrast between any
two fixed, distinct values of $\lambda$ is of exact order $|v|^5$ by
\eqref{supp:eq:fifth-exact-jet} and \eqref{supp:eq:fifth-information-check}; after division by
$v^j$ it tends to zero for every $j<5$.  Thus no lower-order jet can have
positive limiting information, and $k_{\mathrm J}=5$.
\end{proof}

\subsection{Contact and inverse-order claims in Theorem~5.1}
\label{supp:sec:supp-logistic-proofs}

The fifth-order likelihood expansion gives pairs of parameters that remain
hard to distinguish.  The following corollary converts these pairs into
the scalar risk lower bound in Theorem~5.1.
\begin{corollary}[Fifth-order contact]\label{supp:cor:fifth-contact}\label{supp:lem:two-point}
For $B>0$, the scalar experiment on
$\Xi_B=[-B,B]\times\cA\times\cE$ has $k_{\mathrm C}\ge5$ and
\begin{equation*}
 \inf_{T_n}\sup_{\xi\in\Xi_B}\E_\xi(T_n-\theta)^2\ge c_Bn^{-1/5}.
\end{equation*}
\end{corollary}

\begin{proof}[Proof of the contact claim in Theorem~5.1]
Shrink the profile chart so that it lies in $\Xi_B$.  The fifth-order
likelihood jet gives contact pairs as in the proof of Theorem~2.5;
Theorem~2.4 then gives the lower bound.
\end{proof}

\Cref{supp:cor:fifth-contact} shows that the squared risk cannot decrease faster
than $n^{-1/5}$.  To establish equality of the identification orders, we
next bound target separation by observable moment separation.

\subsubsection{Uniform fifth-order inverse modulus}

Fix a compact scalar rectangle
\begin{equation*}
 \Xi_D=[-D,D]\times[\underline a,\overline a]\times
 [\underline\eta,\overline\eta],\qquad \underline a,\underline\eta>0.
\end{equation*}
Let $\mathcal O_L=\{-1,1\}\times\R$.  For a fixed bandwidth $b>0$, define
the bounded kernel
\begin{equation*}
 \kappa_b:\mathcal O_L\times\mathcal O_L\longrightarrow\R,\qquad
 \kappa_b\{(r,z),(r',z')\}=\one\{r=r'\}
 \exp\left\{-\frac{(z-z')^2}{2b^2}\right\}.
\end{equation*}
Let $\mathcal H_\kappa$ be its reproducing-kernel Hilbert space and set
$\mathcal H_L=\R^3\oplus\mathcal H_\kappa$.  Define
\begin{equation}\label{supp:eq:feature}
 \begin{aligned}
 \Phi_L&:\mathcal O_L\longrightarrow\mathcal H_L,&
 \Phi_L(r,z)&=\{z^2,rz,rz^3,\kappa_b((r,z),\cdot)\},\\
 \Psi_L&:\Xi_D\longrightarrow\mathcal H_L,&
 \Psi_L(\theta,a,\eta)&=\E_{\theta,a,\eta}\Phi_L(R,Z).
 \end{aligned}
\end{equation}

The following lemma gives a uniform inverse bound, including the
null parameter set where ordinary local inversion fails.

\begin{lemma}[Uniform fifth-order inverse modulus]\label{supp:lem:inverse-five}
There is $C_D<\infty$ such that, for every two points in $\Xi_D$,
\begin{equation}\label{supp:eq:inverse-five}
 |\theta-\theta'|^5
 \le C_D\norm{\Psi_L(\theta,a,\eta)-
                 \Psi_L(\theta',a',\eta')}_{\mathcal H_L}.
\end{equation}
\end{lemma}

\label{supp:app:proof-inverse-five}

\begin{proof}[Proof of \Cref{supp:lem:inverse-five}]
The marginal law of $Z$ identifies $c=a+\eta$.  Given $c$, define
\begin{equation*}
 q=\frac{a\theta}{c},\qquad
 d=\frac{\theta^2a\eta}{c},\qquad
 g_{q,d}(z)=\E_G\tanh\{(qz+\sqrt dG)/2\}.
\end{equation*}
The sign of $q$ is the sign of $\E(RZ)$.  If $q\ne0$, dominated convergence gives
\begin{equation*}
 1-\operatorname{sgn}(q)g_{q,d}(z)
 \sim2\exp(-|q|z+d/2),\qquad z\to+\infty.
\end{equation*}
Thus the tail slope and intercept identify $|q|$ and $d$, while $\E(RZ)$ identifies the sign.  The original parameters are recovered by
\begin{equation*}
 a=\frac{c^2}{c+d/q^2},\qquad \eta=c-a,
 \qquad \theta=\frac{qc}{a}.
\end{equation*}
If $q=0$, positivity of $a$ is equivalent to $\theta=0$.  Hence the observed law identifies $\theta$ everywhere and all three parameters away from zero.

The Gaussian kernel has a nowhere-vanishing Fourier transform, so the last
coordinate of \eqref{supp:eq:feature} embeds finite signed measures injectively,
separately for $R=1$ and $R=-1$.  On $|\theta|\ge\epsilon>0$, the map
$(\theta,a,\eta)\mapsto(c,q,d)$ is a smooth diffeomorphism onto its image.
Bounded kernel sections, smooth Gaussian densities, and uniform domination on
the compact parameter set show that $\Psi_L$ is Fr\'echet $C^1$ as a map into
$\mathcal H_L$.  If a derivative of the observed law vanished, the $Z$
marginal would first force $\dot c=0$.  Differentiating the Gaussian
expectation directly and using Gaussian exponential domination gives
\[
 \frac{\partial g_{q,d}(z)}{\partial q}
 =2z e^{-|q|z+d/2}\{1+o(1)\},\qquad
 \frac{\partial g_{q,d}(z)}{\partial d}
 =-\operatorname{sgn}(q)e^{-|q|z+d/2}\{1+o(1)\}.
\]
A vanishing directional derivative therefore first forces $\dot q=0$
and then $\dot d=0$.  Thus $D\Psi_L$ is
injective.  Its smallest singular value is continuous and has a positive
minimum on the compact set $|\theta|\ge\epsilon$.  The uniform Fr\'echet
remainder consequently gives
\[
 \norm{\Psi_L(\xi')-\Psi_L(\xi)}_{\mathcal H_L}
 \ge c_\epsilon\norm{\xi'-\xi}_2
\]
uniformly for sufficiently close points in this compact set.

Near $\theta=0$, use the profile coordinates and
\Cref{supp:lem:profile-expansion}, uniformly over the range induced by $\Xi_D$.

Let $\xi,\xi'$ lie in a sufficiently small neighbourhood of $\theta=0$,
with coordinates $(c,v,\lambda)$ and $(c',v',\lambda')$.
Put
$\delta=|\tau_L(c,v,\lambda)-\tau_L(c',v',\lambda')|$ and
$t=\max(|v|,|v'|)$.  Then $\delta\le Ct$.  If $v$ and $v'$ have opposite signs, the first three feature coordinates give
$\norm{\Psi_L(\xi)-\Psi_L(\xi')}_{\mathcal H_L}
\ge c|v-v'|\ge c\delta$, which is stronger than required after shrinking the neighbourhood so $\delta\le1$.

Suppose their signs agree.  If
\begin{equation}\label{supp:eq:case-easy}
 |c-c'|+|v-v'|\ge\kappa\delta^5,
\end{equation}
the moment identities below give the result.  Otherwise, $v,v'$ are nonzero
and comparable.  A uniform Taylor expansion of $\tau_L$ gives
\[
 \delta\le C\{|c-c'|+|v-v'|+t|\lambda-\lambda'|\}.
\]
For a sufficiently small neighbourhood, the first two terms are at most $\delta/2$, so
\begin{equation*}
 |\lambda-\lambda'|\ge c\delta/t.
\end{equation*}
After interchanging the points, take $|v|=t$.  Applying the mean-value theorem first in $\lambda$ and then in $v$, and using \eqref{supp:eq:J-derivs}, yields
\begin{align}
 |J(v,\lambda)-J(v',\lambda')|
 &\ge ct^5|\lambda-\lambda'|-Ct^4|v-v'|\notag\\
 &\ge ct^4\delta-C\kappa t^4\delta^5
 \ge ct^4\delta\ge c\delta^5.\label{supp:eq:J-separation}
\end{align}
The first three coordinates of $\Psi_L$ determine
\begin{equation*}
 c=\E Z^2,\qquad
 \Gamma_1=\frac{\E(RZ)}{\sqrt c},\qquad
 \Gamma_3=\frac{\E(RZ^3)}{c^{3/2}}-3\Gamma_1.
\end{equation*}
On the compact set their differences control
$|c-c'|+|v-v'|+
|J(v,\lambda)-J(v',\lambda')|$.  Hence
\eqref{supp:eq:case-easy}--\eqref{supp:eq:J-separation} prove the local fifth-order
bound.

If no global finite $C_D$ existed, compactness would give two convergent
parameter sequences for which the $\Psi_L$ difference divided by
$|\theta-\theta'|^5$ tends to zero.  Injectivity of the characteristic
embedding and target identification force the limiting observed laws, and
hence the limiting targets, to coincide.  A nonzero common target contradicts
the local first-order inverse bound.  If the common target is zero, both
sequences eventually enter the same zero neighbourhood; the fifth-order
estimate above is uniform over all, possibly different, nuisance limits in the
compact parameter set and gives the other contradiction.  This proves
\eqref{supp:eq:inverse-five}.
\end{proof}

\Cref{supp:lem:inverse-five} makes fifth-order recovery possible throughout the
compact parameter space.  Together with the contact bound, it determines
all three orders in Theorem~5.1.

\begin{proof}[Completion of the order claim in Theorem~5.1]
The fifth-order jet and the inverse modulus give
$5=k_{\mathrm J}\le k_{\mathrm C}\le k_{\mathrm I}\le5$.
\end{proof}

\subsection{Proof of Theorem~5.2: the local experiment and exact local risk}
\label{supp:sec:supp-logistic-limit}

Fix a feasible $c$ with nondegenerate $\Lambda(c)$ and $x\ne0$, and put
$v_n=xn^{-1/10}$.
The fifth-order expansion now permits a direct application of the general
limit and risk results; the following theorem gives their model-specific
form.

\begin{theorem}[Theorem~5.2 of the main paper]\label{supp:thm:logistic-limit}
The profiled experiment converges to
\begin{equation*}
 W_x\mid\lambda\sim N\left(\lambda,\frac{1536}{x^{10}}\right),
 \qquad\lambda\in\Lambda(c).
\end{equation*}
Its exact local target risk is
\begin{equation*}
 \lim_{n\to\infty}\inf_{T_n}\sup_{\lambda\in\Lambda(c)}
 n^{1/5}\E_{(P^{(c)}_{v_n,\lambda})^{\otimes n}}
 \{T_n-\tau_L(c,v_n,\lambda)\}^2
 =\frac{x^2}{c}M\left(\frac{1536}{x^{10}},\Lambda(c)\right).
\end{equation*}
\end{theorem}
\Cref{supp:thm:logistic-limit} gives the local estimation difficulty when the
total variance and scale are fixed; attaining the resulting constant
uniformly also requires estimating those coordinates.

\begin{proof}[Proof of Theorem~5.2]
Apply \Cref{supp:thm:profile-lan,supp:prop:normal-envelope} using
\Cref{supp:lem:fifth-jet}, with $k=5$, $I_5=1/1536$, and $\beta=c^{-1/2}$.
\end{proof}

\subsection{Calibration and inverse-estimator bounds for Theorem~5.3}\label{supp:sec:supp-calibration}
The local rule uses unknown coordinates, and inversion above the critical
scale amplifies their estimation errors.  The bounds below control these
effects when verifying conditions \textup{(C)} and \textup{(O)} for
Theorem~5.3.
Use $F$ and $\tau_L$ from \eqref{supp:eq:F-Theta}.  For constants
$v_1>0$ and $0<k_1<K_1<\infty$ depending on $B$ and the fixed variance
bounds, uniformly over
$c\in\cC$, $\lambda\in\Lambda(c)$, and $0<|v|\le v_1$,
\begin{align*}
 k_1|v|^5&\le|\frac{\partial}{\partial\lambda} F(v,\lambda)|\le K_1|v|^5,
 &|\frac{\partial}{\partial v}F(v,\lambda)|&\le K_1|v|^2,\\
 |\frac{\partial}{\partial\lambda}\tau_L(c,v,\lambda)|&\le K_1|v|,
 &|\frac{\partial}{\partial v}\tau_L|+|\frac{\partial}{\partial c}\tau_L|&\le K_1.
\end{align*}
For $c\in\cC$ and $0<|v|\le v_1$,
\[
 F(v,\cdot):\Lambda(c)\longrightarrow F\{v,\Lambda(c)\}
\]
is a homeomorphism with inverse $F(v,\cdot)^{-1}$.

On an independent calibration fold of size $m=m_N$, define
\begin{equation}\label{supp:eq:calibration-c-v}
 \widehat c=\Pi_{\cC}\left(\frac1m\sum_{i=1}^mZ_i^2\right),
 \qquad
 \overline v=\frac2m\sum_{i=1}^m\frac{R_iZ_i}{\sqrt{\widehat c}},
 \qquad
 \widehat v=\Pi_{[-v_1,v_1]}(\overline v).
\end{equation}
On a disjoint main fold of size $N$, for $\widehat v\ne0$, define
\begin{equation}\label{supp:eq:W-plugin}
 \widehat W_N
 =\frac{16}{\widehat v^5}
 \left\{\frac1N\sum_{i=1}^NR_iH_3(Z_i/\sqrt{\widehat c})
 +\frac{\widehat v^3}{4}\right\}.
\end{equation}
Also set
\[
 \widehat\Gamma_{3,N}(\widehat c)
 =\frac1N\sum_{i=1}^N R_iH_3(Z_i/\sqrt{\widehat c}),
\]
For $0<|\widehat v|\le v_1$, define
\begin{align}
 \widehat\lambda_N^{\mathrm{inv}}
 &=F(\widehat v,\cdot)^{-1}\left[
   \Pi_{F\{\widehat v,\Lambda(\widehat c)\}}
   \{\widehat\Gamma_{3,N}(\widehat c)\}\right],
 \label{supp:eq:plugin-lambda-inverse}\\
\widehat\theta_N^{\mathrm{inv}}
 &=\Pi_{[-C_0|\widehat v|,C_0|\widehat v|]}
   \tau_L(\widehat c,\widehat v,
          \widehat\lambda_N^{\mathrm{inv}}),
 \label{supp:eq:plugin-theta-inverse}
\end{align}
where
$C_0\ge\sup_{c\in\cC,\lambda\in\Lambda(c),0<|v|\le v_1}
|\tau_L(c,v,\lambda)|/|v|$.  Set both quantities to zero
at $\widehat v=0$, and set $\widehat\theta_N^{\mathrm{inv}}=0$ outside
$|\widehat v|\le v_1$.

The following lemma bounds the effect of coordinate estimation on both the
local normal statistic and the inverse estimator, including failed
calibration samples.

\begin{lemma}[Calibration error bounds]\label{supp:lem:calibration}
Assume $m=o(N)$ and $m/N^{4/5}\to\infty$.  For every fixed
$0<r<M<\infty$, uniformly over
\[
 c\in\cC,\qquad \lambda\in\Lambda(c),\qquad
 v=xN^{-1/10},\qquad r\le|x|\le M,
\]
one has
\begin{equation}\label{supp:eq:plugin-normal}
 \widehat W_N\Rightarrow N\left(\lambda,\frac{1536}{x^{10}}\right)
\end{equation}
uniformly in bounded-Lipschitz distance.

For sufficiently small fixed $q>0$, let
\begin{equation*}
 \cG_{c,v}=\set{|\widehat c-c|\le q|v|^4,
                  |\widehat v-v|\le q|v|^3}.
\end{equation*}
The plug-in inverse rule in
\eqref{supp:eq:plugin-lambda-inverse}--\eqref{supp:eq:plugin-theta-inverse}, clipped at
a constant multiple of $|\widehat v|$, satisfies
\begin{equation}\label{supp:eq:plugin-inverse-risk}
 \E\{(\widehat\theta_N^{\mathrm{inv}}-\theta)^2
       \one_{\cG_{c,v}}\}
 \le C\left\{\frac1{N|v|^8}+
              \frac1{m|v|^4}+\frac1{m|v|^6}\right\}.
\end{equation}
For $0<|v|\le v_1$,
\begin{equation}\label{supp:eq:Gcv-prob}
 \Pp(\cG_{c,v}^c)\le C\exp(-cm|v|^8).
\end{equation}
Every estimator clipped at $C_0|\widehat v|$ also obeys
\begin{equation}\label{supp:eq:failure-risk}
 \E\{(\widehat\theta_N-\theta)^2\one_{\cG_{c,v}^c}\}
 \le C|v|^2e^{-cm|v|^8}+C/m.
\end{equation}
\end{lemma}
\Cref{supp:lem:calibration} separates empirical fluctuation, coordinate error,
and failed calibration, the three contributions that must be controlled
to extend the local risk bound uniformly.

\begin{proof}[Proof of \Cref{supp:lem:calibration}]
Chi-square and sub-Gaussian concentration, clipping, and integration give, for every fixed integer $j\ge1$,
\begin{equation}\label{supp:eq:calibration-moments}
 \E|\widehat c-c|^j+\E|\overline v-v|^j\le C_jm^{-j/2}.
\end{equation}
When $|v|\le v_1$, projection onto $[-v_1,v_1]$ cannot increase the
distance to $v$, so the same bound holds with $\widehat v$ in place of
$\overline v$.  The unprojected $\overline v$ selects the outer estimator uniformly over
the full compact parameter space.
The exact rescaling identity
\begin{equation}\label{supp:eq:H3-rescale}
 H_3(sT)=s^3H_3(T)+3s(s^2-1)T
\end{equation}
applies with $s=(c/\widehat c)^{1/2}$.  Since $s-1=O(\widehat c-c)$, the empirical averages of $RH_1(T)$ and
$RH_3(T)$ are $O_{\Pp}(|v|+N^{-1/2})$ and
$O_{\Pp}(|v|^3+N^{-1/2})$, respectively.  Replacing $c$ by $\widehat c$ in \eqref{supp:eq:W-plugin} therefore changes it by
\[
 O_{\Pp}\{|\widehat c-c|/|v|^4\}.
\]
Replacing $v$ by $\widehat v$ contributes
$O_{\Pp}\{|\widehat v-v|/|v|^3\}$, since differentiating
$16\Gamma/v^5+4/v^2$ with
$\Gamma=-v^3/4+O(v^5)$ gives order $v^{-3}$.  At
$v=xN^{-1/10}$, both errors are $o_{\Pp}(1)$ uniformly on fixed annuli because $m/N^{4/5}\to\infty$.  The triangular-array CLT from
\Cref{supp:lem:fifth-jet,supp:thm:profile-lan} proves
\eqref{supp:eq:plugin-normal}.

On $\cG_{c,v}$, the differentiated-profile bounds at the beginning of this
subsection, \eqref{supp:eq:H3-rescale}, and the mean-value theorem give
\[
 |\widehat\theta_N^{\mathrm{inv}}-\theta|
 \le C\left\{
 \frac{|\widehat\Gamma_{3,N}-F(v,\lambda)|}{|v|^4}
 +\frac{|\widehat v-v|}{|v|^2}
 +\frac{|\widehat c-c|}{|v|^3}\right\}.
\]
Fold independence and \eqref{supp:eq:calibration-moments} prove \eqref{supp:eq:plugin-inverse-risk}.  Concentration at thresholds $|v|^4$ and $|v|^3$ gives exponents $m|v|^8$ and $m|v|^6$; the first is smaller for small $v$, proving \eqref{supp:eq:Gcv-prob}.  Uniformly on the profiled set, $|\theta|\le C|v|$.  A rule clipped at $C_0|\widehat v|$ satisfies
\[
 (\widehat\theta_N-\theta)^2
 \le C\{|v|^2+(\widehat v-v)^2\}.
\]
Multiplying by the indicator of the complementary event and using \eqref{supp:eq:Gcv-prob} and \eqref{supp:eq:calibration-moments} proves \eqref{supp:eq:failure-risk}.
\end{proof}

\subsection{Verification of conditions \textup{(E)}, \textup{(C)}, and \textup{(O)} for Theorem~5.3}
\label{supp:sec:supp-logistic-global-verification}

The profile expansion and fifth-order jet give
\begin{equation}\label{supp:eq:logistic-global-inputs}
 k=5,\qquad I(c)=\frac1{1536},\qquad \beta(c)=c^{-1/2}.
\end{equation}
Fix the profile radius $v_1>0$ small enough for the uniform bounds below
and for $|\tau_L|\le B$ throughout the chart. At the two
endpoints of $\cC$, $\Lambda(c)$ is a singleton; no estimation of its
fixed value is needed, and the singleton rule in Section~4.2 applies.

We first identify an observed statistic whose signal and variance match
the fifth-order likelihood information.

\begin{lemma}[Efficient local statistic \textup{(E)}]
\label{supp:lem:logistic-E}
Define
\begin{equation*}
 \phi_{c,v}(r,z)
 =\frac{rH_3(z/\sqrt c)+v^3/4}{96},
 \qquad
 \Gamma_L(c,v,\lambda)
 =\E_{c,v,\lambda}\{\phi_{c,v}(R,Z)\}.
\end{equation*}
Uniformly over feasible $c$, $\lambda$, and $|v|\le v_1$,
\begin{align*}
 \Gamma_L(c,v,\lambda)
 &=\frac{\lambda v^5}{1536}+O(v^7),
 &\frac{\partial}{\partial\lambda}\Gamma_L(c,v,\lambda)
 &=\frac{v^5}{1536}+O(v^7),\\
 \operatorname{Var}_{c,v,\lambda}\{\phi_{c,v}(R,Z)\}
 &=\frac1{1536}+O(v^6).
\end{align*}
The centered statistic has a uniformly bounded fourth moment, so condition
\textup{(E)} holds.
\end{lemma}

\par\smallskip
\begin{proof}[Proof of \Cref{supp:lem:logistic-E}]
By \eqref{supp:eq:F-Theta},
\[
 \Gamma_L(c,v,\lambda)
 =\frac{F(v,\lambda)+v^3/4}{96}
 =\frac{J(v,\lambda)}{96}.
\]
The mean and derivative expansions follow from
\eqref{supp:eq:Gamma3-profile}--\eqref{supp:eq:J-derivs}.

The added term $v^3/4$ is deterministic.  Since $R^2=1$ and
$\E H_3(T)^2=3!=6$,
\[
 \operatorname{Var}_{c,v,\lambda}\{\phi_{c,v}(R,Z)\}
 =\frac{6-F(v,\lambda)^2}{96^2}
 =\frac1{1536}+O(v^6).
\]
Since $|R|=1$, $v$ is bounded, and $T$ is standard normal, the fourth
moments are uniformly bounded.  Thus condition~\textup{(E)} holds with
$\zeta=2$ and $\delta=2$.
\end{proof}

\Cref{supp:lem:logistic-E} verifies the information matching needed for the
exact constant.  The statistic depends on unknown coordinates, so we next
control the effect of estimating them from an independent sample.

\begin{lemma}[Calibration error \textup{(C)}]
\label{supp:lem:logistic-C}
Let $m_n=o(n)$, $N_n=n-m_n$, and suppose
\begin{equation*}
 \frac{m_n}{N_n^{4/5}\log N_n}\longrightarrow\infty.
\end{equation*}
Use the calibration estimates in \eqref{supp:eq:calibration-c-v} and define
\begin{equation*}
 v_{\circ,L}(\theta,a,\eta)
 =\Pi_{[-v_1,v_1]}
 \left\{\frac{2\E_{\theta,a,\eta}(RZ)}{\sqrt{a+\eta}}\right\}.
\end{equation*}
On the profile chart, $v_{\circ,L}=v$.  The chart contains
$\{\xi:|v_{\circ,L}(\xi)|\le v_1/2\}$, and condition \textup{(C)} holds.
\end{lemma}

\par\smallskip
\begin{proof}[Proof of \Cref{supp:lem:logistic-C}]
On the profile chart,
\[
 \frac{2\E_{c,v,\lambda}(RZ)}{\sqrt c}
 =2\E_{c,v,\lambda}(RT)=v,
\]
so $v_{\circ,L}=v$.  Stein's identity gives
\[
 \E(RZ)=\kappa(a,|\theta|)\theta,
 \qquad
 \kappa(a,t)=2a\E m'(\sqrt a\,tG).
\]
The function $\kappa$ is continuous and bounded away from zero on $\Xi_B$.
Before clipping, the coordinate therefore controls $|\theta|$ uniformly.
After decreasing $v_1$, the chart contains
$\{\xi:|v_{\circ,L}(\xi)|\le v_1/2\}$.

Take the calibration tolerance
\[
 \varepsilon_n=C_{\rm cal}\sqrt{\frac{\log N_n}{m_n}}.
\]
Uniform chi-square and Gaussian-product concentration, followed by the two
projections in \eqref{supp:eq:calibration-c-v}, gives the estimation bounds in condition~\textup{(C)}, globally and on the
chart, for sufficiently large fixed $C_{\rm cal}$.  Moreover,
\[
 N_n^{2/5}\varepsilon_n
 =C_{\rm cal}
 \left(\frac{N_n^{4/5}\log N_n}{m_n}\right)^{1/2}
 \longrightarrow0.
\]
In particular, $N_n^{1/10}\varepsilon_n\to0$.

It remains to check the plug-in residual.  On
$\mathcal H_n(c,v)$ put $s=(c/\widehat c)^{1/2}$.  The exact identity
\[
 H_3(sT)=s^3H_3(T)+3s(s^2-1)T
\]
and the Lipschitz endpoints of $\Lambda(c)$, together with
\eqref{supp:eq:joint-profile-remainders}, control the changes in $c$, $v$, and
$\lambda$.  Conditional on the calibration fold, independence and
Gaussian moments therefore give, uniformly in $(c,v,\lambda)$,
\[
 \begin{aligned}
 &\left\|
 \bigl\{Z_{N_n}(\widehat c,\widehat v,\lambda_{\widehat c})
       -Z_{N_n}(c,v,\lambda)\bigr\}
  \one_{\mathcal H_n(c,v)}
 \right\|_2\\
 &\qquad\le C\left(|v|\varepsilon_n+\varepsilon_n^2
             +\frac{\varepsilon_n}{\sqrt{N_n}}\right).
 \end{aligned}
\]
Hence, for every fixed $h>0$ and
$hN_n^{-1/10}\le|v|\le v_1$,
\[
 \begin{aligned}
 &N_n^{1/10}|v|^{-4}
 \left\|
  \bigl\{Z_{N_n}(\widehat c,\widehat v,\lambda_{\widehat c})
       -Z_{N_n}(c,v,\lambda)\bigr\}
  \one_{\mathcal H_n(c,v)}
 \right\|_2\\
 &\qquad\le C_h\left(
 N_n^{2/5}\varepsilon_n
 +N_n^{1/2}\varepsilon_n^2+\varepsilon_n\right)
 \longrightarrow0.
 \end{aligned}
\]
This is \eqref{supp:eq:calibration-orthogonality} for $k=5$, so
condition~\textup{(C)} holds.
\end{proof}

\Cref{supp:lem:logistic-C} shows that the calibration errors are negligible at
the target risk scale.  Uniform attainment also requires stable recovery
outside the critical region, which the next lemma establishes.

\begin{lemma}[Stable recovery away from the null parameter set \textup{(O)}]
\label{supp:lem:logistic-O}\label{supp:lem:differentiated-profile}
After decreasing $v_1$ if necessary, uniformly over feasible $c$, $\lambda$,
and $0<|v|\le v_1$,
\begin{align}
 k_1|v|^5
 &\le|\frac{\partial}{\partial\lambda}\Gamma_L(c,v,\lambda)|
 \le K_1|v|^5,
 \label{supp:eq:GammaL-lambda-deriv}\\
 |\frac{\partial}{\partial\lambda}\tau_L(c,v,\lambda)|
 &\le K_1|v|,
 &|\frac{\partial}{\partial v}\tau_L|+|\frac{\partial}{\partial c}\tau_L|&\le K_1.
 \label{supp:eq:tauL-derivs}
\end{align}
Here $0<k_1<K_1<\infty$.  For fixed nonzero $v$, the map
$\lambda\mapsto\Gamma_L(c,v,\lambda)$ is one-to-one.  With
\begin{equation*}
 \Xi_O=\{\xi\in\Xi_B:|v_{\circ,L}(\xi)|\ge v_1/8\},
 \qquad \mathcal H_O=\mathcal H_L,
 \qquad \Phi_O=\Phi_L,
\end{equation*}
condition \textup{(O)} holds.
\end{lemma}

\par\smallskip
\begin{proof}[Proof of \Cref{supp:lem:logistic-O}]
The differentiated profile expansion \eqref{supp:eq:J-derivs} gives
\[
 \frac{\partial}{\partial\lambda} J(v,\lambda)
 =\frac{v^5}{16}+O(v^7).
\]
The $C^1$ form of \eqref{supp:eq:U-expansion} gives
$\frac{\partial}{\partial v}U=1+O(v^2)$ and $\frac{\partial}{\partial\lambda} U=O(v^3)$.  Hence
\[
 \frac{\partial}{\partial\lambda}\tau_L=\frac v{\sqrt c}+O(v^3),
 \qquad
 \frac{\partial}{\partial v}\tau_L=O(1),
 \qquad
 \frac{\partial}{\partial c}\tau_L=O(v),
\]
uniformly.  Since $\Gamma_L=J/96$, decreasing $v_1$ proves
\eqref{supp:eq:GammaL-lambda-deriv}--\eqref{supp:eq:tauL-derivs}.  Since $\frac{\partial}{\partial\lambda} J$ has the sign of $v^5$, the map
$\lambda\mapsto\Gamma_L(c,v,\lambda)$ is one-to-one.  This verifies the local
requirements in condition~\textup{(O)}.

On $\Xi_O$, $|v_{\circ,L}|\ge v_1/8$.  The proof of
\Cref{supp:lem:inverse-five} away from the null parameter set gives
\[
 |\theta-\theta'|
 \le C\|\Psi_L(\xi)-\Psi_L(\xi')\|_{\mathcal H_L},
 \qquad \xi,\xi'\in\Xi_O.
\]
The RKHS coordinate of $\Phi_L$ has norm one, and its remaining coordinates
require only Gaussian moments through order six.  Thus
\[
 \sup_{\xi\in\Xi_O}\E_\xi
 \|\Phi_L(R,Z)-\Psi_L(\xi)\|_{\mathcal H_L}^2<\infty.
\]
This completes the verification of condition~\textup{(O)}.
\end{proof}

\Cref{supp:lem:logistic-O} supplies the inverse and outer-region bounds needed
to control risk away from the hardest local scale.  The three lemmas now
permit application of Theorem~4.1 to the scalar logistic model.
\subsection{Exact global-risk claim in Theorem~5.3}
\label{supp:sec:supp-scalar-exact}

For the scalar parameter space
$\Xi_B=[-B,B]\times\cA\times\cE$, define
\[
 \cR_n^{\mathrm{sc}}(B)=
 \inf_{T_n}\sup_{(\theta,a,\eta)\in\Xi_B}
 \E_{\theta,a,\eta}(T_n-\theta)^2.
\]

Write $K_5$ for $K_k$ in \eqref{supp:eq:RG-Kk} at $k=5$, and define
\begin{equation*}
 C_5(\cA,\cE)
 =1536^{1/5}K_5\sup_{c\in\cC}
 \frac{\{\diam\Lambda(c)\}^{8/5}}{c}.
\end{equation*}

The following theorem combines the local lower bound with the preceding
verification of \textup{(E)}, \textup{(C)}, and \textup{(O)} to determine
the scalar minimax constant.

\begin{theorem}[Theorem~5.3 of the main paper]
\label{supp:thm:exact-scalar}
For every fixed $B>0$, under \eqref{supp:eq:variance-rect}--\eqref{supp:eq:common-c},
\begin{equation}\label{supp:eq:exact-scalar}
 \lim_{n\to\infty}n^{1/5}\cR_n^{\mathrm{sc}}(B)
 =C_5(\cA,\cE).
\end{equation}
The upper bound is attained by one measurable rule that does not know $a$ or
$\eta$.
\end{theorem}

\begin{proof}[Proof of Theorem~5.3]
Take $m_n=\lceil n^{9/10}\rceil$ and use
\eqref{supp:eq:logistic-global-inputs}.  Conditions \textup{(E)}, \textup{(C)},
and \textup{(O)} hold by
\Cref{supp:lem:logistic-E,supp:lem:logistic-C,supp:lem:logistic-O}.
Substitution in \eqref{supp:eq:general-exact-constant} and
\Cref{supp:thm:exact-globalization} give \eqref{supp:eq:exact-scalar}.
\end{proof}

Theorem~5.3 shows that estimating the unknown signal and error variances
does not increase the constant determined by the hardest local paths.
The sparse extension must additionally control estimation of the signal
direction.

\section{Proof of Theorem~5.4: the sparse fifth-order logistic phase}
\label{supp:sec:supp-sparse-fifth-order}

For $1\le s\le p/2$ and $B>0$, let $\norm{\theta}_0$ denote the number of
nonzero coordinates of $\theta$, and set
\[
 \Theta_{p,s}(B)
 =\set{\theta\in\R^p:\norm{\theta}_0\le s,\ \norm{\theta}_2\le B}.
\]
Also write
\[
 B_2^p(B)=\{z\in\R^p:\norm z_2\le B\},
 \qquad J_p(s)=s\log(ep/s).
\]
We use the convention $J_p(0)=0$.
The minimax risk with unknown nuisance parameters is
\begin{equation*}
 \cR^{\mathrm{joint}}_{n,p,s}(B)
 =\inf_{T_n}\sup_{a\in\cA,\eta\in\cE}
  \sup_{\theta\in\Theta_{p,s}(B)}
  \E_{\theta,a,\eta}\norm{T_n-\theta}_2^2.
\end{equation*}

\subsection{Minimax rate in Theorem~5.4}
\label{supp:sec:supp-sparse-rate}

The scalar upper bound and sparse mean-estimation lemma below support the
minimax-rate assertion of Theorem~5.4.  We first record the scalar rule
that will be applied after estimating a direction.

\begin{corollary}[Scalar minimum-distance upper bound]\label{supp:cor:scalar-md}
Let $(R_i,Z_i)_{i=1}^m$ follow the scalar model over $\Xi_D$, and set
\[
 \widehat\Psi_{L,m}=\frac1m\sum_{i=1}^m\Phi_L(R_i,Z_i).
\]
Choose a measurable minimizer
\begin{equation}\label{supp:eq:md-estimator}
 (\widehat\theta_m,\widehat a_m,\widehat\eta_m)
 \in\argmin_{(t,\alpha,e)\in\Xi_D}
 \norm{\widehat\Psi_{L,m}-\Psi_L(t,\alpha,e)}_{\mathcal H_L}.
\end{equation}
Then
\begin{equation}\label{supp:eq:md-risk}
 \sup_{(\theta,a,\eta)\in\Xi_D}
 \E_{\theta,a,\eta}(\widehat\theta_m-\theta)^2
 \le C_Dm^{-1/5}.
\end{equation}
\end{corollary}

\begin{proof}[Proof of the scalar bound]
The feature has uniformly bounded second moment.  Its RKHS coordinate has
norm one; its polynomial coordinates require only Gaussian moments through
order six.  Theorem~2.4 and the inverse modulus give \eqref{supp:eq:md-risk}.
\end{proof}

\Cref{supp:cor:scalar-md} gives the required fifth-order scalar risk bound
without knowledge of the two variances.  Applying it in a sparse model
requires an estimate of the signal direction.

\subsubsection{Adaptive minimax rate}

Let $p=p_n$, let $1\le s_{+,n}\le p_n/2$, and suppose
\begin{equation}\label{supp:eq:sparse-envelope}
 \frac{J_{p_n}(s_{+,n})}{n}\to0.
\end{equation}
Split the data into independent folds $\cD_0,\cD_1$ of sizes
$n_0\asymp n_1\asymp n$.  Define
\begin{equation*}
 \kappa(a,t)=2a\E m'(\sqrt a\,tG)
 =\frac a2\E\sech^2(\sqrt a\,tG/2).
\end{equation*}
Compactness gives constants $0<\kappa_-\le\kappa_+<\infty$ uniformly over $(a,t)\in\cA\times[0,B]$.
These constants depend only on the specified parameter bounds.

On $\cD_0$, put
\[
 V_0=\frac1{n_0}\sum_{i\in\cD_0}R_iZ_i,
 \qquad M_\mu=\kappa_+B,
\]
and choose a measurable minimizer
\begin{equation}\label{supp:eq:pilot}
 \widehat\mu\in\argmin_{\substack{\norm{u}_0\le s_{+,n}\\
                                  \norm{u}_2\le M_\mu}}
 \left\{\norm{V_0-u}_2^2+
 \Lambda_0(a_++\eta_+)\frac{J_{p_n}(\norm u_0)}{n_0}\right\},
\end{equation}
where $\Lambda_0$ is a sufficiently large numerical constant.

The next lemma shows that the observable mean has the signal direction
and bounds the pilot error at the sparse estimation rate.

\begin{lemma}[Adaptive sparse mean estimation]\label{supp:lem:pilot}
Uniformly over $1\le s\le s_{+,n}$, $a\in\cA$, $\eta\in\cE$, and $\theta\in\Theta_{p_n,s}(B)$,
\begin{equation}\label{supp:eq:mu-direction}
 \mu:=\E(RZ)=\kappa(a,\norm\theta_2)\theta,
\end{equation}
and
\begin{equation}\label{supp:eq:pilot-moments}
 \E\norm{\widehat\mu-\mu}_2^2\le C\frac{J_{p_n}(s)}{n_0},
 \qquad
 \E\norm{\widehat\mu-\mu}_2^4
 \le C\left\{\frac{J_{p_n}(s)}{n_0}\right\}^2.
\end{equation}
\end{lemma}

\Cref{supp:lem:pilot} separates direction estimation from the fifth-order
scalar problem.  We use its pilot direction to project the second fold
and then apply the scalar rule.

If $\widehat\mu=0$, set $\widehat\theta^{\mathrm{jr}}=0$.  Otherwise define
\begin{equation}\label{supp:eq:h-f}
 h=\norm{\widehat\mu}_2,\qquad f=\widehat\mu/h,
\end{equation}
and apply the scalar minimum-distance estimator \eqref{supp:eq:md-estimator} on $\cD_1$ to $(R_i,f^\top Z_i)$ over
\begin{equation*}
 K^\dagger=[-3B,3B]\times[a_-/9,a_+]\times
 [\eta_-,\eta_++a_+].
\end{equation*}
Let $\widetilde\theta_f$ denote the resulting coefficient estimate.  For fixed $C_0\ge6/\kappa_-$, put
\begin{equation}\label{supp:eq:jr-estimator}
 \widehat\theta^{\mathrm{jr}}
 =\Pi_{B_2^p(B)}
 \left\{\Pi_{[-C_0h,C_0h]}(\widetilde\theta_f)f\right\}.
\end{equation}

The following theorem bounds the combined cost of estimating the
direction and the scalar coefficient, uniformly over the unknown
sparsity.

\par\medskip
\begin{theorem}[Minimax-rate assertion of Theorem~5.4 of the main paper]\label{supp:thm:rate}
Under \eqref{supp:eq:variance-rect} and \eqref{supp:eq:sparse-envelope}, there are fixed constants $0<c<C<\infty$ such that, for all sufficiently large $n$,
\begin{equation}\label{supp:eq:rate}
 \begin{aligned}
 c\left\{n^{-1/5}+\frac{J_{p_n}(s)}n\right\}
 &\le \cR^{\mathrm{joint}}_{n,p_n,s}(B)\\
 &\le C\left\{n^{-1/5}+\frac{J_{p_n}(s)}n\right\},
 \qquad 1\le s\le s_{+,n}.
 \end{aligned}
\end{equation}
Moreover, the single estimator \eqref{supp:eq:jr-estimator}, which does not use the actual $s,a,$ or $\eta$, satisfies the upper bound uniformly over the specified parameter ranges.
\end{theorem}

\Needspace{10\baselineskip}

\subsubsection{Proof of the adaptive sparse pilot}
\label{supp:app:proof-sparse-pilot}

\begin{proof}[Proof of \Cref{supp:lem:pilot}]
The feasible set in \eqref{supp:eq:pilot} is a finite union of compact coordinate balls.  The criterion is sample measurable and lower semicontinuous, including $u\mapsto\norm u_0$, so a measurable selector exists.

Gaussian integration by parts gives
\[
 \E(RZ)=\E\{(2m(X^\top\theta)-1)X\}
 =2a\E m'(X^\top\theta)\theta,
\]
which is \eqref{supp:eq:mu-direction}.  Since $|R|=1$ and
$Z\sim N\{0,(a+\eta)I_p\}$, centred sparse directional projections of $RZ$ are uniformly sub-Gaussian.  Write
$\norm z_{(k,2)}=\sup_{|A|\le k}\norm{z_A}_2$.
Sparse Euclidean nets and a union bound give, simultaneously for $1\le k\le2s_{+,n}$,
\begin{equation}\label{supp:eq:sparse-concentration}
 \norm{V_0-\mu}_{(k,2)}^2
 \le C(a_++\eta_+)\frac{J_{p_n}(k)+x}{n_0}
\end{equation}
outside an event of probability at most $2e^{-x}$.

Let $\widehat k=\norm{\widehat\mu}_0$.  Comparing \eqref{supp:eq:pilot} at $\widehat\mu$ and $\mu$, whose difference has at most $s+\widehat k$ coordinates, and applying Cauchy--Schwarz and Young's inequality gives
\[
 \frac12\norm{\widehat\mu-\mu}_2^2
 +\Lambda_0(a_++\eta_+)\frac{J_{p_n}(\widehat k)}{n_0}
 \le C(a_++\eta_+)\frac{J_{p_n}(s+\widehat k)+x}{n_0}
 +\Lambda_0(a_++\eta_+)\frac{J_{p_n}(s)}{n_0}.
\]
Since $J_p(s+\widehat k)\le C\{J_p(s)+J_p(\widehat k)\}$, choosing $\Lambda_0$ large enough bounds $\norm{\widehat\mu-\mu}_2^2$ by $C\{J_{p_n}(s)+x\}/n_0$.  Integrating this tail bound gives both inequalities in \eqref{supp:eq:pilot-moments}.
\end{proof}

\subsubsection{Proof of the adaptive minimax rate}
\label{supp:app:proof-joint-boundary}

\begin{proof}[Proof of the minimax rate in Theorem~5.4]
Embedding \Cref{supp:lem:two-point} in the first coordinate gives the $cn^{-1/5}$ lower bound for every $s\ge1$.  Separately fix an interior nuisance point.  A constant-weight Varshamov packing $\omega^1,\ldots,\omega^M\in\{0,1\}^p$ can be chosen so that
\[
 \norm{\omega^j}_0=s,\qquad \log M\ge c_0J_p(s),
 \qquad d_{\rm Ham}(\omega^j,\omega^k)\ge c_1s,
\]
where $d_{\rm Ham}$ is Hamming distance.
Set $\theta_j=\alpha\omega^j$ with
$\alpha^2=c_2J_p(s)/(ns)$.  Uniformly over the specified parameter ranges, these vectors lie in the $B$-ball.  The data-processing inequality applied to the complete-data experiment gives
\begin{equation*}
 \KL(P_{\theta}^{Y,Z}\Vert P_{\theta'}^{Y,Z})
 \le\KL(P_{\theta}^{Y,X,Z}\Vert P_{\theta'}^{Y,X,Z})
 \le\frac a8\norm{\theta-\theta'}_2^2.
\end{equation*}
Fano's lemma gives $cJ_p(s)/n$.  Since $x\vee y\ge(x+y)/2$, combining the two lower bounds proves the lower bound in \eqref{supp:eq:rate}.

For the upper bound, write $t=\norm\theta_2$, $\theta=te$ for $t>0$, and
$\delta=\widehat\mu-\mu$.  At $t=0$, clipping the coefficient estimate gives
$\norm{\widehat\theta^{\mathrm{jr}}}_2\le C_0\norm\delta_2$, so \Cref{supp:lem:pilot} suffices.  Now assume $t>0$ and define
\[
 \cG=\set{\norm\delta_2\le\norm\mu_2/2},
 \qquad \zeta=e^\top f.
\]
On $\cG$,
\begin{equation}\label{supp:eq:zeta-bounds}
 \zeta\ge\frac13,\qquad h\zeta\ge\frac{\kappa_-t}{2}.
\end{equation}
Conditional on the pilot fold, the projected experiment is exactly scalar.  Indeed, define
\begin{equation}\label{supp:eq:projected-params}
 a^\dagger=a\zeta^2,\qquad
 \eta^\dagger=\eta+a(1-\zeta^2),\qquad
 \theta^\dagger=t/\zeta.
\end{equation}
Conditional on the pilot fold,
\[
 f^\top Z=\zeta e^\top X+
 \{(f-\zeta e)^\top X+f^\top W\}.
\]
The two Gaussian terms have zero covariance and variances
$a^\dagger$ and $\eta^\dagger$, so they are independent.  Moreover,
$\theta^\dagger\zeta e^\top X=\theta^\top X$.  On $\cG$ the projected parameters lie in $K^\dagger$ and
$|\theta^\dagger|\le2h/\kappa_-\le C_0h$, so clipping the coefficient estimate cannot increase scalar error.

The direction error satisfies the exact identity
\begin{equation}\label{supp:eq:direction-error}
 \norm{\theta^\dagger f-\theta}_2^2
 =t^2(\zeta^{-2}-1)
 =\frac{t^2\norm{(I_p-ee^\top)\delta}_2^2}{h^2\zeta^2}
 \le\frac4{\kappa_-^2}\norm\delta_2^2.
\end{equation}
Conditionally applying \Cref{supp:cor:scalar-md} and using $(x+y)^2\le2x^2+2y^2$ gives
\begin{equation}\label{supp:eq:good-risk}
 \E\{\norm{\widehat\theta^{\mathrm{jr}}-\theta}_2^2
 \one_{\cG}\mid\cD_0\}
 \le Cn_1^{-1/5}+C\norm\delta_2^2.
\end{equation}
On $\cG^c$,
$t\le2\norm\delta_2/\kappa_-$ and $h\le3\norm\delta_2$, so the clipped estimator satisfies
\begin{equation}\label{supp:eq:bad-risk}
 \norm{\widehat\theta^{\mathrm{jr}}-\theta}_2^2\one_{\cG^c}
 \le C\norm\delta_2^2.
\end{equation}
Projection onto the $B$-ball cannot increase loss.  Taking expectations in \eqref{supp:eq:good-risk}--\eqref{supp:eq:bad-risk} and applying \Cref{supp:lem:pilot} proves the upper bound.
\end{proof}

The rate assertion of Theorem~5.4 identifies separate contributions from
fifth-order identification and sparse direction estimation.  When the
direction contribution is smaller, the remaining question is whether
the scalar exact constant is preserved.
\subsection{Exact minimax constant in Theorem~5.4}
\label{supp:sec:exact-sparse}

The interval $\Lambda(c)$ in \eqref{supp:eq:Lambda-c} has endpoints
\begin{equation}\label{supp:eq:ell-u}
 \ell(c)=\frac c{\overline a(c)},
 \qquad u(c)=\frac c{\underline a(c)}.
\end{equation}
The projected scalar model uses a shrinking enlargement of
$\Lambda(c)$.

The next lemma controls the interval enlargement needed for the
exact-constant assertion of Theorem~5.4.

\begin{lemma}[Stability under enlargement of the parameter interval]\label{supp:lem:enlargement}
Let $\omega_n>0$ satisfy $\omega_n\to0$, with $\omega_n<1/2$, and put
\begin{equation*}
 \Lambda_n^+(c)=
 \left[\ell(c),\frac{u(c)}{1-\omega_n}\right].
\end{equation*}
Let $d_{\rm Haus}$ denote Hausdorff distance between compact subsets of
$\mathbb R$.  The two limits
\begin{align}
 \sup_{c\in\cC}d_{\rm Haus}\{\Lambda_n^+(c),\Lambda(c)\}
 &\longrightarrow0,
 \label{supp:eq:Hausdorff-Lambda}\\
 1536^{1/5}K_5\sup_{c\in\cC}
 \frac{\{\diam\Lambda_n^+(c)\}^{8/5}}c
 &\longrightarrow C_5(\cA,\cE)
 \label{supp:eq:constant-enlargement}
\end{align}
hold.
More precisely, for $B_0>0$ define
\[
 \Xi_n^+(B_0)=\left\{(\vartheta,a,\eta):
 |\vartheta|\le B_0,\ c=a+\eta\in\cC,\ a,\eta>0,\
 \frac ca\in\Lambda_n^+(c)\right\}.
\]
There are measurable scalar rules $\widehat\vartheta_n^+$ such that
\begin{equation}\label{supp:eq:enlarged-scalar-risk}
 \begin{aligned}
 \limsup_{n\to\infty}\Bigg[{}
 &n^{1/5}\sup_{(\vartheta,a,\eta)\in \Xi_n^+(B_0)}
 \E_{\vartheta,a,\eta}(\widehat\vartheta_n^+-\vartheta)^2\\
 &\quad-1536^{1/5}K_5\sup_{c\in\cC}
 \frac{\{\diam\Lambda_n^+(c)\}^{8/5}}c\Bigg]\le0.
 \end{aligned}
\end{equation}
The same conclusion, normalized by $N_n^{1/5}$, holds when the rule uses
independent calibration and main folds of sizes $m_n$ and $N_n$ satisfying
$m_n=o(N_n)$ and
$m_n/(N_n^{4/5}\log N_n)\to\infty$.
\end{lemma}

\Cref{supp:lem:enlargement} shows that the small interval enlargement needed
after projection does not change the limiting scalar constant.  The
following theorem uses this stability to control the additional direction
error in the sparse model.

\begin{theorem}[Exact-constant assertion of Theorem~5.4 of the main paper]\label{supp:thm:exact-sparse}
Suppose, in addition to \eqref{supp:eq:variance-rect}, that
\begin{equation*}
 J_{p_n}(s_{+,n})=o(n^{4/5}),
 \qquad 1\le s_{+,n}\le p_n/2.
\end{equation*}
Then, for every fixed $B>0$,
\begin{equation}\label{supp:eq:exact-sparse-risk}
 \sup_{1\le s\le s_{+,n}}
 \left|n^{1/5}\cR^{\mathrm{joint}}_{n,p_n,s}(B)
       -C_5(\cA,\cE)\right|\longrightarrow0.
\end{equation}
Moreover, one estimator, using $s_{+,n}$ and the fixed parameter bounds but not the actual $s,a,$ or $\eta$, satisfies
\begin{equation}\label{supp:eq:exact-sparse-estimator}
 \limsup_{n\to\infty}
 \sup_{\substack{1\le s\le s_{+,n}\\(a,\eta)\in\cA\times\cE\\
                  \theta\in\Theta_{p_n,s}(B)}}
 n^{1/5}\E_{\theta,a,\eta}
 \norm{\widehat\theta_n^{\mathrm{sharp}}-\theta}_2^2
 \le C_5(\cA,\cE).
\end{equation}
\end{theorem}

\Needspace{10\baselineskip}

\subsubsection{Proof of stability under interval enlargement}
\label{supp:app:proof-enlargement}

\begin{proof}[Proof of \Cref{supp:lem:enlargement}]
The endpoint maps in \eqref{supp:eq:ell-u} are uniformly piecewise Lipschitz and
bounded away from zero and infinity on $\cC$.  Hence
\[
 \sup_{c\in\cC}d_{\rm Haus}\{\Lambda_n^+(c),\Lambda(c)\}
 \le C\omega_n\to0.
\]
Since $\inf_{c\in\cC}\ell(c)>1$, these intervals eventually lie in a fixed
compact subset of $(1,\infty)$.  This proves
\eqref{supp:eq:Hausdorff-Lambda}; uniform continuity of
$z\mapsto z^{8/5}$ on the resulting common compact interval proves
\eqref{supp:eq:constant-enlargement}.

For \eqref{supp:eq:enlarged-scalar-risk}, fix $0<\omega_0<1/2$ and put
\[
 \Lambda_\omega^+(c)
 =\left[\ell(c),\frac{u(c)}{1-\omega}\right],
 \qquad 0\le\omega\le\omega_0.
\]
Write $\Xi_\omega^+(B_0)$ for the parameter space in the statement with
$\omega_n$ replaced by $\omega$.  These spaces lie in the fixed compact
rectangle
\[
 [-B_0,B_0]\times[(1-\omega_0)a_-,a_+]
 \times[\eta_-,\eta_++\omega_0a_+].
\]
On this rectangle, the profile radius, moment bounds, calibration bounds,
and outer inverse bound in \Cref{supp:lem:inverse-five} can be chosen independently
of $\omega$.  The endpoint maps of $\Lambda_\omega^+(c)$ also have uniformly
bounded Lipschitz constants.  In the upper-bound proof of
\Cref{supp:thm:exact-globalization}, include the known index
$\omega\in[0,\omega_0]$ in the compact parameter ranges for the Gaussian-rule transfer and the
diagonal tuning choices.  This gives measurable
rules $\widehat\vartheta_{n,\omega}^+$ satisfying
\[
 \begin{aligned}
 N_n^{1/5}\sup_{\Xi_\omega^+(B_0)}
 \E(\widehat\vartheta_{n,\omega}^+-\vartheta)^2
 &\le1536^{1/5}K_5\sup_{c\in\cC}
 \frac{\{\diam\Lambda_\omega^+(c)\}^{8/5}}c+o(1),
 \end{aligned}
\]
uniformly in $\omega$, for the independent folds specified in the statement.
Taking $\omega=\omega_n$ proves the fold version of
\eqref{supp:eq:enlarged-scalar-risk}.  Taking
$m_n=\lceil n^{9/10}\rceil$ and $N_n=n-m_n$ proves its original version,
since $N_n/n\to1$.
\end{proof}

\subsubsection{Proof of the exact sparse fifth-order constant}
\label{supp:app:proof-exact-sparse}

\begin{proof}[Proof of the exact minimax constant in Theorem~5.4]
\leavevmode\par
\paragraph{Sample splitting and tuning}
Set
\begin{equation*}
 K_n=J_{p_n}(s_{+,n})+3\log n,
 \qquad q_n=\frac{K_n}{n^{4/5}},
 \qquad \alpha_n=\max\{q_n^{1/2},n^{-1/10}\}.
\end{equation*}
Let $m_n=\lceil\alpha_nn\rceil$ and split the data into independent folds
$\cD_0,\cD_1,\cD_2$ of sizes $m_n,m_n,N_n=n-2m_n$.  Since
$q_n,\alpha_n\to0$, also $N_n/n\to1$.  Put
\[
 d_n=\frac{N_n^{1/5}K_n}{m_n},
 \qquad r_n=d_n^{1/4},
 \qquad \omega_n=C\frac{d_n}{r_n^2},
\]
where $C$ is a sufficiently large fixed constant.  Direct calculation gives
\[
 \frac{m_n}{N_n^{4/5}\log N_n}\to\infty,
 \qquad d_n=(1+o(1))\frac{q_n}{\alpha_n}
 \le(1+o(1))\max\{q_n^{1/2},n^{-1/10}\}\to0.
\]
Hence
\begin{equation}\label{supp:eq:rn-rates}
 r_n,\omega_n\to0,\qquad
 \frac{d_n}{r_n^2}=d_n^{1/2}\to0,\qquad
 \frac{d_n^2}{r_n^2}=d_n^{3/2}\to0.
\end{equation}

Use $\cD_0$ for the pilot estimator \eqref{supp:eq:pilot}, with $n_0=m_n$.
Write $\delta=\widehat\mu-\mu$ and use $h,f$ from \eqref{supp:eq:h-f}
when $h>0$.  Conditional on $\cD_0$, use $\cD_1$ for
scalar calibration and $\cD_2$ for the scalar main fold.

On the projected data $(R_i,f^\top Z_i)$, apply the scalar rule from \Cref{supp:lem:enlargement}, using $\Lambda_n^+(\widehat c)$, coefficient radius $2B$, and independent $\cD_1$ calibration.  Let $\widetilde q$ denote the resulting coefficient estimate.  For $C_0>4/\kappa_-$, define
\begin{equation*}
 q=\Pi_{[-C_0h,C_0h]}(\widetilde q),
 \qquad
 \widehat\theta_n^{\mathrm{sharp}}=\Pi_{B_2^{p_n}(B)}(qf),
\end{equation*}
and set the estimator to zero when $h=0$.

\paragraph{Small signals}
Applying \eqref{supp:eq:sparse-concentration} with $x=3\log n$ in the basic
inequality from the proof of \Cref{supp:lem:pilot} gives
\begin{equation}\label{supp:eq:Hn}
 \cH_n=\set{\norm\delta_2^2\le CK_n/m_n},
 \qquad
 \sup_{s,a,\eta,\theta}\Pp(\cH_n^c)\le2n^{-3}.
\end{equation}
If $t=\norm\theta_2\le C_1r_nN_n^{-1/10}$, clipping the coefficient estimate gives on $\cH_n$
\[
 \norm{\widehat\theta_n^{\mathrm{sharp}}-\theta}_2^2
 \le C\{t^2+\norm\delta_2^2\}.
\]
After multiplication by $N_n^{1/5}$, its expectation is
$O(r_n^2+d_n)=o(1)$; the complement contributes $o(N_n^{-1/5})$ because final projection bounds loss by $4B^2$.

\paragraph{Remaining signals and scalar reduction}
Now take $t>C_1r_nN_n^{-1/10}$, write $\theta=te$, and let
$\zeta=e^\top f$.  On $\cH_n$, for sufficiently large $C_1$, $\zeta>0$ and
\begin{equation}\label{supp:eq:zeta-sharp}
 1-\zeta^2\le\zeta^{-2}-1
 \le C\frac{\norm\delta_2^2}{t^2}
 \le C\frac{d_n}{r_n^2}\le\omega_n.
\end{equation}
Indeed, $t>C_1r_nN_n^{-1/10}$ and $\cH_n$ imply the event $\cG$ from
the proof of \Cref{supp:thm:rate}.  Equations
\eqref{supp:eq:zeta-bounds} and \eqref{supp:eq:direction-error} then give the
first two bounds.

By the projection calculation in the proof of \Cref{supp:thm:rate}, conditional
on $\cD_0$ the scalar parameters are those in
\eqref{supp:eq:projected-params}.  In particular,
\begin{equation*}
 q^\dagger=\frac t\zeta,\qquad
 a^\dagger+\eta^\dagger=c,\qquad
 \lambda^\dagger=\frac\lambda{\zeta^2}.
\end{equation*}
By \eqref{supp:eq:zeta-sharp},
\begin{equation*}
 \lambda^\dagger\in\Lambda_n^+(c).
\end{equation*}
Conditional application of \eqref{supp:eq:enlarged-scalar-risk} gives
\begin{equation}\label{supp:eq:qtilde-risk}
 \begin{aligned}
 \limsup_{n\to\infty}
 \sup_{\substack{s\le s_{+,n},\ a,\eta,\theta\\
                  t>C_1r_nN_n^{-1/10}}}
 &N_n^{1/5}\E\{(\widetilde q-q^\dagger)^2
 \one_{\cH_n}\}\\
 &\le C_5(\cA,\cE).
 \end{aligned}
\end{equation}
The projected coefficient obeys
$|q^\dagger|\le B/\sqrt{1-\omega_n}\le2B$ eventually.  Also
$h\zeta\ge\kappa_-t/2$, so $|q^\dagger|\le C_0h$; clipping the coefficient estimate cannot increase scalar error.

\paragraph{Risk of the vector estimator}
For $t>C_1r_nN_n^{-1/10}$, on $\cH_n$ put
$\xi=q-q^\dagger$ and $b=t(\zeta^{-1}-\zeta)$.  Orthogonal decomposition relative to $f$ gives exactly
\begin{equation}\label{supp:eq:orthogonal-decomp}
 \norm{qf-te}_2^2=(\xi+b)^2+t^2(1-\zeta^2).
\end{equation}
Moreover,
\begin{equation*}
 \begin{aligned}
 t^2(1-\zeta^2)&\le C\norm\delta_2^2,\\
 b^2=t^2\frac{(1-\zeta^2)^2}{\zeta^2}
 &\le C\frac{\norm\delta_2^4}{t^2}.
 \end{aligned}
\end{equation*}
By \Cref{supp:lem:pilot}, \eqref{supp:eq:rn-rates}, and $t^{-2}\le CN_n^{1/5}/r_n^2$,
\begin{equation*}
 \begin{aligned}
 N_n^{1/5}\E\{t^2(1-\zeta^2)\one_{\cH_n}\}
 &\le Cd_n=o(1),\\
 N_n^{1/5}\E(b^2\one_{\cH_n})
 &\le C\frac{d_n^2}{r_n^2}=o(1).
 \end{aligned}
\end{equation*}
Cauchy--Schwarz with \eqref{supp:eq:qtilde-risk} makes the cross term
$N_n^{1/5}|\E(\xi b\one_{\cH_n})|=o(1)$.  Thus the vector risk in \eqref{supp:eq:orthogonal-decomp} has the same
asymptotic upper bound as the scalar risk.  Because
$\theta\in B_2^{p_n}(B)$, the final projection cannot increase loss and bounds it
by $4B^2$; hence $\cH_n^c$ contributes $o(N_n^{-1/5})$ by
\eqref{supp:eq:Hn}.  Combining the two signal ranges proves
\eqref{supp:eq:exact-sparse-estimator}.

\paragraph{Lower bound}
For fixed $c$ and $x$, embed the profile path used in the lower-bound part of
\Cref{supp:thm:exact-scalar} as
$\theta=\tau_L(c,xn^{-1/10},\lambda)e_1$.  This path is eventually in every
specified class with $s\ge1$.  The remaining coordinates have
law $N\{0,cI_{p_n-1}\}$, independent of $\lambda$, so they are ancillary for
this fixed-$c$ experiment.  Optimizing over $x$ and $c$ gives the $C_5(\cA,\cE)$ lower bound
uniformly for $s\ge1$.  Together with the upper bound and
$N_n/n\to1$, this proves \eqref{supp:eq:exact-sparse-risk}.
\end{proof}

The exact-constant assertion of Theorem~5.4 shows that, under the stronger
sparsity condition, uncertainty about the signal direction has no effect
on the leading minimax constant.  One estimator attains this constant
without knowing the actual sparsity or the two variances.

\Needspace{10\baselineskip}
\section{Proofs of Theorems~6.1--6.3: the sixth-order network phase}
\label{supp:sec:network}

\subsection{Bowtie expansion and jet claim in Theorem~6.1}\label{supp:sec:net-order}

Let \(Z_1,\ldots,Z_M\) be independent latent block labels with
\begin{equation*}
 \Pp(Z_i=1)=\Pp(Z_i=-1)=\frac q2,
 \qquad \Pp(Z_i=0)=1-q.
\end{equation*}
Given \(Z=(Z_1,\ldots,Z_M)\), the edges are conditionally independent and
\begin{equation}\label{supp:eq:net-edge-model}
 A_{ij}\mid Z\sim\operatorname{Bernoulli}(p+\theta Z_iZ_j),
 \qquad 1\le i<j\le M.
\end{equation}
In the block order \((1,0,-1)\), the probability matrix is
\begin{equation*}
 \begin{pmatrix}
  p+\theta&p&p-\theta\\
  p&p&p\\
  p-\theta&p&p+\theta
 \end{pmatrix}.
\end{equation*}
\begin{equation*}
 \tau_{\rm net}=B_{11}-B_{1,-1}=2\theta
\end{equation*}

Fix
\begin{equation*}
 \mathcal P=[p_-,p_+]\subset(0,1),\qquad
 0<q_-<q_+<1,\qquad
 \Lambda=[\lambda_-,\lambda_+]=[q_+^{-1},q_-^{-1}].
\end{equation*}
Use the profile coordinates
\begin{equation*}
 v=q\theta,\qquad \lambda=q^{-1},\qquad
 q=\lambda^{-1},\qquad \theta=\lambda v,
 \qquad \tau_{\rm net}(p,v,\lambda)=2v\lambda.
\end{equation*}
Choose \(v_{\max}>0\) so that

\[
 \lambda_+v_{\max}<\min(p_-,1-p_+),
\]

and put

\[
 \Xi_{\rm net}=\mathcal P\times[-v_{\max},v_{\max}]\times\Lambda.
\]

Write \(P^M_{p,v,\lambda}\) for the marginal adjacency law.  At \(v=0\),
it is \(P^M_{p,0}\), independently of \(\lambda\).

Let \(\mathcal E_M\) be the edge set of the complete graph.  For
\(F\subseteq\mathcal E_M\), let \(V(F)\) be its nonisolated vertices and
define
\begin{equation*}
 H_{e,p}(A)=\frac{A_e-p}{p(1-p)},\qquad
 H_{F,p}(A)=\prod_{e\in F}H_{e,p}(A).
\end{equation*}
Let \(\mathcal B_M\) be the collection of bowties: two triangles sharing
exactly one vertex.  Then
\begin{equation*}
 |\mathcal B_M|=15\binom M5,\qquad
 B_{M,p}(A)=\sum_{F\in\mathcal B_M}H_{F,p}(A),
\end{equation*}
and set
\begin{equation*}
 I_M(p)=15\binom M5\{p(1-p)\}^{-6}.
\end{equation*}

The bowtie term distinguishes different values of \(\lambda\) at order
six.  The following lemma verifies the likelihood conditions needed for
the jet claim in Theorem~6.1.

\begin{lemma}[Sixth-order motif jet]\label{supp:lem:net-jet}
There exists \(v_0\in(0,v_{\max})\) such that, uniformly over
\(p\in\mathcal P\), \Cref{supp:ass:profile-jet} holds on \(|v|\le v_0\) for
\(\{P^M_{p,v,\lambda}:\lambda\in\Lambda\}\), with
\begin{equation}\label{supp:eq:net-jet-parameters}
 k=6,\qquad I(p)=I_M(p),\qquad \beta(p)=2.
\end{equation}
Moreover, \(k_{\mathrm J}=6\).
\end{lemma}

The proof fixes a reference $\lambda_0$ and abbreviates the score and
likelihood remainder from Section~6.1 of the main paper as
\[
 S_{p,v}(A)=S_{p,v,\lambda_0}(A),\qquad
 r_{p,v,\lambda}(A)=r(A\mid p,v,\lambda;\lambda_0).
\]

\begin{proof}[Proof of the sixth-order jet claim in Theorem~6.1]
Put \(X_i=Z_i/\sqrt q\).  Then
\begin{equation}\label{supp:eq:net-X-moments}
 \E X_i^d=
 \begin{cases}
  0,&d\text{ odd},\\
  \lambda^{d/2-1},&d\ge2\text{ even},
 \end{cases}
 \qquad \E X_i^0=1.
\end{equation}
For one edge, \(\theta Z_iZ_j=vX_iX_j\).  Its exact likelihood ratio
relative to \(\operatorname{Bernoulli}(p)\) is
\[
 1+vX_iX_jH_{ij,p}(A).
\]
Expanding the product over \(\mathcal E_M\) and then integrating the
independent \(X_i\)'s gives
\begin{equation}\label{supp:eq:net-exact-eulerian-expansion}
 L_{p,v,\lambda}(A)
 :=\frac{dP^M_{p,v,\lambda}}{dP^M_{p,0}}(A)
 =\sum_{F\in\mathfrak E_M}
 v^{|F|}\lambda^{|F|-|V(F)|}H_{F,p}(A),
\end{equation}
where \(\mathfrak E_M\) is the collection of edge sets whose nonisolated
vertices all have positive even degree; the empty set contributes one.
Indeed, odd degrees vanish by \eqref{supp:eq:net-X-moments}, and for an Eulerian
edge set
\[
 \prod_{i\in V(F)}\E X_i^{\deg_F(i)}
 =\lambda^{\sum_i\deg_F(i)/2-|V(F)|}
 =\lambda^{|F|-|V(F)|}.
\]

Every nonempty simple Eulerian graph with at most five edges is a cycle, so
\(|F|=|V(F)|\) and its coefficient is independent of \(\lambda\).  At six
edges, the only Eulerian graph with \(|F|-|V(F)|=1\) has degree sequence
\((4,2,2,2,2)\): two triangles sharing their degree-four vertex.  Choosing
the five vertices, the center, and an unordered pairing of the remaining
four gives
\[
 |\mathcal B_M|
 =\binom M5\cdot5\cdot\frac12\binom42
 =15\binom M5.
\]
The other six-edge Eulerian graphs have \(|F|=|V(F)|\), so their
coefficients do not depend on \(\lambda\).  Since \eqref{supp:eq:net-exact-eulerian-expansion} is a finite
polynomial, uniformly over \(p\in\mathcal P\) and
\(\lambda,\lambda_0\in\Lambda\),
\begin{equation}\label{supp:eq:net-density-difference-six}
 L_{p,v,\lambda}-L_{p,v,\lambda_0}
 =(\lambda-\lambda_0)v^6B_{M,p}+R_{p,v,\lambda},
 \qquad
 \|R_{p,v,\cdot}\|_{C^1(\Lambda)}\le C|v|^7
\end{equation}
pointwise on the finite graph space.

Fix the interior reference point
\(\lambda_0=(\lambda_-+\lambda_+)/2\).  Uniformly over \(p,A,\lambda\),
\(L_{p,v,\lambda}(A)\to1\).  Choose \(v_0\in(0,v_{\max})\) sufficiently
small that
\[
 \inf_{\substack{p\in\mathcal P,\ \lambda\in\Lambda,\ A\\
                  |v|\le v_0}}
 L_{p,v,\lambda}(A)\ge\frac12.
\]
Define
\begin{equation*}
 S_{p,v}(A)=\frac{B_{M,p}(A)}{L_{p,v,\lambda_0}(A)},
 \qquad
 r_{p,v,\lambda}(A)=
 \frac{R_{p,v,\lambda}(A)}{L_{p,v,\lambda_0}(A)}.
\end{equation*}
Dividing \eqref{supp:eq:net-density-difference-six} by the reference density
proves \eqref{supp:eq:ratio-jet} with \(k=6\).  Moreover,
\[
 \E_{p,v,\lambda_0}S_{p,v}
 =\E_{p,0}B_{M,p}=0,
\]
and \eqref{supp:eq:net-density-difference-six} gives
\(\E\|r_{p,v}\|_{C^1(\Lambda)}^2=O(v^{14})=o(v^{12})\).

Under \(P^M_{p,0}\), the centered edge products \(H_{F,p}\) are
orthogonal: if \(F\ne F'\), an edge in their symmetric difference
occurs to the first power and has mean zero.  Also
\(\E_{p,0}H_{e,p}^2=\{p(1-p)\}^{-1}\).  Therefore
\begin{equation}\label{supp:eq:net-bowtie-orthogonality}
 \E_{p,0}B_{M,p}^2
 =|\mathcal B_M|\{p(1-p)\}^{-6}=I_M(p).
\end{equation}
Since \(L_{p,v,\lambda_0}\to1\) uniformly,
\[
 \E_{p,v,\lambda_0}S_{p,v}^2
 =\E_{p,0}\frac{B_{M,p}^2}{L_{p,v,\lambda_0}}
 \longrightarrow I_M(p).
\]
The fourth moments are uniformly bounded because the observation space is
finite and \(p\) stays away from zero and one.  Finally,
\(\tau_{\rm net}=2v\lambda\) makes \eqref{supp:eq:target-jet} exact with
\(\beta(p)=2\).

It remains to show minimality.  For fixed distinct
\(\lambda,\lambda_0\), \eqref{supp:eq:net-density-difference-six} gives
\[
 \left\|\frac{dP^M_{p,v,\lambda}}
 {dP^M_{p,v,\lambda_0}}-1\right\|_2=O(|v|^6).
\]
If an order \(j<6\) satisfied \Cref{supp:ass:profile-jet}, division by
\(|v|^j\) would make the left side tend to zero in \(L^2\), whereas the
jet representation and its \(o(v^j)\) remainder would make its norm tend to
\(|\lambda-\lambda_0|\sqrt I>0\).  This contradiction proves
\(k_{\mathrm J}=6\).
\end{proof}

\Cref{supp:lem:net-jet} identifies the first informative likelihood term.
To obtain a matching rate over the full parameter space, we also need an
inverse bound controlling target separation by observable moments.

\subsection{Inverse-order and rate claims in Theorem~6.1}
\label{supp:sec:supp-net-orders}

Fix \(v_0\) as in \Cref{supp:lem:net-jet}.

Set
\begin{equation*}
 \begin{aligned}
  \Phi_{\rm net}(A)
  &=\left(\prod_{e\in F}A_e\right)_{F\subseteq\mathcal E_M}
    \in\mathbb R^{2^{\binom M2}},\\
  \Psi_{\rm net}(p,v,\lambda)
  &=\E_{p,v,\lambda}\Phi_{\rm net}(A).
 \end{aligned}
\end{equation*}
Let
\begin{equation*}
 \mathcal R_n^{\rm net}
 =\inf_{T_n}\sup_{(p,v,\lambda)\in\Xi_{\rm net}}
   \E_{p,v,\lambda}(T_n-2v\lambda)^2.
\end{equation*}

The next theorem supplies this inverse bound and establishes the order
and rate assertions of Theorem~6.1.

\begin{theorem}[Order and rate assertions of Theorem~6.1 of the main paper]\label{supp:thm:net-order}
There is \(C<\infty\) such that, for every
\(\xi,\xi'\in\Xi_{\rm net}\),
\begin{equation*}
 |\tau_{\rm net}(\xi)-\tau_{\rm net}(\xi')|^6
 \le C\|\Psi_{\rm net}(\xi)-\Psi_{\rm net}(\xi')\|_2.
\end{equation*}
Consequently,
\begin{equation*}
 k_{\mathrm J}=k_{\mathrm C}=k_{\mathrm I}=6,
 \qquad \mathcal R_n^{\rm net}\asymp n^{-1/6}.
\end{equation*}
\end{theorem}

\begin{proof}[Proof of the order and rate claims in Theorem~6.1]
\leavevmode\par
Conditional independence and \eqref{supp:eq:net-edge-model} give
\begin{align}
 \E A_{ij}&=p,\label{supp:eq:net-moment-p-full}\\
 \E\prod_{e\in F_{\triangle}}(A_e-p)&=v^3,\notag\\
 \E\prod_{e\in F_{\bowtie}}(A_e-p)&=\lambda v^6.
 \label{supp:eq:net-moment-bowtie-full}
\end{align}
For example, the triangle product is
\(\theta^3\E(Z_1^2Z_2^2Z_3^2)=\theta^3q^3=v^3\); for a bowtie, the shared
vertex occurs four times, so the product is
\(\theta^6q^5=\lambda v^6\).

The moments in \eqref{supp:eq:net-moment-p-full}--
\eqref{supp:eq:net-moment-bowtie-full} are polynomials in \(p\) and the
components of \(\Psi_{\rm net}\).  Thus
\[
 (p,t,b):=(p,v^3,\lambda v^6)
\]
is a Lipschitz function of \(\Psi_{\rm net}\) on the compact parameter
space.

Use primes for the second parameter point.  For real \(v,v'\),
\begin{equation}\label{supp:eq:net-cube-holder-full}
 |v-v'|^3\le4|v^3-v'^3|.
\end{equation}
Since all variables range over fixed compact intervals,
\begin{align*}
 |v'(\lambda-\lambda')|^6
 &\le C|v'|^6|\lambda-\lambda'|\\
 &\le C\{|\lambda v^6-\lambda'v'^6|
           +\lambda_+|v^6-v'^6|\}\\
 &\le C\{|b-b'|+|t-t'|\}.
\end{align*}
Together with \eqref{supp:eq:net-cube-holder-full} and
\(|v\lambda-v'\lambda'|\le\lambda_+|v-v'|
 +|v'||\lambda-\lambda'|\), this proves
\[
 |2v\lambda-2v'\lambda'|^6
 \le C\{|t-t'|+|b-b'|\}
 \le C\|\Psi_{\rm net}(\xi)-\Psi_{\rm net}(\xi')\|_2.
\]
Thus \(k_{\mathrm I}\le6\).  By \Cref{supp:lem:net-jet},
\(k_{\mathrm J}=6\), while \Cref{supp:cor:order-hierarchy} gives
\(6=k_{\mathrm J}\le k_{\mathrm C}\le k_{\mathrm I}\le6\).
The feature is bounded and finite dimensional, so the lower and upper parts
of \Cref{supp:thm:generic-rate} give
\(\mathcal R_n^{\rm net}\asymp n^{-1/6}\).
\end{proof}

\Cref{supp:thm:net-order} shows that the sixth-order likelihood term also
determines the global minimax rate.  The rate does not specify its
constant; the local Gaussian limit gives the exact risk along the most
difficult paths.

\subsection{Limit experiment and local risk in Theorem~6.2}
\label{supp:sec:supp-net-limit}

For fixed \(p\in\mathcal P\) and \(x\ne0\), put
\(v_n=xn^{-1/12}\).  At this scale, \(nv_n^{12}\) stays constant,
and the following theorem determines the limiting information about
\(\lambda\) and the corresponding risk.

\begin{theorem}[Theorem~6.2 of the main paper]\label{supp:thm:net-limit}
The local experiment converges to
\begin{equation*}
 W_x\mid\lambda\sim
 N\left(\lambda,\frac1{I_M(p)x^{12}}\right).
\end{equation*}
Its exact local risk is
\begin{equation*}
 \lim_{n\to\infty}\inf_{T_n}\sup_{\lambda\in\Lambda}
 n^{1/6}\E_{(P^M_{p,v_n,\lambda})^{\otimes n}}
       (T_n-2v_n\lambda)^2
 =4x^2M\left(\frac1{I_M(p)x^{12}},\Lambda\right).
\end{equation*}
\end{theorem}

\begin{proof}[Proof of Theorem~6.2]
Apply \Cref{supp:thm:profile-lan,supp:prop:normal-envelope} with
\eqref{supp:eq:net-jet-parameters}.
\end{proof}

\subsection{Exact global risk in Theorem~6.3}
\label{supp:sec:net-exact}

\Cref{supp:thm:net-limit} determines the exact risk for fixed \(p\) and
\(x\).  Uniform attainment requires estimating \(p\) and \(v\) and
controlling risk outside these paths.  We verify conditions
\textup{(E)} and \textup{(C)} below; condition \textup{(O)} is checked
in the final proof of Theorem~6.3.

Set
\begin{equation}\label{supp:eq:net-efficient-feature}
 \phi_{p,v}(A)=B_{M,p}(A),\qquad
 \Gamma_{\rm net}(p,v,\lambda)
 =\E_{p,v,\lambda}B_{M,p}(A)=I_M(p)\lambda v^6.
\end{equation}
Uniformly,
\begin{equation}\label{supp:eq:net-E-verification}
 \frac{\partial}{\partial\lambda}\Gamma_{\rm net}=I_M(p)v^6,
 \qquad
 \operatorname{Var}_{p,v,\lambda}\{B_{M,p}(A)\}=I_M(p)+O(|v|^3),
\end{equation}
and all moments are bounded, so \textup{(E)} holds.

For calibration, set
\begin{align}
 \overline A(A)
 &=\binom M2^{-1}\sum_{e\in\mathcal E_M}A_e,
 \label{supp:eq:net-edge-density}\\
 T_p(A)
 &=\binom M3^{-1}\sum_{\{i,j,k\}}
 (A_{ij}-p)(A_{jk}-p)(A_{ki}-p).
 \label{supp:eq:net-triangle-feature}
\end{align}
Here \(\E\overline A=p\) and \(\E T_p=v^3\).  On a fold of size \(m_n\), set
\begin{align}
 \widehat p
 &=\Pi_{\mathcal P}\left\{\frac1{m_n}\sum_{r=1}^{m_n}
      \overline A(A^{(r)})\right\},\notag\\
 \widehat t
 &=\Pi_{[-v_0^3,v_0^3]}\left\{\frac1{m_n}\sum_{r=1}^{m_n}
      T_{\widehat p}(A^{(r)})\right\},
 &\widehat v&=\operatorname{sgn}(\widehat t)|\widehat t|^{1/3}.
 \label{supp:eq:net-calibration}
\end{align}

The next lemma bounds the calibration errors and verifies that replacing
\(p,v\) by their estimates does not affect the leading risk.

\begin{lemma}[Network calibration \textup{(C)}]\label{supp:lem:net-C}
Let \(m_n=o(n)\), \(N_n=n-m_n\), and
\begin{equation}\label{supp:eq:net-fold-rate}
 \frac{m_n}{N_n^{1/2}\log N_n}\longrightarrow\infty.
\end{equation}
For any prescribed polynomial tail order, the calibration rule
\eqref{supp:eq:net-calibration} satisfies condition \textup{(C)} of
\Cref{supp:thm:exact-globalization}, with regular coordinate \(c=p\), constant
interval \(\Lambda(p)=\Lambda\) for \(\lambda\), and coordinate used to determine the estimator weights
\(v_\circ(p,v,\lambda)=\Pi_{[-v_0,v_0]}(v)\).
\end{lemma}

\begin{proof}[Proof of \Cref{supp:lem:net-C}]
Write
\[
 d_n=\sqrt{\frac{\log N_n}{m_n}}.
\]
All graph statistics in \eqref{supp:eq:net-edge-density}--
\eqref{supp:eq:net-triangle-feature} are uniformly bounded, and
\(p\mapsto T_p(A)\) is uniformly Lipschitz on \(\mathcal P\).  Hoeffding's
inequality, first at the true \(p\) and then with the plug-in
\(\widehat p\), shows that for any fixed \(A_0>0\), after enlarging a fixed
constant \(C_0\),
\begin{equation}\label{supp:eq:net-calibration-good-full}
 \sup_{(p,v,\lambda)\in\Xi_{\rm net}}
 \Pp_{p,v,\lambda}\left(
 |\widehat p-p|
 +\left|\widehat t-\Pi_{[-v_0^3,v_0^3]}(v^3)\right|>C_0d_n
 \right)\le C_0N_n^{-A_0}.
\end{equation}
Let \(\mathcal H_n(p,v)\) be the event inside the complement in
\eqref{supp:eq:net-calibration-good-full}.  Projection is nonexpansive, and for
real \(s,t\),
\begin{equation*}
 |\operatorname{sgn}(s)|s|^{1/3}
  -\operatorname{sgn}(t)|t|^{1/3}|^3\le4|s-t|.
\end{equation*}
Hence, on \(\mathcal H_n(p,v)\),
\[
 |\widehat p-p|
 +\left|\widehat v-\Pi_{[-v_0,v_0]}(v)\right|
 \le C d_n^{1/3}
 =C\left(\frac{\log N_n}{m_n}\right)^{1/6}
 =:\varepsilon_n.
\]
Choosing \(A_0>1/6\) gives both bad-event probabilities in condition
\textup{(C)}.  Moreover,
\begin{equation*}
 N_n^{1/12}\varepsilon_n
 \le C\left\{N_n^{1/4}
       \sqrt{\frac{\log N_n}{m_n}}\right\}^{1/3}\longrightarrow0
\end{equation*}
by \eqref{supp:eq:net-fold-rate}.

For the bound in \eqref{supp:eq:calibration-orthogonality}, independence of the
main and calibration samples and uniform Lipschitz continuity of
\(p\mapsto B_{M,p}(A)\) give
\begin{equation}\label{supp:eq:net-empirical-plugin-full}
 \left\|
 (\mathbb P_{N_n}-P_{p,v,\lambda})
 \{B_{M,\widehat p}-B_{M,p}\}\one_{\mathcal H_n(p,v)}
 \right\|_2
 \le C\frac{d_n}{\sqrt{N_n}}.
\end{equation}
For the population part, put \(d=p(1-p)\),
\(d'=p'(1-p')\), \(a=d/d'\), and \(b=(p-p')/d'\).  Then
\(H_{e,p'}=aH_{e,p}+b\).  The only Eulerian subsets of a bowtie are the
empty set, its two triangles, and the full bowtie.  Consequently
\begin{equation}\label{supp:eq:net-shifted-bowtie-mean-full}
 \E_{p,v,\lambda}B_{M,p'}
 =|\mathcal B_M|
 \left\{a^6d^{-6}\lambda v^6
       +2a^3b^3d^{-3}v^3+b^6\right\}.
\end{equation}
The first coefficient equals
\(a^6d^{-6}=(d')^{-6}\).  For \(|v|\le v_0\), on \(\mathcal H_n(p,v)\), take
\(p'=\widehat p\) and use \(\widehat v^3=\widehat t\).  Equations
\eqref{supp:eq:net-efficient-feature} and
\eqref{supp:eq:net-shifted-bowtie-mean-full} yield
\begin{equation}\label{supp:eq:net-population-plugin-full}
 \left|
 \E_{p,v,\lambda}B_{M,\widehat p}
 -\Gamma_{\rm net}(\widehat p,\widehat v,\lambda)
 \right|
 \le C\{ |v|^3d_n+d_n^2\}.
\end{equation}
Indeed,
\(|\widehat v^6-v^6|=|\widehat t^2-v^6|
 \le C(|v|^3d_n+d_n^2)\), while the last two terms in
\eqref{supp:eq:net-shifted-bowtie-mean-full} are smaller than the displayed
bound.

Since \(\Lambda\) is independent of \(p\), projection onto
\(\Lambda(\widehat p)\) leaves \(\lambda\) unchanged:
\(\lambda_{\widehat p}=\lambda\).
Combining \eqref{supp:eq:net-empirical-plugin-full} and
\eqref{supp:eq:net-population-plugin-full} gives
\begin{equation*}
 \left\|
 \left[Z_{N_n}(\widehat p,\widehat v,\lambda)
       -Z_{N_n}(p,v,\lambda)\right]
 \one_{\mathcal H_n(p,v)}
 \right\|_2
 \le C\left\{\frac{d_n}{\sqrt{N_n}}+|v|^3d_n+d_n^2\right\}.
\end{equation*}
For every fixed \(h>0\), uniformly over
\(hN_n^{-1/12}\le|v|\le v_0\), \eqref{supp:eq:calibration-orthogonality}
follows from
\begin{align*}
 &N_n^{1/12}|v|^{-5}
 \left\|
 \left[Z_{N_n}(\widehat p,\widehat v,\lambda)
       -Z_{N_n}(p,v,\lambda)\right]
 \one_{\mathcal H_n(p,v)}
 \right\|_2\\
 &\qquad\le C_h\{d_n+N_n^{1/4}d_n+N_n^{1/2}d_n^2\}
 \longrightarrow0.
\end{align*}
Here \eqref{supp:eq:net-fold-rate} gives
\(N_n^{1/2}d_n^2\to0\), and hence also
\(N_n^{1/4}d_n\to0\).  This verifies condition \textup{(C)}.
\end{proof}

\Cref{supp:lem:net-C} permits calibration with a vanishing fraction of the
sample.  The remaining observations supply the bowtie statistic used
for local estimation and inversion.

Write \(L=\diam\Lambda=q_-^{-1}-q_+^{-1}\) and define
\begin{equation}\label{supp:eq:net-C6}
 C_6^{\rm net}
 =4K_6\left\{15\binom M5\right\}^{-1/6}
 L^{5/3}\sup_{p\in\mathcal P}p(1-p).
\end{equation}

The following theorem combines the local lower bound with uniform
attainment and gives the global minimax constant.

\begin{theorem}[Theorem~6.3 of the main paper]\label{supp:thm:net-exact}
For every fixed \(M\ge5\),
\begin{equation}\label{supp:eq:net-exact-limit}
 \lim_{n\to\infty}n^{1/6}\mathcal R_n^{\rm net}=C_6^{\rm net}.
\end{equation}
One measurable estimator attains the upper bound without knowing
\(p,q\), or \(\theta\).
\end{theorem}

\begin{proof}[Proof of Theorem~6.3]
Apply \Cref{supp:thm:exact-globalization} with \(c=p\), parameter space
\(\Xi_{\rm net}\), the chart \(|v|\le v_0\), and
\[
 v_\circ(p,v,\lambda)=\Pi_{[-v_0,v_0]}(v).
\]
The interval \(\Lambda(p)=\Lambda\) is fixed, the target expansion is
exact with \(\beta(p)=2\), and \Cref{supp:lem:net-jet} verifies the uniform
likelihood conditions.

For \textup{(E)}, use \eqref{supp:eq:net-efficient-feature} and
\eqref{supp:eq:net-E-verification}.  The variance under \(P^M_{p,0}\) is
\(I_M(p)\) by \eqref{supp:eq:net-bowtie-orthogonality}, and
\eqref{supp:eq:net-exact-eulerian-expansion} starts at order \(v^3\), giving
the stated uniform variance expansion.  All moments are uniformly bounded
because the graph space is finite and \(\mathcal P\) is a compact subset
of \((0,1)\).  With \(m_n=\lceil n^{3/4}\rceil\),
\eqref{supp:eq:net-fold-rate} and \Cref{supp:lem:net-C} verify \textup{(C)}.

For \textup{(O)}, the exact mean curve satisfies, uniformly for nonzero
\(v\),
\[
 I_{\min}|v|^6
 \le|\frac{\partial}{\partial\lambda}\Gamma_{\rm net}(p,v,\lambda)|
 \le I_{\max}|v|^6,
\]
where \(0<I_{\min}\le I_M(p)\le I_{\max}<\infty\).  It is one-to-one in
\(\lambda\), and
\[
 |\frac{\partial}{\partial\lambda}\tau_{\rm net}|=2|v|,
 \qquad |\frac{\partial}{\partial v}\tau_{\rm net}|=2\lambda,
 \qquad \frac{\partial}{\partial p}\tau_{\rm net}=0.
\]
Let
\[
 \Xi_O=\{(p,v,\lambda)\in\Xi_{\rm net}:|v|\ge v_0/8\},
 \qquad \mathcal H_O=\mathbb R^{2^{\binom M2}},
 \qquad \Phi_O=\Phi_{\rm net}.
\]
This compact set contains \(\{\xi:|v_\circ(\xi)|\ge v_0/4\}\).
The expectation \(\Psi_O=\Psi_{\rm net}|_{\Xi_O}\) determines
\((p,t,b)=(p,v^3,\lambda v^6)\) through the polynomial centered moments
\eqref{supp:eq:net-moment-p-full}--\eqref{supp:eq:net-moment-bowtie-full}.  On
\(\Xi_O\), the inverse maps
\[
 v=\operatorname{sgn}(t)|t|^{1/3},\qquad
 \lambda=b/v^6,\qquad \tau_{\rm net}=2\lambda v
\]
have uniformly bounded derivatives.  Hence the target is Lipschitz in
\(\Psi_O\).  The feature is bounded, so its variance is uniformly bounded;
continuity, compactness, and the finite-dimensional minimum-distance
criterion ensure the existence of the required measurable selector.  This proves
\textup{(O)}.

Apply \Cref{supp:thm:exact-globalization}.  Since
\[
 I_M(p)^{-1/6}
 =\left\{15\binom M5\right\}^{-1/6}p(1-p),
 \qquad \beta(p)^2=4,
\]
and \(2-2/6=5/3\), its constant
\eqref{supp:eq:general-exact-constant} is precisely
\eqref{supp:eq:net-C6}. This proves \eqref{supp:eq:net-exact-limit} for the
minimum-distance outer construction.

We next justify the moment-based outer candidate displayed in Section~6.3
of the main paper. Let \(\widetilde v\) be its calibration coordinate,
\(\widehat\lambda_n^O=\Pi_\Lambda\{
\widehat\Gamma_{\rm net}/[I_M(\widehat p)\widetilde v^6]\}\), and
\(\widehat\tau_n^O=2\widetilde v\widehat\lambda_n^O\).
On \(|v|\ge v_0/4\), Hoeffding's inequality and the bounded derivative
of the real cube root away from zero give, for any fixed \(A>0\),
\[
 \begin{aligned}
 d_n&=\sqrt{\frac{\log N_n}{m_n}},\qquad
 \mathcal F_{O,n}
 =\{|\widehat p-p|+|\widetilde v-v|\le C_A d_n\},\\
 \sup_{\substack{(p,v,\lambda)\in\Xi_{\rm net}\\|v|\ge v_0/4}}
 P_{p,v,\lambda}^{\otimes m_n}(\mathcal F_{O,n}^c)&=O(N_n^{-A}).
 \end{aligned}
\]
For all large \(n\), \(|\widetilde v|\ge v_0/8\) on this event.
The exact mean formula \eqref{supp:eq:net-shifted-bowtie-mean-full} gives
\[
 P_{p,v,\lambda}B_{M,\widehat p}
 =I_M(\widehat p)\lambda v^6+O(|\widehat p-p|^3).
\]
Nonexpansiveness of \(\Pi_\Lambda\) therefore implies, on
\(\mathcal F_{O,n}\),
\[
 |\widehat\tau_n^O-2v\lambda|
 \le C\left\{|\widetilde v-v|+|\widehat p-p|^3
 +|\widehat\Gamma_{\rm net}-P_{p,v,\lambda}B_{M,\widehat p}|\right\}.
\]
Conditional on calibration, the last centered term has variance at most
\(C/N_n\), since \(B_{M,p'}\) is uniformly bounded for
\(p'\in\mathcal P\). Hence
\begin{equation}\label{supp:eq:net-explicit-outer-risk}
 \sup_{\substack{(p,v,\lambda)\in\Xi_{\rm net}\\|v|\ge v_0/4}}
 \E_{p,v,\lambda}
 \{(\widehat\tau_n^O-2v\lambda)^2\one_{\mathcal F_{O,n}}\}
 \le C\left\{N_n^{-1}+\frac{\log N_n}{m_n}\right\}.
\end{equation}
On the calibration events from \Cref{supp:lem:net-C}, a positive outer weight
implies \(|v|\ge v_0/4\) for all large \(n\). With
\(m_n=\lceil n^{3/4}\rceil\), \(A>1/6\), and the final projection onto
\([-B_{\rm net},B_{\rm net}]\), where
\(B_{\rm net}=2q_-^{-1}v_{\max}\), both the bound above and the
exceptional-event contribution are \(o(n^{-1/6})\). Thus this candidate
can replace the minimum-distance outer candidate without changing the
attaining constant. The displayed implementation retains
\(\widetilde v\) in addition to the summaries used by the generic
reduced experiment in Section~4.2 of the main paper.
\end{proof}

\Cref{supp:thm:net-exact} shows that the largest critical-path risk determines
the global constant even when all three original parameters are unknown.

\section{Computing the estimator across four regions}\label{supp:sec:supp-computable}

The calculations below implement the estimator in Section~4.2 and its
model-specific forms in Sections~5--6 of the main paper.

\subsection{Bounds for the bounded-normal minimax risk}
\label{supp:sec:supp-normal-certificate}

For $z>0$, write
\[
 R(\mu,d)=\E_\mu\{d(Y)-\mu\}^2,\qquad
 Y\sim N(\mu,1),\qquad 0\le\mu\le z.
\]
By projection, it suffices to consider $d:\mathbb R\to[0,z]$.  For a finite prior
$p=\{(\mu_j,p_j):1\le j\le s\}$, let $d_p$ be its posterior mean and define
\[
 B_z(p)=\inf_d\sum_{j=1}^s p_jR(\mu_j,d),
 \qquad U_z(p)=\sup_{0\le\mu\le z}R(\mu,d_p).
\]
Then
\begin{equation}\label{supp:eq:supp-exact-normal-bracket}
 B_z(p)\le R_G(z)\le U_z(p).
\end{equation}
For every $d:\mathbb R\to[0,z]$,
\begin{equation}\label{supp:eq:supp-normal-risk-lipschitz}
 |R(\mu,d)-R(\mu',d)|
 \le L_R(z)|\mu-\mu'|,
 \qquad L_R(z)=2z+z^2\sqrt{2/\pi}.
\end{equation}
Indeed, differentiation and $|d(y)-\mu|\le z$ give
\[
 \left|\frac{\partial}{\partial\mu}R(\mu,d)\right|
 \le\int_{\mathbb R}
 \{2z+z^2|y-\mu|\}\varphi(y-\mu)\,dy=L_R(z),
\]
and the mean-value theorem gives \eqref{supp:eq:supp-normal-risk-lipschitz}.
Let $0=\mu_0<\cdots<\mu_G=z$ have maximum spacing $h$, with numerical
evaluations satisfying the rigorous error bounds
\begin{equation}\label{supp:eq:supp-normal-numerical-errors}
 |\widehat B-B_z(p)|\le\varepsilon_B,\qquad
 |\widehat R_g-R(\mu_g,d_p)|\le\varepsilon_U
 \quad(0\le g\le G).
\end{equation}
Define
\begin{equation}\label{supp:eq:supp-certified-normal-bounds}
 B^-=(\widehat B-\varepsilon_B)\vee0,\qquad
 U^+=\max_{0\le g\le G}\widehat R_g
       +\frac12L_R(z)h+\varepsilon_U.
\end{equation}
Equations \eqref{supp:eq:supp-exact-normal-bracket}--\eqref{supp:eq:supp-normal-numerical-errors}
give
\begin{equation*}
 B^-\le R_G(z)\le U^+.
\end{equation*}
Thus $U^+-B^-\le\varepsilon$ establishes $\varepsilon$-minimaxity, provided
the error bounds $\varepsilon_B,\varepsilon_U$ are rigorous.

\subsection{Exchange computation}

For $S\subseteq\{\mu_0,\ldots,\mu_G\}$ and $p\in\Delta(S)$,
\begin{equation*}
 d_p(y)=
 \frac{\sum_{\mu\in S}p_\mu\mu\varphi(y-\mu)}
      {\sum_{\mu\in S}p_\mu\varphi(y-\mu)},\qquad
 \frac{\partial B_z(p)}{\partial p_\mu}=R(\mu,d_p).
\end{equation*}
For $\eta>0$, take an even
$G\ge2\max\{1,\lceil L_R(z)z/(2\eta)\rceil\}$,
$\mu_g=gz/G$, and $S_0=\{0,z/2,z\}$.  At step $t$, let $p_t$ be the
minimum-$\ell_2$ maximizer of $B_z$ on $\Delta(S_t)$, compute
\eqref{supp:eq:supp-certified-normal-bounds}, and stop when
$U^+-B^-\le\eta$.  Otherwise let
\begin{equation*}
 \mu_t=\min\operatorname*{arg\,max}_{\mu\in\{\mu_0,\ldots,\mu_G\}\setminus S_t}
             \widehat R(\mu,d_{p_t}),\qquad
 S_{t+1}=S_t\cup\{\mu_t\}.
\end{equation*}
For exact inner maximization and exact grid risks, the KKT conditions give
\[
 \max_{\mu\in S_t}R(\mu,d_{p_t})=B_z(p_t).
\]
Hence each step before termination adds a new grid point, and the procedure
terminates after at most $G-2$ additions.  In the implementation, choose
the smallest unused grid point among those with the highest risk; if none
remains, tighten the inner tolerances.  Analytic gradients and
Gauss--Hermite integration are used.  Software first maximizes $B_z$, then
minimizes $\lVert p\rVert_2$ subject to
$B_z(p)\ge\widehat B_{\max}-\varepsilon_{\rm opt}$.  It differentiates the
quadrature approximation itself; this gradient converges to $R(\mu,d_p)$.
A candidate prior is retained only when it meets the feasibility and regret
tolerances; a solution returned with a nonzero solver status is retained only
when the KKT residual also meets its tolerance.
With rigorous error bounds $\varepsilon_B,\varepsilon_U$,
$U^+-B^-\le\eta$ bounds the excess risk.  If the grid excess risk
$\max_gR(\mu_g,d_{p_t})-B_z(p_t)$ also has a rigorous bound, the gap is bounded by
\[
 \max_gR(\mu_g,d_{p_t})-B_z(p_t)
 +L_R(z)h/2+2(\varepsilon_B+\varepsilon_U).
\]
Without these error bounds, the reported gap is a numerical diagnostic.

\subsection{Finite-grid approximation of the outer estimator}

The outer minimum-distance estimator is
\begin{equation*}
 \widehat\xi_{O,N_n}^{\mathrm{id}}
 \in\argmin_{\zeta\in\Xi_O}
 \|\widehat\Psi_{O,N_n}-\Psi_O(\zeta)\|_{\mathcal H_O},
 \qquad
 \widehat\Psi_{O,N_n}=\frac1{N_n}\sum_{i=1}^{N_n}\Phi_O(O_i).
\end{equation*}
For a finite $h$-net $\mathcal G_h\subset\Xi_O$, break ties by grid order and set
\[
 \widehat\xi_{O,N_n}^{h}
 \in\argmin_{\zeta\in\mathcal G_h}
 \|\widehat\Psi_{O,N_n}-\Psi_O(\zeta)\|_{\mathcal H_O}.
\]
Set
\[
 \omega_O(h):=\sup_{d(\xi,\xi')\le h}
 \|\Psi_O(\xi)-\Psi_O(\xi')\|_{\mathcal H_O}\longrightarrow0.
\]
The two minimum discrepancies satisfy
\begin{equation*}
 \begin{aligned}
 0\le{}&\|\widehat\Psi_{O,N_n}
            -\Psi_O(\widehat\xi_{O,N_n}^{h})\|_{\mathcal H_O}\\
       &-\|\widehat\Psi_{O,N_n}
            -\Psi_O(\widehat\xi_{O,N_n}^{\mathrm{id}})\|_{\mathcal H_O}
       \le\omega_O(h).
 \end{aligned}
\end{equation*}
Condition~\textup{(O)} gives
\begin{align*}
 |\tau(\widehat\xi_{O,N_n}^{h})-\tau(\xi)|
 &\le A_O\{2\|\widehat\Psi_{O,N_n}-\Psi_O(\xi)\|_{\mathcal H_O}
              +\omega_O(h)\},\\
 \sup_{\xi\in\Xi_O}
 \E_\xi\{\tau(\widehat\xi_{O,N_n}^{h})-\tau(\xi)\}^2
 &\le \frac{8A_O^3}{N_n}+2A_O^2\omega_O(h)^2.
\end{align*}
Thus choosing $h_n$ so that
$N_n^{1/(2k)}\omega_O(h_n)\to0$ makes the contribution of the grid approximation
to the risk bound
$o(N_n^{-1/k})$.

\subsection{Expanded logistic implementation}
\label{supp:sec:supp-logistic-expanded-rule}

\noindent The calibration estimates and feature are those in Section~5.3 of the main paper. Use the calibration and main-fold sizes specified in \Cref{supp:lem:logistic-C}. On the main fold, define the empirical moment, local information, critical-scale coordinate, and Gaussian statistic by
\[
 \begin{aligned}
 \widehat\Gamma_L&=\frac1{N_n}\sum_{i=1}^{N_n}
 \phi_{\widehat c,\widehat v}(R_i,Z_i),\\
 \widehat J_L&=\frac{N_n|\widehat v|^{10}}{1536},
 &\widehat s_L&=N_n^{1/10}|\widehat v|,\\
 \widehat W_L&=
 \begin{cases}
  1536\widehat\Gamma_L/\widehat v^5,&\widehat v\ne0,\\
  0,&\widehat v=0.
 \end{cases}
 \end{aligned}
\]
The following specifies the remaining recovery and combination steps.

On the outer region, fifth-order inversion is no longer needed. Use
\(\Phi_L\), \(\Psi_L\), and \(\mathcal H_L\) from
\Cref{supp:sec:supp-logistic-global-verification}, and set
\(\Xi_{O,L}=\{\xi\in\Xi_B:|v_{\circ,L}(\xi)|\ge v_1/8\}\).
Choose a measurable minimum-distance estimator
\[
 \widehat\xi_{O,L}\in\argmin_{\zeta\in\Xi_{O,L}}
 \left\|\frac1{N_n}\sum_{i=1}^{N_n}\Phi_L(R_i,Z_i)
              -\Psi_L(\zeta)\right\|_{\mathcal H_L}.
\]

It remains to combine the local, inversion, and outer estimators into a single procedure. Let the thresholds and bounded-normal rules be those specified in Section~4.2 of the main paper, with \(k=5\) and \(v_0=v_1\), and define
\[
 u_n=\rho_{r_n,2r_n}(\widehat s_L),\qquad
 g_n=\rho_{M_n,2M_n}(\widehat s_L),\qquad
 o_n=\rho_{3v_1/8,\,7v_1/16}(|\widehat v|).
\]
For \(L(\widehat c)>0\), define the local and inversion estimates of the profile coordinate by
\begin{align*}
 \widehat\lambda_n^L
  =\delta^{(\varepsilon_n^{\rm BN})}
    _{\Lambda(\widehat c),\widehat J_L}(\widehat W_L),
 \qquad\quad
 \widehat\lambda_n^I
  =\Gamma_L(\widehat c,\widehat v,\cdot)^{-1}
   \left[\Pi_{\Gamma_L\{\widehat c,\widehat v,\Lambda(\widehat c)\}}
     \widehat\Gamma_L\right].
\end{align*}
Each quantity is evaluated only when its associated weight is positive; a zero-weight contribution is defined to be zero. The resulting estimator is
\[
\begin{aligned}
 \widehat\theta_n
 =\Pi_{[-B,B]}\Bigg[
 \frac{(1-o_n)u_n(1-g_n)}{\sqrt{\widehat c}}
    \widehat\lambda_n^L U(\widehat v,\widehat\lambda_n^L) +\frac{(1-o_n)u_ng_n}{\sqrt{\widehat c}}
    \widehat\lambda_n^I U(\widehat v,\widehat\lambda_n^I)
  +o_n(\widehat\xi_{O,L})_1\Bigg].
\end{aligned}
\]
The weights provide a smooth transition across the three regimes: the bounded-normal rule is used near the critical scale, moment inversion takes over as identification strengthens, and the outer estimator is used away from the local chart.

When \(L(\widehat c)=0\), the profile coordinate is a singleton and no estimation of \(\lambda\) is required. In this case, we use
\[
 \widehat\theta_n
 =\Pi_{[-B,B]}\left[
   \frac{1-o_n}{\sqrt{\widehat c}}\,
   \ell(\widehat c)U\{\widehat v,\ell(\widehat c)\}
   +o_n(\widehat\xi_{O,L})_1\right].
\]
By Theorem~4.1 of the main paper, together with the model-specific verification of its conditions, this estimator attains the exact minimax constant \(C_5(\cA,\cE)\) established in Theorem~5.3 of the main paper.

\subsection{Computing the motif estimator}
\label{supp:sec:network-computation}

$A^{(1)},\ldots,A^{(n)}$ are independent labelled graphs on a fixed vertex
set.
Use the calibration and main-fold sizes specified in \Cref{supp:lem:net-C}.

On the calibration fold set
\begin{align*}
 \widehat t_{\rm raw}
 &=\frac1{m_n}\sum_{r=1}^{m_n}T_{\widehat p}(A^{(r)}),
 &\widetilde t
 &=\Pi_{[-v_{\max}^3,v_{\max}^3]}(\widehat t_{\rm raw}),\\
 \widetilde v
 &=\operatorname{sgn}(\widetilde t)|\widetilde t|^{1/3},
 &\widehat v
 &=\Pi_{[-v_0,v_0]}(\widetilde v).
\end{align*}
$\widehat v$ is projected onto the local parameter interval; retain
$\widetilde v$ for the moment-based outer estimator. This additional
calibration coordinate is not part of the generic reduced statistic
$T_n$ in Section~4.2 of the main paper. The outer risk bound
\eqref{supp:eq:net-explicit-outer-risk} verifies that this implementation
attains the same constant.
On the main fold set
\begin{equation*}
 \widehat\Gamma_{\rm net}
 =\frac1{N_n}\sum_{r=1}^{N_n}B_{M,\widehat p}(A^{(r)}),
 \qquad
 \widehat J_{\rm net}
 =N_n I_M(\widehat p)|\widehat v|^{12},
 \qquad
 \widehat s_{\rm net}=N_n^{1/12}|\widehat v|.
\end{equation*}
Set
\begin{equation*}
 \widehat\tau_n^0=0.
\end{equation*}
When the local weight is positive, hence $\widehat v\ne0$, set
\begin{equation*}
 \widehat W_{\rm net}
 =\frac{\widehat\Gamma_{\rm net}}
        {I_M(\widehat p)\widehat v^6},
 \qquad
 \widehat\lambda_n^L
 =\delta^{(\varepsilon_n^{\rm BN})}_{\Lambda,\widehat J_{\rm net}}(\widehat W_{\rm net}),
 \qquad
 \widehat\tau_n^L=2\widehat v\widehat\lambda_n^L.
\end{equation*}
When the inverse weight is positive, compute
\begin{equation*}
 \widehat\lambda_n^I
 =\Pi_\Lambda\left\{
   \frac{\widehat\Gamma_{\rm net}}
        {I_M(\widehat p)\widehat v^6}\right\},
 \qquad
 \widehat\tau_n^I=2\widehat v\widehat\lambda_n^I.
\end{equation*}
When the outer weight is positive, compute
\begin{equation*}
 \widehat\lambda_n^O
 =\Pi_\Lambda\left\{
   \frac{\widehat\Gamma_{\rm net}}
        {I_M(\widehat p)\widetilde v^6}\right\},
 \qquad
 \widehat\tau_n^O=2\widetilde v\widehat\lambda_n^O.
\end{equation*}
Let $\rho_{a,b}$ be the linear ramp from zero to one and set
\[
 u_n=\rho_{r_n,2r_n}(\widehat s_{\rm net}),\qquad
 g_n=\rho_{M_n,2M_n}(\widehat s_{\rm net}),\qquad
 o_n=\rho_{3v_0/8,\,7v_0/16}(|\widehat v|).
\]
The resulting estimator is
\begin{equation*}
 \widehat\tau_n^{\rm net}
 =\Pi_{[-B_{\rm net},B_{\rm net}]}
 \left[(1-o_n)\{(1-u_n)\widehat\tau_n^0
 +u_n((1-g_n)\widehat\tau_n^L+g_n\widehat\tau_n^I)\}
 +o_n\widehat\tau_n^O\right],
\end{equation*}
where \(B_{\rm net}=2q_-^{-1}v_{\max}\). The sequences \(r_n,M_n\)
and the ramps are those of Section~4.2 of the main paper.

Precompute the $15\binom M5$ bowties.  The total computational cost is
$O(N_nM^5+m_nM^3)$.

\section{Numerical details}
\label{supp:sec:numerics}

We give the protocols and additional results for Section~7 of the main paper.

\subsection{Finite-sample recovery at moderate sample sizes}
\label{supp:sec:supp-moderate-n}

The main-paper experiments use fixed profile coordinates, rather than the
least-favourable shrinking paths used in the large-sample checks below.
The total sample sizes are $n=10^3,10^4,10^5,10^6$. At each combination
of signal, sample size and endpoint $\lambda\in\{3/2,3\}$, we generate
200 independent split samples. There are 3200 replications per model.
For logistic regression, $c=3$, $\mathcal A=\mathcal E=[1,2]$,
$B=4$, $v\in\{0.25,0.65\}$, and $m=\lceil n^{0.9}\rceil$.
The exact target is $\lambda U(v,\lambda)/\sqrt3$.
For networks, $M=5$, $p=1/2$, $\mathcal P=[0.49,0.51]$,
$\Lambda=[3/2,3]$, $v\in\{0.10,0.14\}$,
$m=\lceil n^{0.75}\rceil$, and the target is $2v\lambda$.
In both cases the main fold has size $N=n-m$.
The sampling laws are specified below. Calibration, clipping, routing
and solver settings for this fixed-coordinate benchmark are given in
Table~\ref{supp:tab:supp-implementation-settings}. These are finite numerical implementations
of the constructions; the fixed tuning parameters and numerical search
do not themselves certify the asymptotic uniform risk bound.

\begin{table}[t]
\centering
\small
\renewcommand{\arraystretch}{1.12}
\begin{tabular}{@{}p{0.19\textwidth}p{0.75\textwidth}@{}}
\toprule
Component & Setting\\
\midrule
Tuning thresholds
& Logistic $(r,M,v_0,v_{\rm clip})=(0.5,5,0.5,0.45)$, with target
  projection $[-4,4]$; network
  $(r,M,v_0,v_{\max})=(0.08,0.40,0.155,0.16)$, with $\widehat v$
  restricted by projection to absolute values at most $v_0$.\\
Bounded-normal action
& Observed standardized interval width, without binning;
  standardized numerical gap tolerance $10^{-3}$;
  minimum parameter and validation grids $101$ and $2001$; Gauss--Hermite
  order $120$; optimization tolerance $10^{-12}$;
  solver parameter \texttt{max\_iterations}=500;
  minimum-$\ell_2$ tie break.\\
Profile and inverse
& Logistic bracketed solves use Gauss--Hermite order $256$; network uses
  exact finite graph sums and the clipped bowtie-moment ratio.\\
Outer candidate
& Logistic Gaussian-kernel bandwidth $1$; $16$ Fourier frequencies
  ($67$ features);
  $33$ values of $\theta$ from $-4$ to $4$ in steps of $0.25$;
  $c\in\{2,2.5,3,3.5,4\}$; $9$ values of $\lambda$ per $c$; Gauss--Hermite order
  $160$; $928$ admissible candidates.  Network uses the closed-form clipped
  bowtie-moment ratio with the unclipped calibration coordinate $\widetilde v$.\\
Monte Carlo uncertainty
& Logistic uses $5000$ within-cell bootstrap repetitions for endpoint maxima
  and the tail slope.  Network uses $4000$ for the tail slope; its pointwise
  MSE bars use $t_{0.975,199}=1.9719565$.\\
\bottomrule
\end{tabular}
\caption{Reference settings for the unchanged fixed-coordinate simulations
and archived candidate-weight checks.  The asymptotic sequences are specified separately
in the main paper.}
\label{supp:tab:supp-implementation-settings}
\end{table}

\paragraph{Logistic outer minimization}
The original recovery grid is retained, but its resolution is insufficient
for some of these finite-sample settings. We refine the same 67-dimensional
feature criterion using the identity
\[
 \E[R\exp(iwZ)]
 =\exp(-\eta w^2/2)\E[\tanh(\theta X/2)\exp(iwX)],
 \qquad X\sim N(0,a),
\]
where $R=2Y-1$ and $Z=X+\epsilon$ with independent
$\epsilon\sim N(0,\eta)$. The expectation and its analytic gradient are
evaluated by 256-node Gauss--Hermite quadrature. The three grid points
with smallest criterion values, in the original stable grid order, are
used as starting points for L-BFGS-B in $(\theta,a,\eta)$ over
$[-4,4]\times[1,2]^2$. Its settings are
\texttt{ftol=1e-17}, \texttt{gtol=1e-11},
\texttt{maxiter=300}, and \texttt{maxls=50}.
The selector is restricted to $|v|\ge v_0/8=0.0625$.
An optimizer output with smaller $|v|$ is replaced by the two boundary
candidates with the same $(a,\eta)$ and
$\theta=\lambda U(\pm0.0625,\lambda)/\sqrt c$,
where $c=a+\eta$ and $\lambda=c/a$.
All admissible refined candidates are compared with the best original
grid candidate; ties use lexicographic order in $(\theta,c,\lambda)$.
Thus the chosen criterion never exceeds the best original grid value.
This deterministic refinement is not a certificate of global minimization.
Aggregate optimizer termination counts are released with the output.
No unknown generating parameter is passed to this selector or the routing rule.

\paragraph{Random streams and risk summaries}
The master seed is $20261016$. In each list, signal, size, endpoint and
replication indices start at zero and follow the orders given above.
The logistic and network streams initialize NumPy \texttt{SeedSequence}
with the respective tuples
\[
\begin{aligned}
 &[20261016,11,100+10i_v+i_n,i_\lambda,i_b,f],\\
 &[20261016,2,i_v,i_n,i_\lambda,i_b].
\end{aligned}
\]
Here $f=0,1$ denotes the logistic calibration and main fold, respectively.
The reported risk is the larger of the two endpoint sample means of
squared error. For each model--signal--size combination, we independently
resample 200 losses within each endpoint cell 5000 times, recompute
both means and their maximum, and take its 2.5th and 97.5th percentiles.
The bootstrap stream is
\texttt{SeedSequence}$([20261016,71,i_{\rm model},i_v,i_n])$,
with logistic preceding network. No asymptotic slope is fitted to these
fixed-coordinate experiments.

\paragraph{Validation and scope}
For logistic regression, comparisons at quadrature orders 128, 256 and
512 give weighted conditional-mean discrepancies below
$8.4\times10^{-16}$ and maximum feature discrepancies below
$7.3\times10^{-14}$ at the four generating settings.
A separate check uses direct independent draws
$X\sim N(0,a)$, $\epsilon\sim N(0,\eta)$,
$Z=X+\epsilon$, and $Y\mid X\sim\operatorname{Bernoulli}
\{\operatorname{expit}(\theta X)\}$, with 50 replications per cell
at $n=10^3,10^4$. It uses the same estimator, with stream
\texttt{SeedSequence}$([20261016,31,i_v,i_n,i_\lambda,i_b])$.
An internal-interval check at $\lambda=2.25$ is also released.
In particular, its network MSE at $v=0.14$, $n=10^6$ is $0.0211$,
compared with endpoint maximum $0.0166$ (100 and 200 replications per
cell, respectively). Finite-sample endpoint risks are therefore not
interpreted as estimates of the full-interval supremum.
The programs, environment, complete losses, weights and summaries are in
\texttt{moderate\_sample/} of Supplement B. The accompanying README gives
commands for rerunning the simulations and regenerating the figure.

\subsection{Critical-path checks with moment calibration}
\label{supp:sec:supp-estimator-monte-carlo}

We keep the original models, parameter ranges, targets and critical-path
coefficients. The numerical changes reuse observations for calibration
and feature estimation, and correct the endpoint feature means of the
logistic local rule. No generating coordinate is supplied to the decision.
The fixed-coordinate experiments above retain their original split-sample
implementation.

\subsubsection{Models, paths and data generation}
For logistic regression, $c=3$, $\mathcal A=\mathcal E=[1,2]$, $B=4$,
$\Lambda=[3/2,3]$, $I=1/1536$ and
$\tau=\lambda U(v,\lambda)/\sqrt3$.
For networks, $M=5$, $p=1/2$, $\mathcal P=[.49,.51]$,
$\Lambda=[3/2,3]$, $v_{\max}=.16$, $I_M(1/2)=61440$ and $\tau=2v\lambda$.
The following approximations identify the constants; Supplement B records
the full-precision binary64 inputs:
\[
\begin{aligned}
x_{5,\star}&\approx1.907547,&C_5&\approx0.550691,\\
x_{6,\star}&\approx0.364588,&C_6^{\rm net}&\approx0.250545.
\end{aligned}
\]
The corresponding standardized interval widths are
$z_{5,\star}\approx0.966657$ and $z_{6,\star}\approx0.873236$.
To preserve the original paths define
\[
N_{\rm ref}=n-m_{\rm ref},\qquad
v_n=x_{k,\star}N_{\rm ref}^{-1/(2k)},\qquad
m_{\rm ref}=
\begin{cases}
\lceil\exp(0.9\log n)\rceil,& k=5,\\
\lceil n^{3/4}\rceil,& k=6.
\end{cases}
\]
$N_{\rm ref}$ indexes data generation only; the revised rule uses all $n$
observations. Both $n$ and $N_{\rm ref}$ normalizations are released.
Logistic uses $n=10^9,\ldots,10^{13}$ and seven equally spaced
$\lambda$ values from $3/2$ to $3$, with $5000$ replications per cell.
Additional $10^8$ and $10^{14}$ runs are sensitivity checks.
Network uses $n=21366,30000,10^5,10^6,10^7,10^8$ and thirteen equally
spaced $\lambda$ values. Each endpoint has $5000$ replications and each
interior point $2000$, except at $21366$, where the counts are $10000$
and $3000$. The smallest displayed network size is the smallest integer
for which this reference path satisfies $|v_n|\le.16$; it is not a
statistical sample-size lower bound.

The graph PMF is computed on all $1024$ labelled graphs and checked
against direct enumeration of the $243$ latent-label configurations.
For paired comparisons, independent histograms of sizes
$m_{\rm ref},N_{\rm ref}$ provide paired
comparisons with the original implementation; the revised rule uses their
sum. The primary logistic runs directly draw one all-$n$ histogram.
Logistic uses normalized $128$-node Gauss--Hermite quadrature for
$T=Z/\sqrt3$ and for the conditional response mean. The $256$ categories
are ordered by increasing node, with positive responses first, and their
probabilities are proportional to $w_j\{1\pm g(t_j\mid v,\lambda)\}/2$.
Counts are sampled from this discretized observed law, not from the
Gaussian limit experiment. Profile roots use $256$ nodes and Brent's
method with \texttt{xtol=1e-14}, \texttt{rtol=4*eps},
\texttt{maxiter=200}, initial upper endpoint
$\max(.02,1.05|v|)$ and successive expansion by $1.2$ up to $32$.
The original continuous-data and quadrature checks are retained.

\subsubsection{Numerical decisions and their theoretical relation}
Both models use their original calibration moments and local feature,
computed on all $n$ observations. Logistic projects $\widehat c$ onto
$[2,4]$ and $\widehat v$ onto $[-.5,.5]$; its outer transition is
$[.4,.4375]$, and its inverse and outer candidates are unchanged.
In particular, the outer search domain remains $|v|\ge.0625$.
The local rule uses the exact endpoint feature means and targets.
Networks project $\widehat p$ onto $[.49,.51]$ and $\widehat v$ onto
$[-.16,.16]$, and use only zero, local and inverse branches.
Their inverse uses the exact identity $\Gamma_{\rm net}=I_M(p)\lambda v^6$.
The numerical thresholds are
$r_n=.5\{\log(n+e)\}^{-1/4}$,
$M_n=5\{\log(n+e)\}^{1/8}$ for logistic, and
$r_n=.08\{\log(n+e)\}^{-1/4}$,
$M_n=\{\log(n+e)\}^{1/8}$ for networks.
The target projections remain $[-4,4]$ and $[-.96,.96]$.
All primary replications have local weight one.
For observed standardized widths below $2$ in logistic and at most $1.8$
in networks, the equal-endpoint normal rule is used; these widths lie in
the short-interval range of \cite{supp:casellaStrawderman1981}.
For larger widths the reference exchange solver is called, with the
settings in Table~\ref{supp:tab:supp-implementation-settings}.
No correction is chosen by fitting risks to $C_k$.

\paragraph{Joint calibration bounds.}
The general theorem retains its split-sample condition (C). For the
same-sample variation, apply the proof in
Section~\ref{supp:sec:supp-uniform-attainment} under $P_\xi^{\otimes n}$,
using $N_n=n$ and the following probability and $L^2$ bounds.
Write $\mathbb P_n$ for the empirical measure. In logistic regression put
$\epsilon_n=C\sqrt{\log n/n}$. Concentration gives events $H_n$ of
arbitrarily high polynomial probability on which
$|\widehat c-c|+|\widehat v-v|\le\epsilon_n$ on the chart;
globally use the clipped coordinate $v_{\circ,L}$.
For $T=Z/\sqrt c$ and $s=(c/\widehat c)^{1/2}$, the identity
$H_3(sT)=s^3H_3(T)+3s(s^2-1)T$ gives
\[
\begin{aligned}
&(\mathbb P_n-P)(\phi_{\widehat c,\widehat v}-\phi_{c,v})\one_{H_n}\\
&\quad=\frac{\one_{H_n}}{96}\bigl[(s^3-1)(\mathbb P_n-P)RH_3(T)
+3s(s^2-1)(\mathbb P_n-P)RT\bigr].
\end{aligned}
\]
Its $L^2$ norm is at most $C\epsilon_n/\sqrt n$, without independence,
because its coefficients are deterministically bounded on $H_n$.
The profile bounds give the population substitution error
$C(|v|\epsilon_n+\epsilon_n^2)$. Hence the weighted residual in the
attainment proof is bounded, for each $h>0$, by
\[
C_h\{\epsilon_n+n^{2/5}\epsilon_n+n^{1/2}\epsilon_n^2\}=o(1),
\qquad n^{1/10}\epsilon_n=o(1).
\]
For networks put $d_n=C\sqrt{\log n/n}$. Uniform concentration gives
$|\widehat p-p|+|\widehat t-v^3|\le d_n$ on $H_n$ and
$|\widehat v-v|\le Cd_n^{1/3}$ after projection onto the full domain.
The finite graph space implies
\[
\left\|\sup_{p'\in\mathcal P}
|(\mathbb P_n-P)\partial_{p'}B_{M,p'}|\right\|_2
\le Cn^{-1/2}.
\]
Indeed, the supremum is bounded by a fixed constant times the sum of
the absolute empirical graph-probability errors. The mean-value theorem
and shifted-bowtie identity bound the empirical and population
substitution errors by $Cd_n/\sqrt n$ and $C(|v|^3d_n+d_n^2)$.
Their weighted residual is at most
\[
C_h\{d_n+n^{1/4}d_n+n^{1/2}d_n^2\}=o(1),\qquad
n^{1/12}d_n^{1/3}=o(1).
\]
Choose $P(H_n^c)=O(n^{-A})$ with $A>1/k$; final target projection makes
its normalized risk negligible. The local and inverse parts of the
original risk proof therefore remain valid for joint calibration.

\paragraph{Finite-profile correction and routing.}
For $\ell=\ell(c)$, $u=u(c)$ and $L=u-\ell>0$, define
\[
\begin{aligned}
D(c,v)&=\frac{\Gamma(c,v,u)-\Gamma(c,v,\ell)}{L},\\
\widetilde W_n&=\ell(\widehat c)+
\frac{\widehat\Gamma_n-\Gamma\{\widehat c,\widehat v,\ell(\widehat c)\}}
{D(\widehat c,\widehat v)},\qquad
\widetilde J_n=\frac{nD(\widehat c,\widehat v)^2}{I(\widehat c)}.
\end{aligned}
\]
Use the singleton convention when $L=0$ and the derivative extension
for $D$. Retain its sign for negative $v$.
The local normal decision is applied to $(\widetilde W_n,\widetilde J_n)$.
On the reported logistic paths its target action is the interpolation
of the exact endpoint targets; the full implementation uses the
normal rule and profile inversion at larger widths.
Uniformly on every fixed critical annulus, the expansions
\[
D(c,v)=I(c)v^5\{1+O(v^2)\},\qquad
\Gamma(c,v,\ell)-\ell D(c,v)=O(|v|^7)
\]
give $\widetilde W_n-\widehat W_n=o_{L^2}(1)$ and
$\widetilde J_n/\widehat J_n=1+o(1)$ on $H_n$.
The interpolated and exact targets differ by $O(|v|^3)$, so their
$n^{1/10}$-scaled difference vanishes. The continuous normal rules and
the sufficiently slow diagonal selection in the original proof preserve
the leading constant for this variation.

The shifted logistic gate stays within its chosen profile chart. With
$H\sim N(0,\lambda)$,
\[
V_\lambda(u)=\frac2\lambda\E\{H\tanh(uH/2)\},\qquad
V_\lambda'(u)=\lambda^{-1}\E\{H^2\sech^2(uH/2)\}>0.
\]
Thus $U$ is the unique inverse on the attainable range and the chart
covers $|v_{\rm raw}|<v_0$. For an outer gate
$[4v_0/5,7v_0/8]$, positive inverse weight implies
$|v|\le15v_0/16<v_0$ on $H_n$ for large $n$; positive outer weight
implies $|v|\ge3v_0/4>v_0/4$. The same within-chart and outer bounds apply.
The explicit numerical radius and thresholds do not certify the
analytic constants in those bounds.
For networks, the exact moment and target formulas hold throughout the
legal parameter domain, so projected inversion replaces the outer branch.
Use weights $1-u_n,u_n(1-g_n),u_ng_n$ as in the general construction.
The same risk bounds for zero, local and inverse branches apply, with
the outer term omitted. The likelihood approximation is still needed
only near the null. These variations admit the sufficiently slow tuning
sequences of the proof; the specific finite numerical schedules are not
certificates of uniform sharp risk at a stated sample size.

\subsubsection{Risk, uncertainty and comparison with the original runs}
For the declared grid $G$ and replication count $B_{n,\lambda}$, report
\begin{equation}\label{supp:eq:supp-empirical-endpoint-risk}
\widehat R_{n,G}
=\max_{\lambda\in G}\frac1{B_{n,\lambda}}
\sum_{b=1}^{B_{n,\lambda}}
\{\widehat\tau_n^{(b)}-\tau(c,v_n,\lambda)\}^2.
\end{equation}
This is a finite-grid maximum at the stated regular coordinate, not the
supremum over the full parameter space.
Logistic intervals use $5000$ within-cell bootstrap resamples, recomputing
the maximum across all seven cells. Network intervals bound the maximum
of the thirteen cell means by approximate Bonferroni $t$ intervals at
each fixed $n$, without a simultaneous claim across sample sizes.
The slope fits use the last four displayed sample sizes; network slope
intervals use $2000$ within-cell bootstrap resamples.
These intervals describe Monte Carlo error and exclude quadrature error.

\begin{table}[t]
\centering\small
\begin{tabular}{@{}rrrr@{}}
\toprule
$n$ & $v_n$ & $n^{1/k}\widehat R_{n,G}/C_k$ & $95\%$ MC interval\\
\midrule
\multicolumn{4}{c}{Logistic, $k=5$}\\
$10^{9}$ & 0.243399 & 1.2262 & $[1.2103,1.2562]$\\
$10^{10}$ & 0.192775 & 1.1343 & $[1.1206,1.1632]$\\
$10^{11}$ & 0.152781 & 1.0838 & $[1.0647,1.1107]$\\
$10^{12}$ & 0.121145 & 1.0638 & $[1.0433,1.0892]$\\
$10^{13}$ & 0.096097 & 1.0412 & $[1.0275,1.0687]$\\
\addlinespace
\multicolumn{4}{c}{Network, $k=6$}\\
$21{,}366$ & 0.159999 & 1.0261 & $[1.0020,1.0502]$\\
$30{,}000$ & 0.155442 & 1.0433 & $[1.0075,1.0791]$\\
$10^{5}$ & 0.140356 & 1.0040 & $[0.9696,1.0385]$\\
$10^{6}$ & 0.115602 & 1.0019 & $[0.9685,1.0353]$\\
$10^{7}$ & 0.095306 & 0.9934 & $[0.9606,1.0262]$\\
$10^{8}$ & 0.078614 & 0.9993 & $[0.9664,1.0322]$\\
\bottomrule
\end{tabular}
\caption{Critical-path results after moment calibration. Logistic uses
seven $\lambda$ values and networks thirteen. The interval constructions
are specified in the text; neither is a confidence statement about the
full parameter-space minimax risk. Exact inputs and all cell risks are
included in Supplement B.}
\label{supp:tab:supp-primary-all-sizes}
\end{table}

\begin{figure}[t]
\centering
\includegraphics[page=1,width=\textwidth]{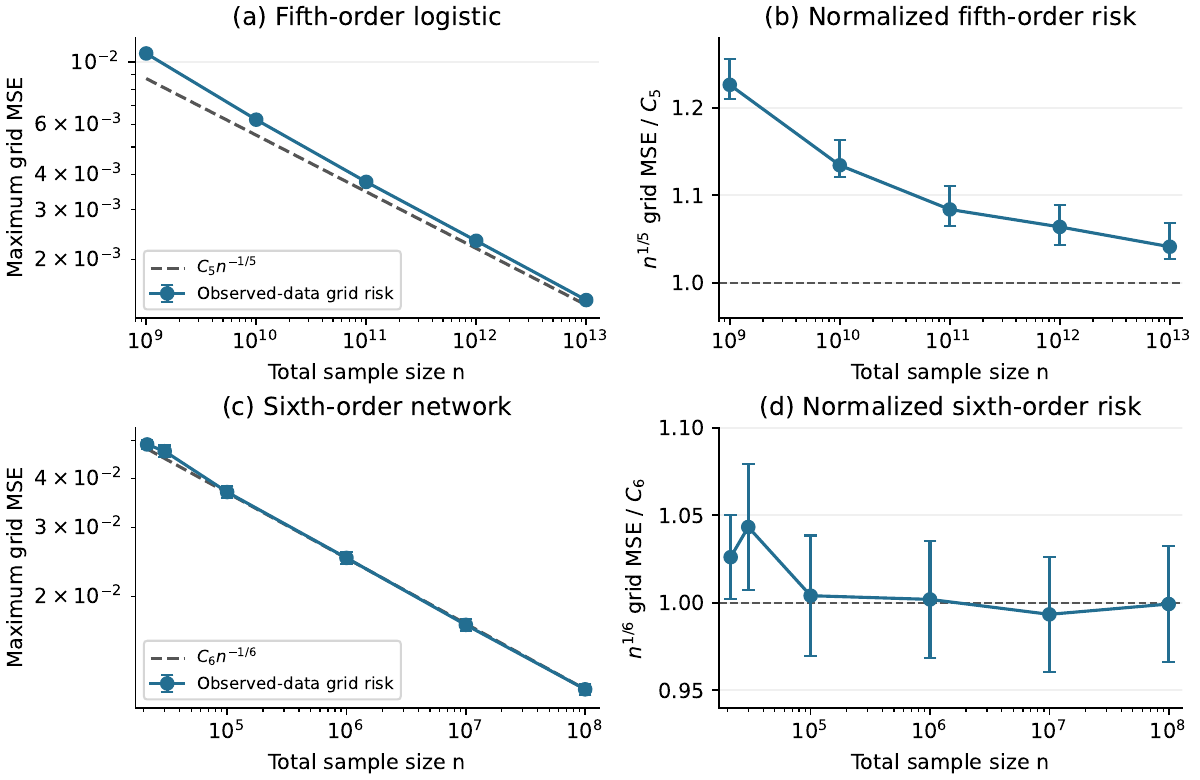}
\caption{Critical-path checks with joint moment calibration. Left panels
show grid MSE against $C_kn^{-1/k}$; right panels show normalized risk.
The original model, path coefficient and parameter ranges are retained.
Both axes and normalization use total $n$, although generation preserves
the original reference-path index $N_{\rm ref}$. Intervals reflect Monte
Carlo uncertainty, and finite-profile bias remains visible for logistic
regression. These curves are separate from the fixed-coordinate study
in the main paper.}
\label{supp:fig:supp-large-critical}
\end{figure}

The last-four-point log--log slopes against total $n$ are
$-0.211965$ $[-0.215714,-0.208244]$
for logistic regression and $-0.167656$
$[-0.171618,-0.163714]$ for networks,
compared with $-1/5$ and $-1/6$. The logistic interval excludes its
limiting exponent: the smaller-sample curve still has detectable
finite-profile bias. These are Monte Carlo intervals for the fitted
finite-sample curve, not bounds on its asymptotic approximation error.
Using the original reference index $N_{\rm ref}$ instead gives slopes
$-0.210328$ and $-0.166516$; both conventions
are retained in Supplement B, while the table and figure use total $n$.

The original checks used maximum sizes $10^{17}$ and $10^{15}$, with
$200$ replications at each endpoint. The revised maxima are $10^{13}$ and
$10^8$, with denser grids and more replications. Their final normalized
risk ratios are $1.0412$ and $0.9993$. These compare agreement with the
limit, not equality of unnormalized risk across sample sizes.
Supplement B also gives independent same-data comparisons with the
original implementation; no pointwise domination is claimed.

\subsubsection{Finite-sample likelihood comparison}
\label{supp:sec:supp-finite-likelihood}
For logistic regression we also compare the Hermite-feature rule above
with a decision based on the finite observed likelihood, on the same
reference paths. Compute $\widehat c,\widehat v$ from all $n$ observations
as in \eqref{supp:eq:calibration-c-v}, with projections onto $[2,4]$ and
$[-.7,.7]$, respectively. Let $\widehat\lambda_0,\widehat\lambda_1$ be
the endpoints of $\Lambda(\widehat c)$,
$\widehat\theta_j=\tau(\widehat c,\widehat v,\widehat\lambda_j)$ and
$\widehat T_i=Z_i/\sqrt{\widehat c}$. Using the exact conditional response
mean $g$, set
\[
\begin{aligned}
\ell_n&=\sum_{i=1}^n\log
\frac{1+R_i g(\widehat T_i\mid\widehat v,\widehat\lambda_1)}
     {1+R_i g(\widehat T_i\mid\widehat v,\widehat\lambda_0)},\\
\widehat\tau_n^{\rm lik}
&=\widehat\theta_0+(\widehat\theta_1-\widehat\theta_0)
\operatorname{expit}(\ell_n).
\end{aligned}
\]
The equal endpoint weights are fixed before evaluation. The decision
uses observed data only; likelihood increments are accumulated with
\texttt{log1p} and extended precision. It is a finite-path comparison,
not a replacement for the globally attaining rule proved above.

Independent checks use the same $128$-node observed-law discretization.
Seed $2026101748$ gives $3000$ replications at each of thirteen equally
spaced $\lambda$ values for $n=10^{11},10^{12},10^{13}$; seed
$2026101760$ gives the same grid at $n=3\times10^{12}$.
The eleven interior cell risks are below the endpoint risks in these
runs. This does not establish an endpoint maximum over the continuous
interval. Complete profiling over the regular coordinates and a
finite-sample variance correction gave no material further improvement.

To assess the remaining algorithmic loss, fix $c=3,v=v_n$ and the two
endpoints. Write $\theta_j=\tau(3,v_n,\lambda_j)$,
$\Delta=\theta_1-\theta_0$, and let $p_j$ be their one-observation
densities. The equal-prior Bayes risk and Cauchy--Schwarz give
\[
\cR_n^{\mathrm{sc}}(4)\ge B_n
=\frac{\Delta^2}{2}\int
\frac{p_0^{\otimes n}p_1^{\otimes n}}
     {p_0^{\otimes n}+p_1^{\otimes n}}
\ge\frac{\Delta^2}{4}
\left(\int\sqrt{p_0p_1}\right)^{2n}.
\]
Knowing $c,v_n$ can only reduce risk. We evaluate $B_n$ by Fourier
integration of the finite-product likelihood characteristic function;
$128$, $256$ and $512$ quadrature nodes give consistent values. These
are numerical approximations, not certified integration bounds.
The normalized values $n^{1/5}B_n/C_5$ are $1.343393$, $1.205865$ and
$1.077928$ at $n=10^8,10^9,10^{11}$, respectively.

For the equal-endpoint mixture, let $a_n^*$ be its known-coordinate
posterior mean. For any data-based rule $\delta$, the identity
$\frac12\sum_{j=0}^1\E_j(\delta-\theta_j)^2
=B_n+\E_{\rm mix}(\delta-a_n^*)^2$
provides a low-variance estimate of average endpoint risk.
Seed $2026101759$ supplies independent paired data, with
$10000$, $50000$, $10000$ replications per endpoint at the three sizes
in Table~\ref{supp:tab:supp-finite-likelihood}. True coordinates enter only
this diagnostic baseline, never the feasible decisions. Interior risks
remain direct Monte Carlo estimates.

\begin{table}[t]
\centering\small
\begin{tabular}{@{}rrr@{}}
\toprule
$n$ & Hermite feature & Finite likelihood\\
\midrule
$10^{12}$ & $1.051564$ & $1.048522$\\
$3\times10^{12}$ & $1.041216$ & $1.038763$\\
$10^{13}$ & $1.032225$ & $1.030332$\\
\bottomrule
\end{tabular}
\caption{Normalized average endpoint risks
$n^{1/5}\{\E_0(\delta-\theta_0)^2+
\E_1(\delta-\theta_1)^2\}/(2C_5)$, estimated using the mixture identity.
These averages are distinct from the grid maxima in
Table~\ref{supp:tab:supp-primary-all-sizes}.}
\label{supp:tab:supp-finite-likelihood}
\end{table}

At $3\times10^{12}$, the paired reduction in average MSE is about
$0.2355\%$. The likelihood decision is close to the two-point Bayes
benchmark, leaving little room for improvement on these paths.
Its risk at this size exceeds the more precise Hermite-feature average
at $10^{13}$; the older noisy grid maximum $1.0412$ is not a basis for
claiming improved accuracy after reducing sample size.
Supplement B records the paired uncertainties, separate endpoint risks
and all interior checks under \texttt{finite\_likelihood/}.
The commands \texttt{python reproduce.py likelihood-figures} and
\texttt{python reproduce.py likelihood-full} rebuild these summaries or
rerun the additional experiments.

\subsubsection{Random streams and exact reproduction}
All streams use NumPy PCG64. Logistic data and bootstrap keys are,
respectively,
\[
[2026101702,n,1000\lambda,b],\qquad
[2026101704,n,1000\lambda],\qquad b=0,\ldots,4999.
\]
The network size lists and data seeds are
\[
\begin{array}{ll}
(22000,30000,10^5,3\times10^5,10^6)&2026101706,\\
(21366)&2026101776,\\
(10^7)&2026101786,\\
(10^8)&2026101796.
\end{array}
\]
Within each list initialize \texttt{SeedSequence}$([s,i,j])$, where $s$
is its seed and $i,j$ are zero-based size and $\lambda$ indices. Draw
calibration histograms before main histograms in batches of the declared
replication count. The slope-bootstrap key is $[2026101796,71]$.
The extra $22000$ and $3\times10^5$ runs remain in the archive.
Edges, graph masks and latent labels retain their lexicographic reference
ordering. Logistic count means accumulate in \texttt{longdouble}.
The verified environment is CPython 3.12.14, NumPy 2.3.5, SciPy 1.17.0
and Matplotlib 3.10.8 on Linux x86\_64.

Supplement B contains all raw records, full-precision configurations,
solver tolerances, paired comparisons and replay checks under
\texttt{critical\_path/}. From its root the commands
\begin{quote}\small
\texttt{python reproduce.py verify}\\
\texttt{python reproduce.py figures}\\
\texttt{python reproduce.py critical-full}
\end{quote}
check integrity, rebuild summaries and figures, or regenerate the
critical-path observations.
The fixed-coordinate and auxiliary sparse, information, likelihood and
routing checks are retained. The historical primary records are archived
under \texttt{legacy\_checks/}; they are not inputs to the revised figure.

\subsection{Sixth-order information decay}

For $M=5$, $p=0.35$, and $\lambda=2$, exact summation over $2^{10}$ graphs
and $3^5$ label vectors gives slope $12.046$ on $0.016\le v\le0.030$; at
$v=0.016$, the ratio to $15\{p(1-p)\}^{-6}v^{12}$ is $1.01194$.

\begin{figure}[!htbp]
 \centering
 \includegraphics[page=5,width=.72\textwidth]{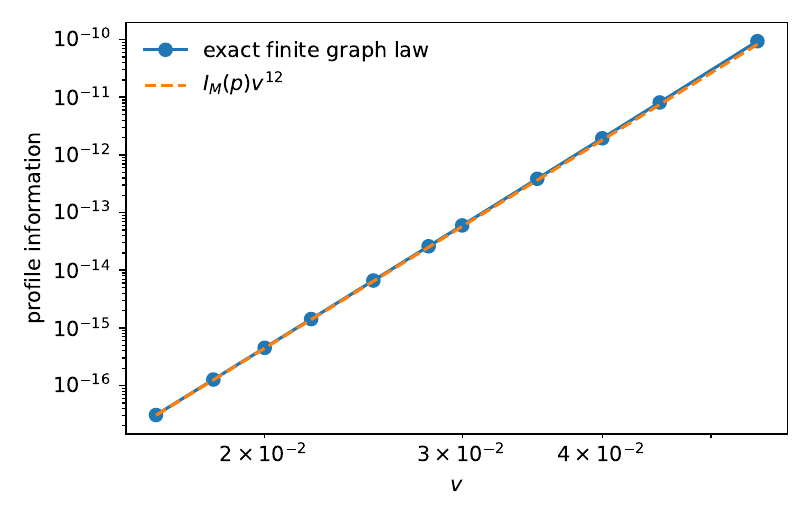}
 \caption{\textbf{Sixth-order information.} Exact information and
 $I_M(p)v^{12}=15\binom M5\{p(1-p)\}^{-6}v^{12}$.}
 \label{supp:fig:network-information}
\end{figure}

\subsection{Logistic design}
\label{supp:sec:numerical-design}

For $\mathcal A=\mathcal E=[1,2]$, the maximizer of
$\{\operatorname{diam}\Lambda(c)\}^{8/5}/c$ is $c=3$, with
$\Lambda(3)=[3/2,3]$. We use the profile inversion and coordinate map in
Section~\ref{supp:sec:supp-logistic-expanded-rule}. The geometry and local
experiments below use Gauss--Hermite (GH) quadrature, exact Bernoulli
sampling, seed $20260723$ and pointwise 95\% intervals.

\subsection{Information decay and the contact path}
\label{supp:sec:numerical-geometry}

The one-observation information for \(\lambda\) at fixed \(v\) is
\begin{equation*}
 \mathcal I_\lambda(v)
 =
 \E\left[
 \frac{\{\frac{\partial}{\partial\lambda}
 g_{v,\lambda}(T)\}^2}
 {1-g_{v,\lambda}(T)^2}
 \right],
 \qquad T\sim N(0,1).
\end{equation*}
At the endpoints and midpoint of $\Lambda(3)$, the six-point slopes are
$9.9994$, $9.9989$, and $9.9985$; at $v=0.005$,
\[
 \frac{\mathcal I_\lambda(v)}{v^{10}/1536}
 =0.999931,\quad0.999875,\quad0.999819,
\]
respectively.

The contact path uses $(a_0,\eta_0)=(1.25,1.75)$ and
$(a_1,\eta_1)=(1.75,1.25)$. For $c=3$, put $\rho_j=a_j/c$,
$h_j=c(1-\rho_j)/\rho_j$ and $\Delta=h_1-h_0$, and take
\[
 \theta_0=\frac{t}{\rho_0},\qquad
 \theta_1=\frac1{\rho_1}\left\{
 t+\frac{\Delta}{4}t^3
 -\frac{\Delta}{16}(c+2h_0-h_1)t^5\right\}.
\]
The probabilities differ by $\Delta t^5H_{3,c}/192+O_{L^2}(t^7)$, where
$H_{3,c}(z)=z^3-3cz$.  The KL slope is $9.9977$; at the smallest $t$, exact KL
divided by
\[
 \frac{\Delta^2c^3}{3072}t^{10}
\]
equals $0.999719$. Quadrature comparisons are reported in
Table~\ref{supp:tab:quadrature-diagnostics}.

\begin{figure}[t]
 \centering
 \includegraphics[page=2,width=\textwidth]{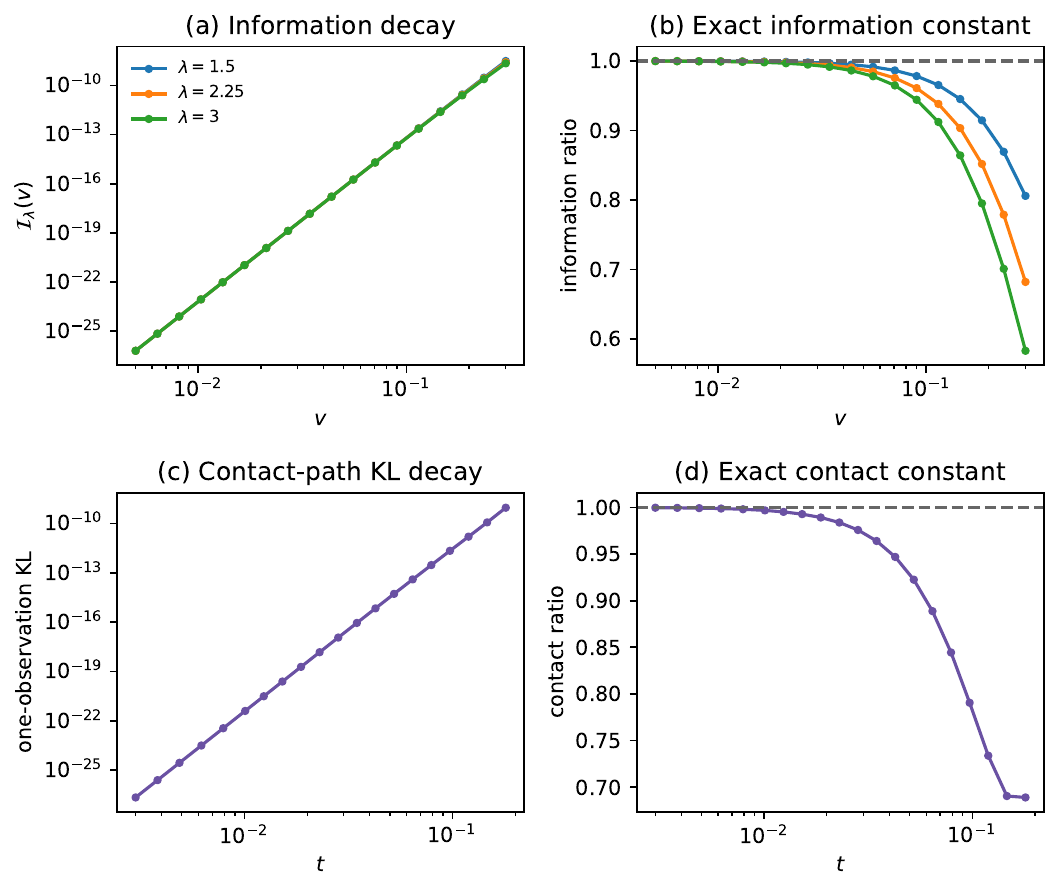}
 \caption{\textbf{Fifth-order geometry.} Information and contact-path KL,
 with their leading normalizations.}
 \label{supp:fig:numerical-geometry}
\end{figure}

\subsection{Sparse estimation: rate and exact-constant comparisons}
\label{supp:sec:numerical-sparse}

We compare a fixed-fraction $J_p(s)/n$ pilot with the sharp scale
$d_{n,s}=N_n^{1/5}J_p(s)/m_n$, using three independent folds and
$m_n=o(n)$ in the latter case. Throughout, $B=2$ and the positive
$x_\star$ and $C_5$ are as specified above.

\subsubsection{Data generation and estimators}
Write $\theta=te$ and take $e$ to have its first $s$ coordinates equal
to $s^{-1/2}$ and all others zero. Put $c=a+\eta$, $\lambda=c/a$,
$u=at/\sqrt c$ and $T=e^\top Z/\sqrt c$.
For independent $G_p\sim N(0,I_p)$, isotropy gives
\begin{equation*}
 \frac1m\sum_{i=1}^mR_iZ_i
 \ \overset d=\ 
 \frac{\sqrt c}{m}\left\{
 e\sum_{i=1}^mR_iT_i+
 \sqrt m\,(I-ee^\top)G_p\right\}.
\end{equation*}
Here $T_i,G_i$ are independent standard normals and
\[
 \Pp(R_i=1\mid T_i,G_i)
 =m\{u(T_i+\sqrt{\lambda-1}G_i)\}.
\]
This costs $O(m+p)$ per replication.

The adaptive rule is the penalized pilot \eqref{supp:eq:pilot}, with
$(\Lambda_0,c_+,M_\mu,C_0)=(2,4,2,20)$ and no knowledge of $s$ or its
support. For each $k=0,\ldots,s_+$, retain the $k$ largest coordinates of
$|V_0|$, project onto the $M_\mu$-ball and compare the criteria.
The theoretical bounds in \Cref{supp:lem:pilot} require a sufficiently large
penalty; these computations use $\Lambda_0=2$.

For $f=\widehat\mu/\norm{\widehat\mu}_2$ and $\zeta=e^\top f$, generate
\[
 T_f=\zeta T_e+\sqrt{1-\zeta^2}\,Q,
 \qquad
 \Pp(R=1\mid T_e,G)=m\{u(T_e+\sqrt{\lambda-1}G)\},
\]
with independent standard normals.  For $\zeta\ne0$, the exact coordinates are
\[
 v^\dagger=\zeta v,\qquad
\lambda^\dagger=\lambda/\zeta^2,\qquad
 q^\dagger=t/\zeta.
\]
The estimate is zero when $\widehat\mu=0$.  Otherwise, calibration estimates
$(c,v^\dagger)$, and the main fold computes the third Hermite statistic.  Before inversion,
the argument is projected onto the exact profile range.

The adaptive estimator uses only $s_+$, the parameter bounds, and the folds.
The oracle direction estimator sets $f=e$; the profile oracle also uses $(c,v)$.

\subsubsection{Direction estimation with a fixed sample split}
Take $a=\eta=1.5$ and $m\in\{1000,4000\}$. The dimensions are
$p\in\{500,2000\}$, the sparsities are $s\in\{1,8,16\}$ and $s_+=32$.
Set $t^2=\min\{B^2/4,16J_p(s)/m\}$. The two $m$ values use 120 and 60
replications per $(p,s)$ cell. Direction loss is $t^2\{1-(e^\top f)^2\}$,
with loss $t^2$ for a zero pilot.  With
$V_0=m^{-1}\sum_iR_iZ_i$, the dense benchmark is
\[
 \widehat\mu_{\rm dense}=\Pi_{B(0,M_\mu)}V_0,
 \qquad
 f_{\rm dense}=\widehat\mu_{\rm dense}/
                   \|\widehat\mu_{\rm dense}\|_2,
\]
with the zero convention.

For $J_p(s)/m\le0.04$, OLS slopes (SE) are $0.9964$ ($0.0199$) for pilot MSE
and $1.0596$ ($0.0243$) for direction loss.  Across all cells, normalized
ratios are $3.579$--$4.828$ and $7.200$--$12.870$; dense ratios are
$21.140$--$696.392$ and $10.151$--$15.817$.  A sensitivity analysis at
$(m,p,s)=(4000,2000,\{1,4,16,32\})$ uses $80$ pairs and
$\Lambda_0\in\{1,2,4\}$.  At $\Lambda_0=1$, mean support is $29.1$--$32.0$
and zero frequency is $0$; at $4$, zero frequency is $0.95$--$1.00$; at $2$,
they are $0.6$--$9.4$ and $0.05$--$0.40$.

\begin{figure}[t]
 \centering
 \includegraphics[page=8,width=\textwidth]{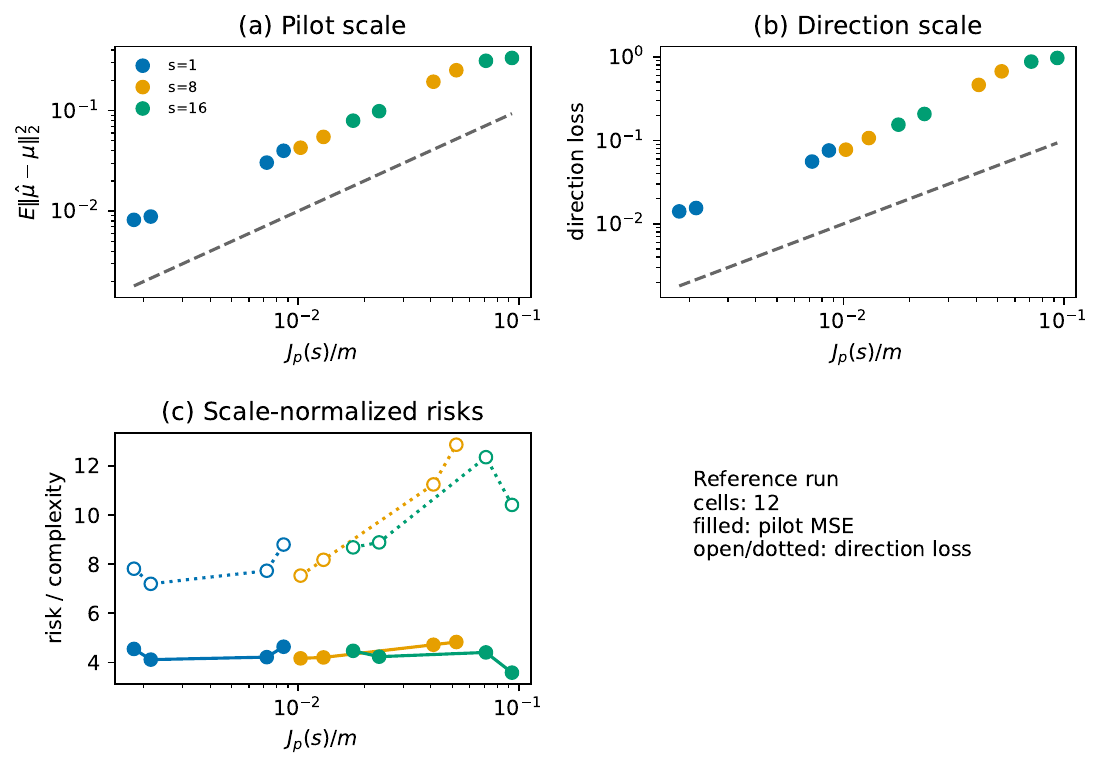}
 \caption{\textbf{Fixed-fraction direction experiment.} Pilot MSE, direction loss,
 normalized losses, and run sizes.}
 \label{supp:fig:sparse-rate}
\end{figure}

\subsubsection{Exact-constant comparison with a vanishing pilot fraction}
Set $p=500$, $s_+=16$, $n\in\{10^5,3\times10^5\}$,
$s\in\{1,8,16\}$ and $\lambda\in\{3/2,3\}$. The
$(m_n,m_n,N_n)$ folds use 20 and 10 pairs per $(s,\lambda)$ at the two
sample sizes, respectively.  Set
$v=x_\star N_n^{-1/10}$, solve $u=U(v,\lambda)$, and use
\[
 \omega_n=\min\{0.49,4d_n/r_n^2\},
 \qquad
 \rho_n=\min\{0.25,\log(n)/\sqrt{m_n}\},
\]
with $d_n,r_n$ as in the exact sparse theorem.

Normalized adaptive risks are $3.988$--$10.632$.
Table~\ref{supp:tab:sparse-sharp} gives selected paired comparisons with the
oracle direction rule.

\begin{table}[b]
\centering
\small
\setlength{\tabcolsep}{4pt}
\begin{tabular}{@{}rrrrcc@{}}
\toprule
$n$ & $s$ & $\lambda_{\rm emp}$ & $d_{n,s}$
& $N^{1/5}\widehat R_{\rm AD}$ & $N^{1/5}
(\widehat R_{\rm AD}-\widehat R_{\rm OD})$\\
\midrule
$10^5$ & 1  & 1.5 & 0.00187 & 6.976 $[3.085,10.866]$
 & 0.002 $[-0.004,0.009]$\\
$10^5$ & 16 & 3.0 & 0.01840 & 5.543 $[2.997,8.089]$
 & 0.098 $[0.014,0.182]$\\
$3\times10^5$ & 1  & 1.5 & 0.00089 & 8.709 $[-0.362,17.779]$
 & 0.000 $[0.000,0.000]$\\
$3\times10^5$ & 16 & 1.5 & 0.00881 & 7.304 $[-1.109,15.718]$
 & 0.165 $[-0.492,0.821]$\\
\bottomrule
\end{tabular}
\caption{\textbf{Selected settings for sparse estimation.} Simultaneous $95\%$ Bonferroni
intervals; the larger adaptive endpoint is shown.}
\label{supp:tab:sparse-sharp}
\end{table}

These intervals do not establish a finite-$n$ trend.

\begin{figure}[!htbp]
 \centering
 \includegraphics[page=9,width=\textwidth]{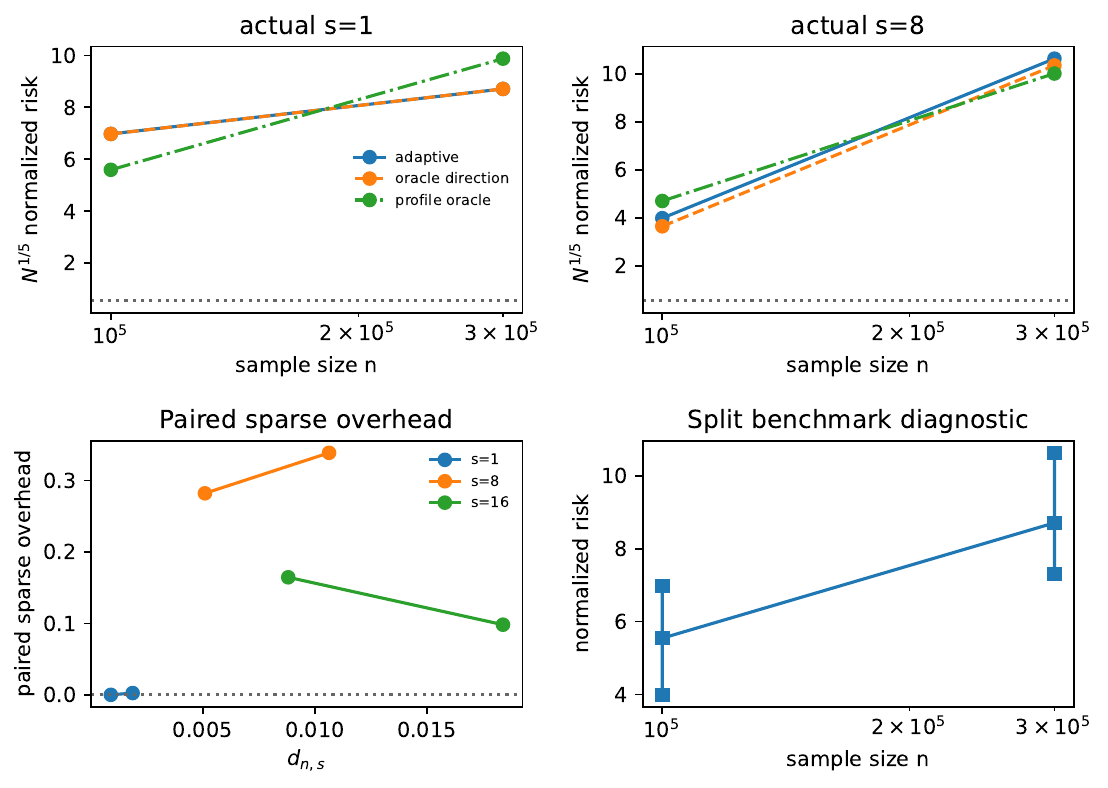}
 \caption{\textbf{Sparse estimation.} Adaptive and oracle risks, paired
 risk differences, and the sample-splitting benchmark.}
 \label{supp:fig:sparse-sharp}
\end{figure}

\subsubsection{Implementation across four regions}
The four candidates and ramp weights are those of Section~4.2 of the
main paper, with a finite-grid outer minimum-distance rule. The weighted
sum is projected successively onto $[-2B,2B]$,
$[-C_0\|\widehat\mu\|_2,C_0\|\widehat\mu\|_2]$ and the vector $B$-ball.

The $72$ settings include $18$ per intended region.  The numbers of settings
in which the zero, local, inverse, and outer estimators have the largest
weight are $(21,15,18,18)$, respectively.
For $51$ applicable settings, $\lambda^\dagger\in\Lambda_n^+(c)$ always and
$\lambda^\dagger\in\Lambda_n^+(\widehat c)$ in $49$.

Among settings with positive weight on the inverse estimator, the minimum
$257$-grid monotonicity margin is
$2.07\times10^{-8}$; maximum
inverse residual is $2.71\times10^{-19}$; maximum queried rule gap is
$1.83\times10^{-12}$.  The largest differences between the coarse and fine
outer grids are $1.5$ in
$\theta$ and $0.725$ in feature distance.  The chosen settings give
$\max N^{1/10}\rho_n/r=1.235$ and $\max MN^{-1/10}=0.175$.

\begin{figure}[!htbp]
 \centering
 \includegraphics[page=3,width=\textwidth]{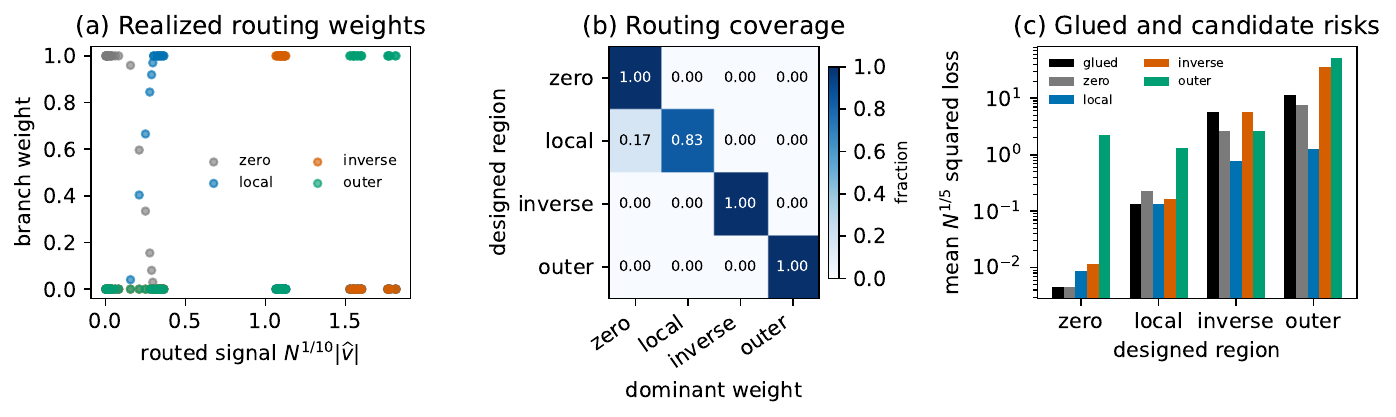}
 \caption{\textbf{Combination of four estimators.} Weights, estimator
 selection, and normalized losses.}
 \label{supp:fig:global-glue}
\end{figure}

\subsubsection{Numerical accuracy}
Over $6000$ direct $p=6$ runs, maximum standardized mean and variance
discrepancies are $1.49$ and $4.63\times10^{-4}$.  Maximum profile-root and
risk-decomposition residuals are $6.16\times10^{-14}$ and
$1.11\times10^{-15}$.  All $72$ settings satisfy the weight, support,
envelope, and projection constraints.

\subsection{The bounded normal experiment and its constant}
\label{supp:sec:numerical-normal}

The exchange computation uses a 2001-point grid and GH quadrature.
At the previously specified $z_\star$, it gives
$R_G(z_\star)\approx0.1885562107$ and the stated $K_5$ and $C_5$.
The mesh-corrected gaps are numerical.

\begin{table}[t]
\centering
\small
\begin{tabular}{@{}p{0.20\textwidth}p{0.35\textwidth}p{0.33\textwidth}@{}}
\toprule
Quantity & Comparison & Maximum change or gap\\
\midrule
Profiled information & GH 96 versus 128
 & $6.6\times10^{-15}$ relative\\
Contact-path KL & GH 96 versus 128
 & $2.7\times10^{-7}$ relative\\
Profile inversion & GH 96 with bracketed roots
 & $6.2\times10^{-14}$ maximum root residual\\
Likelihood surface & cubic versus direct evaluation; GH 128/192
 & $3.41\times10^{-12}$ for $g$; $6.27\times10^{-13}$ for the score derivative\\
$R_G(z_\star)$ & GH 80, 120, 160, and 240
 & $2.9\times10^{-15}$ absolute\\
$K_5$ computation & Exchange; 2001-point risk grid
 & endpoint support at $z_\star$; numerical gap $6.47\times10^{-4}$\\
\bottomrule
\end{tabular}
\caption{Quadrature and grid comparisons; gaps are numerical.}
\label{supp:tab:quadrature-diagnostics}
\end{table}

At $z_\star<2/\sqrt3$, the short-interval result gives the equal-endpoint
least-favourable prior, which the exchange calculation recovers.
Across GH orders $80,120,160,240$, the maximum change in $K_5$ is
$4.7\times10^{-16}$; the other checks are in
Table~\ref{supp:tab:quadrature-diagnostics}. The scan uses $z\in[0.08,8]$.  Outside the scan,
$R_G(z)\le\min(1,z^2/4)$ bounds $z^{-8/5}R_G(z)$ by $0.091$ on the left and
$0.036$ on the right, below the displayed maximum.

\begin{figure}[t]
 \centering
 \includegraphics[page=6,width=\textwidth]{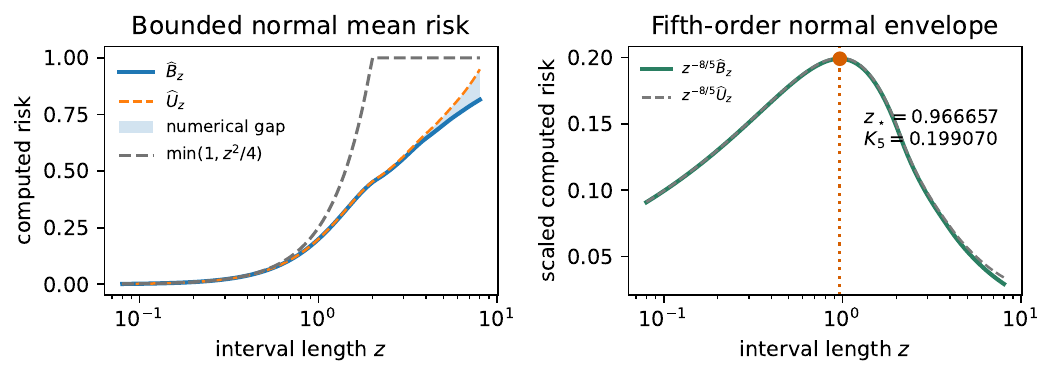}
 \caption{\textbf{Bounded-normal computation.} Computed Bayes and
 mesh-corrected rule risks, before and after scaling by $z^{-8/5}$.  The dot is
 the value $(z_\star,K_5)$ computed using the prior supported at the
 endpoints; gaps are numerical.}
 \label{supp:fig:normal-envelope}
\end{figure}

\subsection{Gaussian approximation and local minimax risk}
\label{supp:sec:numerical-lan-risk}

For $v_n=x_\star n^{-1/10}$, compute
\begin{equation*}
 \widehat W_{n,x_\star}
 =
 \frac{16}{v_n^5}\left\{
 \frac1n\sum_{i=1}^nR_iH_3(T_i)+\frac{v_n^3}{4}
 \right\}.
\end{equation*}
Center and scale by the exact $F(v_n,\lambda)$ moments.  On a $21$-point
$\lambda$ grid, compare the likelihood ratio at $\lambda_0=2.25$ with LAN.
The centered score is
\[
 S_{H,n}
 =\frac{v_n^5}{96}\sum_{i=1}^n
 \{R_iH_3(T_i)-F(v_n,\lambda_0)\},
 \qquad I_{H,n}=\frac{nv_n^{10}}{1536}.
\]
Use
\[
 n\in\{10^4,3\mathbin{\times}10^4,10^5,3\mathbin{\times}10^5,10^6\},
\]
with $150,90,45,30,20$ runs per $\lambda=1.5,2.25,3$.  Likelihood-ratio counts
are $30,18,10,5$ through $n=3\times10^5$.

At $n=10^6$, the standardized biases and variance ratios are, respectively,
\[
 (-0.190,-0.386,-0.626),\qquad
 (0.999891,0.999897,0.999902).
\]  From $n=10^4$ to
$3\mathbin{\times}10^5$, median grid-maximum LAN remainder decreases
$0.0431\to0.0184$ and its $90\%$ quantile $0.0798\to0.0384$.  Hermite-score
medians are $0.241,0.308,0.256,0.264$ and endpoint $90\%$ quantiles
$0.725,0.312$.

\begin{figure}[t]
 \centering
 \includegraphics[page=4,width=\textwidth]{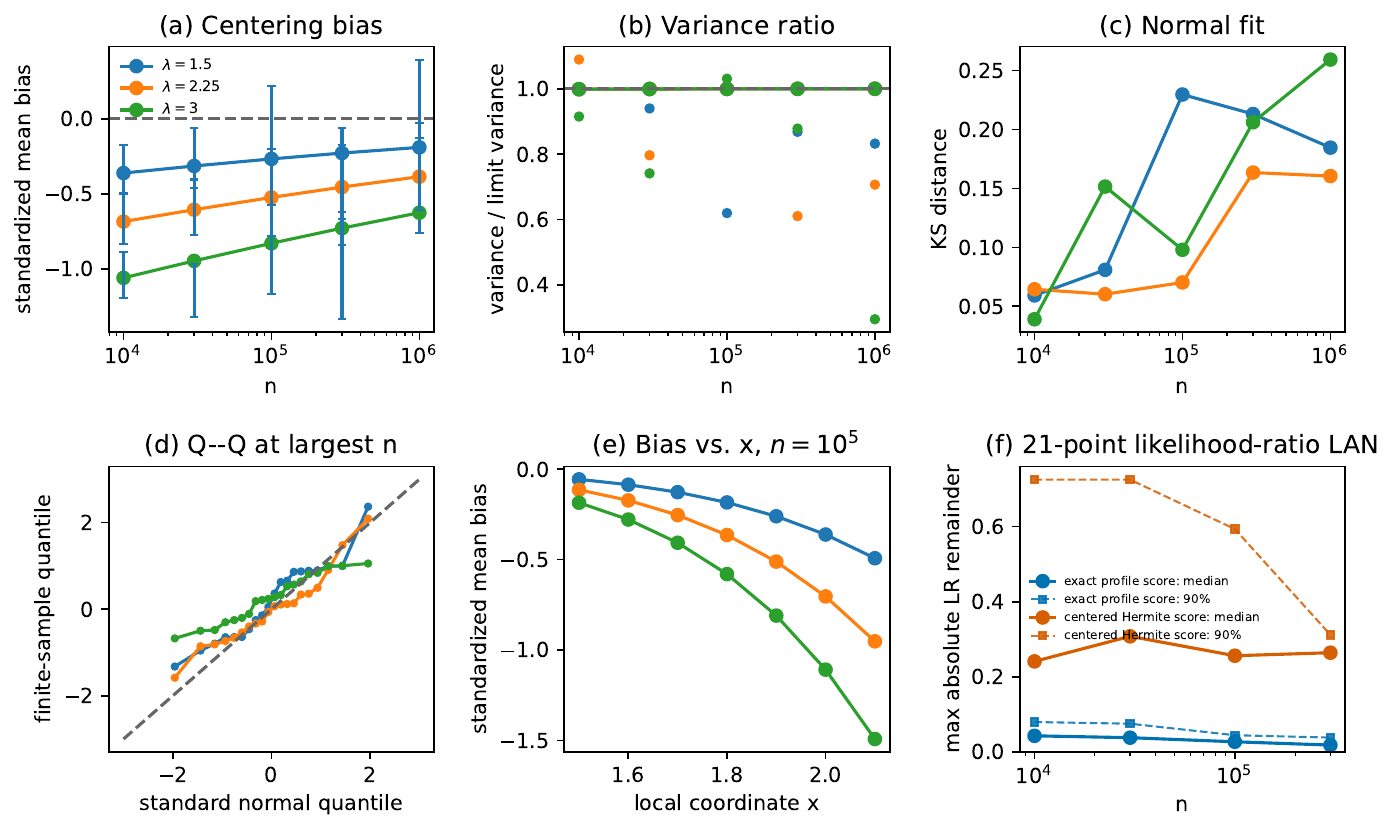}
 \caption{\textbf{Gaussian approximation.} Centering, variance, shape, bias,
 and the $21$-grid LAN remainder.}
 \label{supp:fig:numerical-lan}
\end{figure}

Let $\delta^\star_x$ be the computed rule for
$N(\lambda,1536/x^{10})$, $\lambda\in[3/2,3]$, and use
\begin{equation*}
 \widehat\theta_{n,x}
 =\frac{x n^{-1/10}}{\sqrt3}\,
   \delta^\star_x(\widehat W_{n,x}),
\end{equation*}
which uses the leading-order target map.  Its limit risk is
\[
 \frac{x^2}{3}
 M\left(\frac{1536}{x^{10}},[3/2,3]\right),
\]
and comparisons use projected Gaussian MLE and midpoint rules.

The finite-profile calibration defines
\[
 m_v(\lambda)
 =\frac{16}{v^5}\{F(v,\lambda)+v^3/4\}.
\]
Invert $m_v$ on its range, extend the inverse linearly, apply
$\delta^\star_x$, and use the same leading-order target map.  Uniformly on $[3/2,3]$,
\[
 m_v(\lambda)=\lambda+O(v^2),
 \qquad
 m_v'(\lambda)=1+O(v^2).
\]
At $n=10^6$, largest normalized risks are $2.241$ (leading-order) and $1.548$
(profile-calibrated), both at $\lambda=3$, versus limit $0.551$ ($20$ replications).

\begin{figure}[!htbp]
 \centering
 \includegraphics[page=7,width=.8\textwidth]{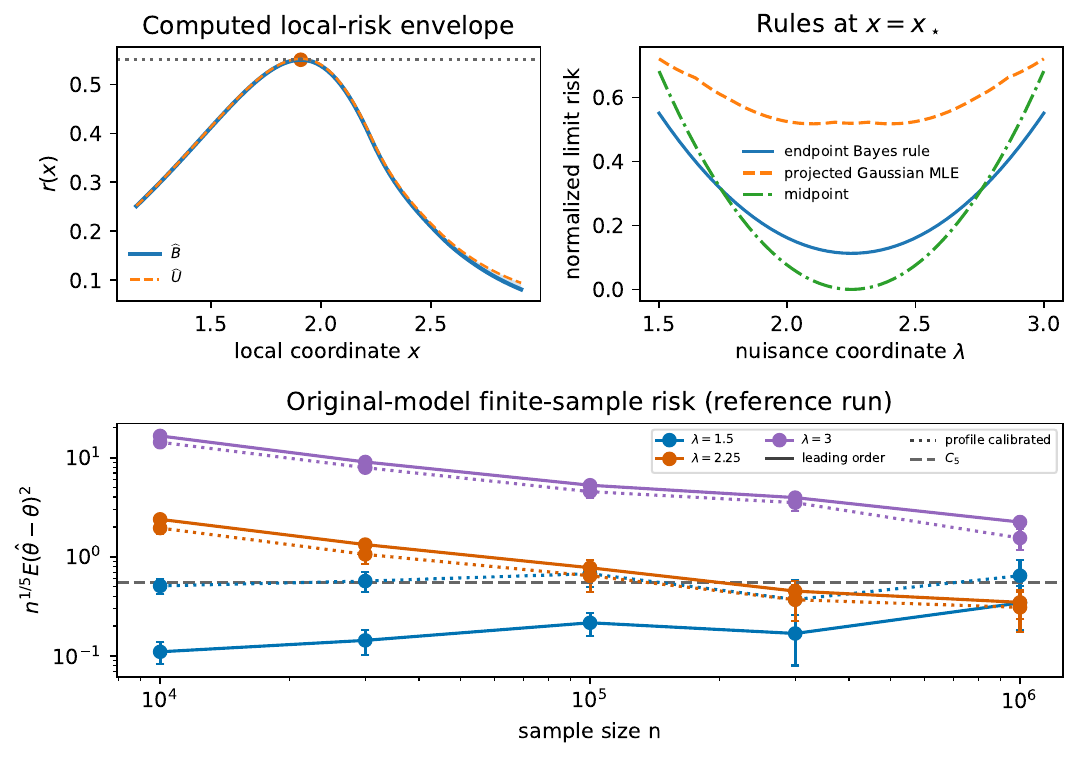}
 \caption{\textbf{Local risk.} Normal approximation, bounded-normal rules, and
 $1005$ exact-model statistics; bars are pointwise $95\%$ intervals.}
 \label{supp:fig:numerical-risk}
\end{figure}


\begin{thebibliography}{34}

\bibitem{abbe2018sbm}
\begin{barticle}
\bauthor{\bsnm{Abbe}, \binits{E.}}
(\byear{2018}).
\batitle{Community Detection and Stochastic Block Models: Recent Developments}.
\bjtitle{Journal of Machine Learning Research}
\bvolume{18}
\bfpage{1}--\blpage{86}.
\end{barticle}
\endbibitem

\bibitem{allmanMatiasRhodes2011}
\begin{barticle}
\bauthor{\bsnm{Allman}, \binits{E. S.}},
\bauthor{\bsnm{Matias}, \binits{C.}} and
\bauthor{\bsnm{Rhodes}, \binits{J. A.}}
(\byear{2011}).
\batitle{Parameter Identifiability in a Class of Random Graph Mixture Models}.
\bjtitle{Journal of Statistical Planning and Inference}
\bvolume{141}
\bfpage{1719}--\blpage{1736}.
\end{barticle}
\endbibitem

\bibitem{andrewsCheng2012}
\begin{barticle}
\bauthor{\bsnm{Andrews}, \binits{D. W. K.}} and
\bauthor{\bsnm{Cheng}, \binits{X.}}
(\byear{2012}).
\batitle{Estimation and Inference with Weak, Semi-Strong, and Strong
  Identification}.
\bjtitle{Econometrica}
\bvolume{80}
\bfpage{2153}--\blpage{2211}.
\end{barticle}
\endbibitem

\bibitem{bickel1981}
\begin{barticle}
\bauthor{\bsnm{Bickel}, \binits{P. J.}}
(\byear{1981}).
\batitle{Minimax Estimation of the Mean of a Normal Distribution When the
  Parameter Space Is Restricted}.
\bjtitle{The Annals of Statistics}
\bvolume{9}
\bfpage{1301}--\blpage{1309}.
\end{barticle}
\endbibitem

\bibitem{carroll2006measurement}
\begin{bbook}
\bauthor{\bsnm{Carroll}, \binits{R. J.}},
\bauthor{\bsnm{Ruppert}, \binits{D.}},
\bauthor{\bsnm{Stefanski}, \binits{L. A.}} and
\bauthor{\bsnm{Crainiceanu}, \binits{C. M.}}
(\byear{2006}).
\bbtitle{Measurement Error in Nonlinear Models: A Modern Perspective},
\bedition{2} ed.
\bpublisher{Chapman and Hall/CRC}.
\end{bbook}
\endbibitem

\bibitem{carroll1984binary}
\begin{barticle}
\bauthor{\bsnm{Carroll}, \binits{R. J.}},
\bauthor{\bsnm{Spiegelman}, \binits{C. H.}},
\bauthor{\bsnm{Lan}, \binits{K. K. G.}},
\bauthor{\bsnm{Bailey}, \binits{K. T.}} and
\bauthor{\bsnm{Abbott}, \binits{R. D.}}
(\byear{1984}).
\batitle{On Errors-in-Variables for Binary Regression Models}.
\bjtitle{Biometrika}
\bvolume{71}
\bfpage{19}--\blpage{25}.
\end{barticle}
\endbibitem

\bibitem{carrollWand1991}
\begin{barticle}
\bauthor{\bsnm{Carroll}, \binits{R. J.}} and
\bauthor{\bsnm{Wand}, \binits{M. P.}}
(\byear{1991}).
\batitle{Semiparametric Estimation in Logistic Measurement Error Models}.
\bjtitle{Journal of the Royal Statistical Society: Series B (Methodological)}
\bvolume{53}
\bfpage{573}--\blpage{585}.
\end{barticle}
\endbibitem

\bibitem{casellaStrawderman1981}
\begin{barticle}
\bauthor{\bsnm{Casella}, \binits{G.}} and
\bauthor{\bsnm{Strawderman}, \binits{W. E.}}
(\byear{1981}).
\batitle{Estimating a Bounded Normal Mean}.
\bjtitle{The Annals of Statistics}
\bvolume{9}
\bfpage{870}--\blpage{878}.
\end{barticle}
\endbibitem

\bibitem{chen1995mixture}
\begin{barticle}
\bauthor{\bsnm{Chen}, \binits{J.}}
(\byear{1995}).
\batitle{Optimal Rate of Convergence for Finite Mixture Models}.
\bjtitle{The Annals of Statistics}
\bvolume{23}
\bfpage{221}--\blpage{233}.
\end{barticle}
\endbibitem

\bibitem{chernozhukovEtAl2018}
\begin{barticle}
\bauthor{\bsnm{Chernozhukov}, \binits{V.}},
\bauthor{\bsnm{Chetverikov}, \binits{D.}},
\bauthor{\bsnm{Demirer}, \binits{M.}},
\bauthor{\bsnm{Duflo}, \binits{E.}},
\bauthor{\bsnm{Hansen}, \binits{C.}},
\bauthor{\bsnm{Newey}, \binits{W. K.}} and
\bauthor{\bsnm{Robins}, \binits{J.}}
(\byear{2018}).
\batitle{Double/Debiased Machine Learning for Treatment and Structural
  Parameters}.
\bjtitle{The Econometrics Journal}
\bvolume{21}
\bfpage{1}--\blpage{68}.
\end{barticle}
\endbibitem

\bibitem{dacunhaGassiat1997}
\begin{barticle}
\bauthor{\bsnm{Dacunha-Castelle}, \binits{D.}} and
\bauthor{\bsnm{Gassiat}, \binits{E.}}
(\byear{1997}).
\batitle{Testing in Locally Conic Models, and Application to Mixture Models}.
\bjtitle{ESAIM: Probability and Statistics}
\bvolume{1}
\bfpage{285}--\blpage{317}.
\end{barticle}
\endbibitem

\bibitem{donohoLiu1991}
\begin{barticle}
\bauthor{\bsnm{Donoho}, \binits{D. L.}} and
\bauthor{\bsnm{Liu}, \binits{R. C.}}
(\byear{1991}).
\batitle{Geometrizing Rates of Convergence, II}.
\bjtitle{The Annals of Statistics}
\bvolume{19}
\bfpage{633}--\blpage{667}.
\end{barticle}
\endbibitem

\bibitem{donohoLiu1991iii}
\begin{barticle}
\bauthor{\bsnm{Donoho}, \binits{D. L.}} and
\bauthor{\bsnm{Liu}, \binits{R. C.}}
(\byear{1991}).
\batitle{Geometrizing Rates of Convergence, III}.
\bjtitle{The Annals of Statistics}
\bvolume{19}
\bfpage{668}--\blpage{701}.
\end{barticle}
\endbibitem

\bibitem{duranteDunsonVogelstein2017}
\begin{barticle}
\bauthor{\bsnm{Durante}, \binits{D.}},
\bauthor{\bsnm{Dunson}, \binits{D. B.}} and
\bauthor{\bsnm{Vogelstein}, \binits{J. T.}}
(\byear{2017}).
\batitle{Nonparametric Bayes Modeling of Populations of Networks}.
\bjtitle{Journal of the American Statistical Association}
\bvolume{112}
\bfpage{1516}--\blpage{1530}.
\end{barticle}
\endbibitem

\bibitem{gundersonBravoHermsdorffOrbanz2023}
\begin{bchapter}
\bauthor{\bsnm{Gunderson}, \binits{L. M.}},
\bauthor{\bsnm{Bravo-Hermsdorff}, \binits{G.}} and
\bauthor{\bsnm{Orbanz}, \binits{P.}}
(\byear{2023}).
\bctitle{The Graph Pencil Method: Mapping Subgraph Densities to Stochastic
  Block Models}.
In \bbtitle{Advances in Neural Information Processing Systems}
\bseriesno{36}
\bfpage{79583}--\blpage{79593}.
\end{bchapter}
\endbibitem

\bibitem{hajek1972}
\begin{bchapter}
\bauthor{\bsnm{H{\'a}jek}, \binits{J.}}
(\byear{1972}).
\bctitle{Local Asymptotic Minimax and Admissibility in Estimation}.
In \bbtitle{Proceedings of the Sixth Berkeley Symposium on Mathematical
  Statistics and Probability, Volume I: Theory of Statistics}
(\beditor{\binits{L. M.} \bsnm{Le~Cam}},
\beditor{\binits{J.} \bsnm{Neyman}} and
\beditor{\binits{E. L.} \bsnm{Scott}}, eds.)
\bfpage{175}--\blpage{194}.
\blocation{Berkeley}: \bpublisher{University of California Press}.
\end{bchapter}
\endbibitem

\bibitem{heinrichKahn2018}
\begin{barticle}
\bauthor{\bsnm{Heinrich}, \binits{P.}} and
\bauthor{\bsnm{Kahn}, \binits{J.}}
(\byear{2018}).
\batitle{Strong Identifiability and Optimal Minimax Rates for Finite Mixture
  Estimation}.
\bjtitle{The Annals of Statistics}
\bvolume{46}
\bfpage{2844}--\blpage{2870}.
\end{barticle}
\endbibitem

\bibitem{hoNguyen2016weak}
\begin{barticle}
\bauthor{\bsnm{Ho}, \binits{N.}} and
\bauthor{\bsnm{Nguyen}, \binits{X.}}
(\byear{2016}).
\batitle{Convergence Rates of Parameter Estimation for Some Weakly Identifiable
  Finite Mixtures}.
\bjtitle{The Annals of Statistics}
\bvolume{44}
\bfpage{2726}--\blpage{2755}.
\end{barticle}
\endbibitem

\bibitem{hoNguyen2016strong}
\begin{barticle}
\bauthor{\bsnm{Ho}, \binits{N.}} and
\bauthor{\bsnm{Nguyen}, \binits{X.}}
(\byear{2016}).
\batitle{On Strong Identifiability and Convergence Rates of Parameter
  Estimation in Finite Mixtures}.
\bjtitle{Electronic Journal of Statistics}
\bvolume{10}
\bfpage{271}--\blpage{307}.
\end{barticle}
\endbibitem

\bibitem{hoNguyen2019singularity}
\begin{barticle}
\bauthor{\bsnm{Ho}, \binits{N.}} and
\bauthor{\bsnm{Nguyen}, \binits{X.}}
(\byear{2019}).
\batitle{Singularity Structures and Impacts on Parameter Estimation in Finite
  Mixtures of Distributions}.
\bjtitle{SIAM Journal on Mathematics of Data Science}
\bvolume{1}
\bfpage{730}--\blpage{758}.
\end{barticle}
\endbibitem

\bibitem{jochmans2024}
\begin{barticle}
\bauthor{\bsnm{Jochmans}, \binits{K.}}
(\byear{2024}).
\batitle{Nonparametric Identification and Estimation of Stochastic Block Models
  from Many Small Networks}.
\bjtitle{Journal of Econometrics}
\bvolume{242}
\bfpage{105805}.
\end{barticle}
\endbibitem

\bibitem{kuchenhoff1995identification}
\begin{barticle}
\bauthor{\bsnm{K{\"u}chenhoff}, \binits{H.}}
(\byear{1995}).
\batitle{The Identification of Logistic Regression Models with Errors in the
  Variables}.
\bjtitle{Statistical Papers}
\bvolume{36}
\bfpage{41}--\blpage{47}.
\end{barticle}
\endbibitem

\bibitem{lecamYang2000}
\begin{bbook}
\bauthor{\bsnm{Le~Cam}, \binits{L.}} and
\bauthor{\bsnm{Yang}, \binits{G. L.}}
(\byear{2000}).
\bbtitle{Asymptotics in Statistics: Some Basic Concepts},
\bedition{2} ed.
\blocation{New York}: \bpublisher{Springer}.
\end{bbook}
\endbibitem

\bibitem{nakamura1990corrected}
\begin{barticle}
\bauthor{\bsnm{Nakamura}, \binits{T.}}
(\byear{1990}).
\batitle{Corrected Score Function for Errors-in-Variables Models: Methodology
  and Application to Generalized Linear Models}.
\bjtitle{Biometrika}
\bvolume{77}
\bfpage{127}--\blpage{137}.
\end{barticle}
\endbibitem

\bibitem{nussbaum1996equivalence}
\begin{barticle}
\bauthor{\bsnm{Nussbaum}, \binits{M.}}
(\byear{1996}).
\batitle{Asymptotic Equivalence of Density Estimation and Gaussian White
  Noise}.
\bjtitle{The Annals of Statistics}
\bvolume{24}
\bfpage{2399}--\blpage{2430}.
\end{barticle}
\endbibitem

\bibitem{rothenberg1971}
\begin{barticle}
\bauthor{\bsnm{Rothenberg}, \binits{T. J.}}
(\byear{1971}).
\batitle{Identification in Parametric Models}.
\bjtitle{Econometrica}
\bvolume{39}
\bfpage{577}--\blpage{591}.
\end{barticle}
\endbibitem

\bibitem{rotnitzkyEtAl2000}
\begin{barticle}
\bauthor{\bsnm{Rotnitzky}, \binits{A.}},
\bauthor{\bsnm{Cox}, \binits{D. R.}},
\bauthor{\bsnm{Bottai}, \binits{M.}} and
\bauthor{\bsnm{Robins}, \binits{J.}}
(\byear{2000}).
\batitle{Likelihood-Based Inference with Singular Information Matrix}.
\bjtitle{Bernoulli}
\bvolume{6}
\bfpage{243}--\blpage{284}.
\end{barticle}
\endbibitem

\bibitem{soganRebafka2026}
\begin{botherref}
\oauthor{\bsnm{Sogan}, \binits{R. B.}} and
\oauthor{\bsnm{Rebafka}, \binits{T.}}
(\byear{2026}).
Beyond Degree: Rooted Motif Signatures for Latent Position Identifiability in
  Graphon Models.
\textit{arXiv preprint arXiv:2607.01358}.
\end{botherref}
\endbibitem

\bibitem{sorensen2015measurement}
\begin{barticle}
\bauthor{\bsnm{S{\o}rensen}, \binits{{\O}.}},
\bauthor{\bsnm{Frigessi}, \binits{A.}} and
\bauthor{\bsnm{Thoresen}, \binits{M.}}
(\byear{2015}).
\batitle{Measurement Error in Lasso: Impact and Likelihood Bias Correction}.
\bjtitle{Statistica Sinica}
\bvolume{25}
\bfpage{809}--\blpage{829}.
\end{barticle}
\endbibitem

\bibitem{sorensenEtAl2018covariate}
\begin{barticle}
\bauthor{\bsnm{S{\o}rensen}, \binits{{\O}.}},
\bauthor{\bsnm{Hellton}, \binits{K. H.}},
\bauthor{\bsnm{Frigessi}, \binits{A.}} and
\bauthor{\bsnm{Thoresen}, \binits{M.}}
(\byear{2018}).
\batitle{Covariate Selection in High-Dimensional Generalized Linear Models with
  Measurement Error}.
\bjtitle{Journal of Computational and Graphical Statistics}
\bvolume{27}
\bfpage{739}--\blpage{749}.
\end{barticle}
\endbibitem

\bibitem{stefanski1985logistic}
\begin{barticle}
\bauthor{\bsnm{Stefanski}, \binits{L. A.}} and
\bauthor{\bsnm{Carroll}, \binits{R. J.}}
(\byear{1985}).
\batitle{Covariate Measurement Error in Logistic Regression}.
\bjtitle{The Annals of Statistics}
\bvolume{13}
\bfpage{1335}--\blpage{1351}.
\end{barticle}
\endbibitem

\bibitem{stefanskiCarroll1987}
\begin{barticle}
\bauthor{\bsnm{Stefanski}, \binits{L. A.}} and
\bauthor{\bsnm{Carroll}, \binits{R. J.}}
(\byear{1987}).
\batitle{Conditional Scores and Optimal Scores for Generalized Linear
  Measurement-Error Models}.
\bjtitle{Biometrika}
\bvolume{74}
\bfpage{703}--\blpage{716}.
\end{barticle}
\endbibitem

\bibitem{vanderVaart1998}
\begin{bbook}
\bauthor{\bparticle{van~der} \bsnm{Vaart}, \binits{A. W.}}
(\byear{1998}).
\bbtitle{Asymptotic Statistics}.
\blocation{Cambridge}: \bpublisher{Cambridge University Press}.
\end{bbook}
\endbibitem

\bibitem{weiMukherjeeNguyen2025}
\begin{botherref}
\oauthor{\bsnm{Wei}, \binits{Y.}},
\oauthor{\bsnm{Mukherjee}, \binits{S.}} and
\oauthor{\bsnm{Nguyen}, \binits{X.}}
(\byear{2025}).
Minimum {$\Phi$}-Distance Estimators for Finite Mixing Measures.
\textit{arXiv preprint arXiv:2304.10052v2}.
\end{botherref}
\endbibitem

\end{thebibliography}

\begin{thebibliography}{4}

\bibitem{supp:bickel1981}
\begin{barticle}
\bauthor{\bsnm{Bickel}, \binits{P. J.}}
(\byear{1981}).
\batitle{Minimax Estimation of the Mean of a Normal Distribution When the
  Parameter Space Is Restricted}.
\bjtitle{The Annals of Statistics}
\bvolume{9}
\bfpage{1301}--\blpage{1309}.
\end{barticle}
\endbibitem

\bibitem{supp:casellaStrawderman1981}
\begin{barticle}
\bauthor{\bsnm{Casella}, \binits{G.}} and
\bauthor{\bsnm{Strawderman}, \binits{W. E.}}
(\byear{1981}).
\batitle{Estimating a Bounded Normal Mean}.
\bjtitle{The Annals of Statistics}
\bvolume{9}
\bfpage{870}--\blpage{878}.
\end{barticle}
\endbibitem

\bibitem{supp:donohoLiu1991iii}
\begin{barticle}
\bauthor{\bsnm{Donoho}, \binits{D. L.}} and
\bauthor{\bsnm{Liu}, \binits{R. C.}}
(\byear{1991}).
\batitle{Geometrizing Rates of Convergence, III}.
\bjtitle{The Annals of Statistics}
\bvolume{19}
\bfpage{668}--\blpage{701}.
\end{barticle}
\endbibitem

\bibitem{supp:sion1958}
\begin{barticle}
\bauthor{\bsnm{Sion}, \binits{M.}}
(\byear{1958}).
\batitle{On General Minimax Theorems}.
\bjtitle{Pacific Journal of Mathematics}
\bvolume{8}
\bfpage{171}--\blpage{176}.
\end{barticle}
\endbibitem

\end{thebibliography}
\end{document}